\documentclass[ms,final,nonblindrev]{informs3_hide}

\OneAndAHalfSpacedXI
\renewcommand{\ABSfont}{\normalsize}
\usepackage{natbib,amsmath,bm,dsfont}
\usepackage{algorithm}
\usepackage{algorithmicx}
\usepackage[noend]{algpseudocode}
\usepackage{subcaption,caption,booktabs,placeins,threeparttable,makecell,array}
\usepackage[dvipsnames]{xcolor}
\usepackage{tikz}
\usetikzlibrary{arrows.meta}
\usepackage[hypertexnames=false,hidelinks]{hyperref}

\bibpunct[, ]{(}{)}{,}{a}{}{,}
\def\bibfont{\small}

\newcommand{\mathbbm}[1]{\mathds{#1}}
\newcommand{\eps}{\varepsilon}
\newcommand{\bI}{\mathbbm{1}}
\newcommand{\cP}{\mathcal{P}}
\newcommand{\Var}{\mathrm{Var}}
\newcommand{\Cov}{\mathrm{Cov}}
\newcommand{\bE}{\mathrm{E}}
\newcommand{\bP}{\mathrm{P}}

\TheoremsNumberedThrough
\ECRepeatTheorems

\EquationsNumberedThrough

\newcommand{\fitwidth}[1]{%
  \leavevmode\setbox0=\hbox{#1}%
  \ifdim\wd0>\linewidth
    \resizebox{\linewidth}{!}{\box0}%
  \else
    \box0
  \fi\par
}


\MANUSCRIPTNO{}
\renewcommand{\theARTICLETOP}{}
\RUNAUTHOR{Ye, Lyu, and Tao}
\providecommand{\RRHFirstLine}[1]{}
\providecommand{\LRHFirstLine}[1]{}
\providecommand{\ECRRHFirstLine}[1]{}
\providecommand{\ECLRHFirstLine}[1]{}
\providecommand{\ECRRHSecondLine}[1]{}
\providecommand{\ECLRHSecondLine}[1]{}
\RRHFirstLine{Preprint}
\LRHFirstLine{Preprint}
\ECRRHFirstLine{Supplementary Appendix}
\ECLRHFirstLine{Supplementary Appendix}
\ECRRHSecondLine{Allocating Human Oversight in AI-Enabled Analytics}
\ECLRHSecondLine{Allocating Human Oversight in AI-Enabled Analytics}
\hypersetup{
  pdftitle={Allocating Human Oversight in AI-Enabled Analytics},
  pdfauthor={Zikun Ye; Jiameng Lyu; Rui Tao},
  pdfsubject={Preprint with supplementary appendix}
}

\begin{document}
\RUNTITLE{Allocating Human Oversight in AI-Enabled Analytics}

\TITLE{Allocating Human Oversight in AI-Enabled Analytics}

\ARTICLEAUTHORS{
\AUTHOR{Zikun Ye}
\AFF{Michael G. Foster School of Business, University of Washington, \EMAIL{zikunye@uw.edu}}
\AUTHOR{Jiameng Lyu\thanks{Corresponding Author: Jiameng Lyu}}
\AFF{Department of Management Science, School of Management, Fudan  University, Shanghai, China, \EMAIL{jiamenglyu@fudan.edu.cn}}
\AUTHOR{Rui Tao}
\AFF{Guanghua School of Management, Peking University, \EMAIL{taorui@stu.pku.edu.cn}}
}

\ABSTRACT{Artificial intelligence provides low-cost predictions for many analytics tasks, but their reliability and the value of another human label vary across tasks. We study how to allocate a shared human-validation budget when task-level uncertainty is unknown and each acquired label informs both estimation and subsequent allocation. For known difficulty levels, the benchmark follows a cost-adjusted Neyman rule that accounts for task importance and label cost. Within our Tuned Rectification-Difficulty framework, we develop an upper-confidence-bound policy that jointly learns the correction coefficients of tuned prediction-powered inference and the resulting difficulties from paired human labels and predictions, and establish an allocation-regret guarantee for the confidence-radius policy. In synthetic experiments and a real digital-twin survey with 68 tasks and over 2{,}000 respondents, our method closes most of the gap to the oracle on the allocation objective. It matches the estimation error of uniform allocation with up to 22.7\% and 4.5\% fewer human labels, respectively. The same allocation structure extends beyond scalar means to vector-valued decision models, including regression and choice/conjoint estimation, by replacing the scalar residual variance with a scalarized sandwich covariance. The value of AI depends not only on model accuracy but also on the operational policy that targets human oversight where it matters most.}


\maketitle

\section{Introduction}

Generative AI is becoming a low-cost prediction layer in customer-facing decision processes. Firms use large language models (LLMs) to summarize reviews and chat logs, score open-ended feedback, simulate customer reactions to product concepts, support service-quality monitoring, and assist demand sensing, all at very low cost \citep{peng2025mega, ziems2024can, dominguez2024questioning}. These AI signals increasingly feed operational decisions: what to stock, which issues to escalate, how to staff service teams, which concepts or promotions to test, and where to deploy marketing resources.

AI reliability, however, varies across products, customer segments, and task types (LLM predictions exhibit systematic biases and task-dependent accuracy; \citealp{motoki2024more, brucks2023prompt, toubia2025database}), and this heterogeneity is rarely known before deployment. Human input therefore remains the ground truth: firms must still collect survey responses, label open-ended feedback, audit AI-coded interactions, verify service-quality judgments, or run follow-up measurements. We use \emph{human validation} broadly for these costly human-provided ground-truth signals used to anchor and correct AI outputs. The productive response is not to replace humans with AI but to combine cheap AI predictions with human labels, reducing how many labels are needed while correcting AI errors; several frameworks formalize this, including prediction-powered inference \citep[PPI;][]{angelopoulos2023prediction, angelopoulos2024ppiplus}, fine-tuning-then-rectification \citep{wang2025finetune}, and mixed human--LLM designs \citep{broska2025mixed, yin2026synthetic}. Crucially, human ground truth is itself noisy, varying across labelers and even across repeated judgments by the same labeler, so a firm must collect and average several labels per task. The variation that remains after the cheap AI signal is used optimally as a control variate is exactly what governs how many human labels each task needs, and it exists in every human-in-the-loop application, not only in surveys.

This creates an allocation problem: AI can generate many low-cost signals, but firms must decide where scarce human-validation capacity should be spent across a portfolio of tasks. LLM prediction quality varies dramatically across tasks, and this heterogeneity persists even after prompt engineering or fine-tuning \citep{dominguez2024questioning, wang2025finetune, krsteski2025valid}. Prediction quality affects how many human labels each task requires \citep{neyman1934two}. When difficulty is heterogeneous, directing more of the budget to the harder tasks substantially reduces aggregate error relative to uniform allocation; when it is homogeneous, there is little to gain from adapting to difficulty. Uniform validation is thus simple but can be wasteful, while a separate pilot study is costly and brittle, since reliability estimates may not transfer across populations, prompts, models, or task portfolios. Many validation workflows are naturally sequential or batched, however (audits, labels, survey waves, and follow-up measurements arrive over time), so a firm can learn where human oversight matters most while it spends the validation budget, from the uncertainty that remains in each task after AI correction.

This paper studies the following question: \emph{how should a firm adaptively allocate a finite human-validation or labeling budget across many AI-assisted analytics tasks when task-level AI reliability is heterogeneous and unknown before deployment?}

LLM-augmented surveys and digital-twin studies provide a concrete empirical testbed for this broader problem, because they generate paired human--AI observations across many analytics tasks \citep{brand2023using, argyle2023out, wang2024market}. We use survey questions as the running example, but the allocation unit is general (Table~\ref{tab:applications}). Figure~\ref{fig:pipeline} illustrates the allocation pipeline.

\begin{figure}[t]
\centering
\resizebox{\textwidth}{!}{%
\begin{tikzpicture}[
  box/.style={rectangle, draw, rounded corners=3pt, minimum height=1.1cm, text width=2.5cm, align=center, font=\small, execute at begin node={\setlength{\baselineskip}{1.2em}}},
  arr/.style={-{Stealth[length=6pt]}, thick}
]
\node[box] (ai) at (0,0) {AI Prediction\\Layer};
\node[box] (tasks) at (5,0) {$Q$ Analytics Tasks\\{\footnotesize(unknown $A_q$)}};
\node[box] (alloc) at (10,0) {Adaptive\\Allocation};
\node[box] (out) at (15,0) {Calibrated\\Estimates};

\draw[arr] (ai) -- node[above, font=\footnotesize, yshift=1pt] {cheap signals} (tasks);
\draw[arr] (tasks) -- node[above, font=\footnotesize, yshift=1pt] {paired data} (alloc);
\draw[arr] (alloc) -- (out);

\node[box, dashed, text width=3.3cm] (human) at (10,2.0) {Human Validation\\{\footnotesize(budget $B$, costly)}};
\draw[arr] (human) -- (alloc);

\draw[arr, densely dashed, gray] (alloc.south) -- ++(0,-0.9) -| node[below, pos=0.25, font=\footnotesize, text=gray] {residuals reveal $A_q$} (tasks.south);

\node[font=\footnotesize, text=gray] at (0,-1.0) {scalable, low cost};
\node[font=\footnotesize, text=gray] at (15,-1.0) {operational decisions};
\end{tikzpicture}%
}
\caption{Adaptive human-validation allocation pipeline. The AI layer generates cheap predictions for $Q$ analytics tasks, but task-level rectification difficulty $A_q$ is heterogeneous and unknown. Our policy allocates scarce human validation online: each human label both sharpens the estimate and reveals the task's difficulty (dashed feedback), directing effort where an additional label reduces error the most per unit cost.}
\label{fig:pipeline}
\end{figure}

We develop the online Tuned Rectification-Difficulty (TRD) framework for this allocation problem within tuned prediction-powered inference (PPI++). The key object is each unit's \emph{rectification difficulty} $A_q$ \citep{ye2025llm}, the variance of the PPI++ tuned residual $Y_q-\lambda_q^* Y_q^{\mathrm{LLM}}$ \citep{angelopoulos2024ppiplus}, i.e., the human-response variation left unexplained once the LLM prediction is optimally used as a control variate. The difficulty $A_q$ is small when human responses vary little or when the LLM tracks them closely, and large otherwise. If the difficulties $\{A_q\}$ were known, the optimal allocation would follow the classical Neyman square-root rule \citep{neyman1934two}, directing human labels in proportion to $\sqrt{w_q A_q / c_q}$, where $w_q$ is the importance weight and $c_q$ the per-label cost. This heterogeneity is real and large in practice: on the real survey data we study, rectification difficulty ranges from $0.024$ to $0.239$ across the 68 questions (coefficient of variation $0.63$).  The coefficient $\lambda_q^*$, constrained to $[0,1]$, spans $0$ to $0.72$.%

The difficulties are unknown before data collection and must be learned from the same responses used for estimation. Within this framework, we propose an upper confidence bound (UCB) policy, TRD-UCB, that learns the correction coefficients and rectification difficulties online while spending the budget. At each step, it selects the unit with the highest uncertainty-adjusted marginal estimation value per unit cost. Each human label plays a dual role: it sharpens the population estimate for that unit and, through the paired human--LLM residual, reveals that unit's difficulty. The policy needs no prior knowledge of LLM accuracy, is computationally lightweight, and is a drop-in alternative to uniform sampling. Under the stated regularity assumptions, the confidence-radius version of TRD-UCB (Algorithm~\ref{alg:ucb-time-indexed}) achieves $O(\ln B / B^2)$ allocation regret relative to the Neyman oracle, with standard PPI ($\lambda{=}1$) and human-only sampling ($\lambda{=}0$) as nested cases, so the efficiency loss from not knowing the coefficients and difficulties vanishes rapidly as the budget grows.

We compare TRD-UCB with human-only baselines using uniform allocation (the traditional practice) or adaptive TS-Neyman allocation \citep{morikawa2026tsneyman}, as well as PPI++ baselines using uniform allocation or a paid-pilot adaptation of stratified PPI \citep{fisch2024stratified}. Within the TRD framework, we also compare TRD-TS and TRD-$\varepsilon$-Greedy, which use the same rectification difficulty and PPI++ estimator but allocate labels through Thompson sampling (TS) and $\varepsilon$-greedy selection, respectively. In a synthetic design with heterogeneous AI quality and 2,000 labels, TRD-UCB's allocation objective is 3.4\% above the Oracle's, against 36.4\% for uniform allocation, and its realized error is about one quarter below that of uniform allocation with the same AI-augmented estimate and about half that of the traditional human-only practice. Our empirical application uses the 68-question Wave 4 retest battery of the Twin-2K-500 digital-twin survey \citep{toubia2025database}. There, TRD-UCB attains the lowest allocation objective among the implementable policies at every evaluated budget and matches the realized error of uniform allocation and the traditional human-only practice with up to 4.5\% and 6.9\% fewer human labels, respectively. In constructed scenarios where differences in importance, cost, and difficulty make the optimal allocation more uneven, its realized error is lower than that of uniform allocation by up to 18\% when one of them varies and by up to 40\% when they vary together.

\noindent\textbf{Our contributions.} Using AI predictions as a surrogate for costly human labels is now a widely adopted framework \citep{angelopoulos2023prediction, angelopoulos2024ppiplus, zrnic2024cross, mozer2026ppi, ji2025predictions, vafa2025estimating}. The variance that remains after this correction sets the estimator's precision, and through Neyman's rule \citep{neyman1934two} it also guides how scarce human labels are allocated \citep{fisch2024stratified, angelopoulos2025costoptimal, ye2025llm}. Building on this residual variance, the rectification difficulty $A_q$, we develop the online TRD framework for allocating human validation. This paper makes three contributions to the operations--marketing interface.

First, we formulate the allocation of a limited human-validation budget across tasks in AI-enabled analytics, linking estimation and allocation through the TRD framework. Each acquired label corrects its task's AI-based estimate and updates the information used to allocate subsequent labels. The tuned residual variance serves as task-level rectification difficulty, determining the value of additional validation together with task importance, label cost, and the current sample count. The allocation unit can be a survey question, product concept, customer segment, service-audit category, demand-sensing stream, or question module; importance weights $w_q$ reflect the economic value of accuracy for each unit (e.g., margin exposure, churn risk, or launch investment), and label costs $c_q$ can differ across units. The same allocation objective extends to regression and choice models once the PPI++ covariance is scalarized (Appendix~\ref{appsec:M-estimation}).

Second, we develop the TRD-UCB policy within this framework to jointly learn the PPI++ correction coefficients and rectification difficulties from the same acquired labels, allocating each new label to the task with the highest upper confidence bound on its marginal estimation value per unit cost. Because each difficulty is a variance minimized over a coefficient estimated from those labels, we bound its estimation error uniformly over the coefficient and prove that the confidence-radius version of the policy has $O(\ln B/B^2)$ allocation regret relative to the Neyman oracle.

Third, we conduct extensive experiments to evaluate the performance of our proposed method. In synthetic experiments and on a real digital-twin survey, TRD-UCB closes most of the gap between uniform allocation and the Oracle allocation and matches the estimation error of uniform allocation with up to 22.7\% and 4.5\% fewer human labels, respectively. The gains are largest when rectification difficulties, importance weights, and label costs make the optimal allocation uneven.

\subsection{Related Literature}
\label{sec:related}

\textbf{AI-enabled analytics and human validation.} A growing literature studies how LLMs generate, code, or simulate analytics signals (survey responses, open-ended feedback, reviews, and choice predictions) \citep{argyle2023out, brand2023using, ye2025lola, wang2024market, li2024frontiers, peng2025mega, toubia2025database, ziems2024can, motoki2024more, brucks2023prompt}. LLM-augmented surveys are one prominent instance, but the common statistical structure is broader: an AI system provides a low-cost surrogate outcome while human labels provide costly ground truth. Several frameworks have been proposed to combine human labels with LLM predictions, including prediction-powered inference \citep[PPI;][]{angelopoulos2023prediction, angelopoulos2024ppiplus, mozer2026ppi, ji2025predictions, vafa2025estimating}, fine-tuning-then-rectification \citep{wang2025finetune}, and mixed human--LLM designs \citep{broska2025mixed, yin2026synthetic, huang2026many}. These studies focus on how to combine human and LLM data for a given question. The allocation of human labels \emph{across} questions is a complementary decision. Classical Neyman allocation \citep{neyman1934two} assigns samples to strata in proportion to stratum size times the square root of the variance-to-cost ratio, and a rich literature on adaptive survey design shares the philosophy of directing scarce respondent effort where it is most informative \citep[e.g.,][]{toubia2003fast, raghunathan1995split}. In LLM-augmented surveys, Neyman allocation requires the variance of the PPI++ tuned residual for each question, which is unknown before data collection. \citet{ye2025llm} predict these variances for a new survey from historical surveys by meta-learning. Our work learns these tuned residual variances online, with a regret guarantee at the stated theoretical scope. Complementary to our estimation-stage allocation, \citet{gui2025leveraging} use LLM-generated covariates to improve \emph{design-stage} stratification, reducing treatment-effect variance; our framework instead allocates human validation across units once they are defined.

\citet{fisch2024stratified} choose sampling rates from autorater confidence scores, with asymptotic inference guarantees, and \citet{angelopoulos2025costoptimal} derive cost-optimal policies that divide a labeling budget between a cheap AI rater and an expensive human rater when estimating a single mean. We allocate human labels across many weighted tasks with different costs and learn correction coefficients and rectification difficulties while allocating, with an allocation-regret guarantee for the scalar-mean confidence-radius policy.

\textbf{Online learning and bandit algorithms.} Our problem of sequentially allocating a budget across estimation targets under uncertainty connects to the multi-armed bandit literature \citep{auer2002finite, lattimore2020bandit} and its resource-constrained variants \citep[bandits with knapsacks;][]{badanidiyuru2018bandits, agrawal2016linear}. However, these formulations target cumulative reward maximization, whereas our objective is to minimize the terminal aggregate estimation error across a portfolio of estimands. 
We adopt a terminal-MSE-gap regret definition similar to that of \citet{carpentier2015adaptive}, who model stratified sampling as a bandit problem, and \citet{aznag2023active}, who study active allocation for multi-group mean estimation. This multi-group mean estimation problem is the closest methodological precedent: in the special case of equal weights, equal costs, a fixed $\lambda{=}1$, and directly observed group outcomes, our objective reduces to the same Neyman-allocation structure. \citet{morikawa2026tsneyman} propose TS-Neyman, which allocates samples across strata adaptively by posterior draws of the within-stratum variances. Like the methods above, it learns a raw variance; we include it as a human-only benchmark and adapt its posterior sampling to the rectification difficulty as an online variant. Our setting differs in the object being learned and in the operational context. Each costly observation is a paired human--LLM response; the allocation-relevant variance is the optimized PPI++ residual variance rather than a raw group variance; the tuning parameter $\lambda_q^*$ and the rectification difficulty $A_q$ are learned jointly online; and the allocation must accommodate question importance and heterogeneous human-labeling costs.  Our methodological contribution is a confidence construction and allocation-regret analysis that account for learning the coefficients and the difficulties they determine from the same acquired pairs; the regret guarantee applies to scalar means and the confidence-radius policy. In a related vein, \citet{li2026asymptotically} study sequential hypothesis testing with heterogeneous LLMs, and \citet{simchi2025multi} study the trade-off between regret minimization and statistical inference in bandit problems.

\textbf{Human-in-the-loop AI operations.} More broadly, our problem connects to the growing literature on human-in-the-loop AI operations: how firms should direct scarce human oversight across AI-generated outputs when AI reliability is heterogeneous and uncertain. \citet{fugener2026roles} study when AI should automate tasks versus augment human judgment, showing that the optimal division depends on the type of complementarity between human and AI capabilities. \citet{dissorbo2025warnings} study how to direct human attention to AI predictions most likely to be unreliable, using warnings and endorsements to improve human-AI collaboration. \citet{poulidis2025action} compare action recommendations with attention-directing signals as alternative ways for AI to assist human decisions. Related work includes fairness in human-AI collaboration \citep{ge2023rethinking}. This literature largely asks \emph{whether} and \emph{how} a human should engage with a given AI output; we study the complementary upstream allocation decision: given many tasks and a fixed human-validation budget, \emph{where} should scarce oversight be spent when task-level AI reliability is unknown? Our work contributes to this literature by providing an online algorithm, TRD-UCB, that learns where human oversight is most needed and allocates effort accordingly, with a regret guarantee for the confidence-radius policy of Section~\ref{sec:algorithm}.

\noindent\textbf{Organization of the paper.} The remainder of this paper is organized as follows. Section~\ref{sec:problem_formulation} formalizes the problem under PPI++. Section~\ref{sec:algorithm} presents TRD-UCB, our UCB-based online learning policy, and its regret guarantee. Sections~\ref{sec:synthetic} and~\ref{sec:digital_twin} evaluate the approach on synthetic data and on a real digital-twin survey, respectively. Section~\ref{sec:conclusion} concludes. The E-Companion provides theoretical proofs, extensions to general $M$-estimation, implementation details, additional comparisons, and robustness checks.

\section{Problem Formulation}
\label{sec:problem_formulation}

In the following, we formalize a general AI-augmented human-validation allocation problem and use LLM-augmented surveys as the running example. We first introduce the setting in Section~\ref{subsec:survey_setting}. We then review prediction-powered inference  (PPI) and its tuned PPI++ form, highlighting the variance of the  tuned LLM rectification residual, called rectification difficulty, that governs estimation variance in Section~\ref{subsec:ppi_and_difficulty}. Last, we frame our problem as an online learning problem, in which the correction coefficients and rectification difficulties are unknown and are learned sequentially, and define regret relative to the oracle in Section~\ref{subsec:online_learning_pf}.

\subsection{AI-Augmented Human-Validation Setting}
\label{subsec:survey_setting}

Consider a set of allocation units $\mathcal{T}=\{1,\ldots,Q\}$. A unit may be a survey question, product concept, customer segment, service-audit category, content-evaluation task, demand-sensing stream, or question module. For each unit $q\in\mathcal{T}$ and instance $i$, let $X_{q,i}\in\mathcal X$ denote the observable context (e.g., a customer profile, review text, chat transcript, product concept, creative, or survey prompt), and let $Y_{q,i}\in\mathcal Y$ denote the costly human-provided ground truth (e.g., a response, label, audit score, verification judgment, or choice). The label $Y_{q,i}$ is itself a noisy draw, varying across and within labelers, so multiple labels are averaged per unit, and its variance net of the AI signal is the rectification difficulty defined below. Within each unit, observations are i.i.d.\ from a unit-specific population distribution:
$$
(X_{q,i},Y_{q,i}) \overset{\mathrm{i.i.d.}}{\sim} \mathcal P_q.
$$
We write $\cP^X_q$ for the marginal covariate distribution of $X_q$ and $\cP_q(\cdot\mid X)$ for the conditional response distribution, so that $\cP_q = \cP^X_q \otimes \cP_q(\cdot\mid X)$.

The marginal distribution $\cP_q^X$ is the \emph{target} respondent population for unit $q$ (the full survey population, a panel, or a prespecified customer segment). We require that both the labeled human--LLM pairs and the LLM-only predictions are generated from covariates drawn from the \emph{same} target marginal distribution $\cP_q^X$; the LLM-only pool may be far larger than the labeled sample and need not contain the same individuals. This can be enforced operationally by drawing respondent profiles from the same sampling frame, panel, or prespecified customer segment before collecting human responses or querying the LLM. It is satisfied automatically in our digital-twin data, where the LLM is queried with each human respondent's own profile. If instead the LLM-only covariates were drawn from a different population, the PPI++ correction could be biased unless reweighting or other covariate-shift corrections are introduced \citep{angelopoulos2023prediction}; such distribution-shift extensions are orthogonal to the allocation problem we study and lie outside our baseline model. For a subgroup target $G$, $\cP_q^X$ is the covariate distribution conditional on $G$, so
\[
\theta_q^*
=
\bE_{X_q\sim \cP_q^X}\!\left[\bE(Y_q\mid X_q)\right]
=
\bE(Y_q\mid X_q\in G)
\]
is the corresponding segment mean.

For brevity, we consider scalar mean responses and write the estimand as the target-population mean
$\theta_q^* := \bE_{\mathcal P_q}[Y_q]$, where the relevant population is encoded by the marginal covariate distribution $\cP_q^X$. Thus changing $\cP_q^X$ from the full respondent population to a subgroup distribution changes $\theta_q^*$ from an aggregate mean to the corresponding subgroup mean, without changing the allocation formulation.
Given importance weights $\{w_q\}_{q=1}^Q$, reflecting the relative economic value of estimation accuracy for each unit (e.g., margin or shortage-cost exposure in demand sensing, churn risk in service monitoring, launch investment in product testing, or strategic priority in survey design), the goal is to choose sample sizes $\{n_q\}_{q=1}^Q$ to minimize the weighted aggregate MSE:
$$
\min_{\{n_q\}} \sum_{q=1}^Q w_q\,\bE\!\left[(\hat\theta_q-\theta_q^*)^2\right]
\quad \text{s.t.}\quad
\sum_{q=1}^Q c_q\,n_q \le B.
$$
This objective has a decision-theoretic foundation: when the firm's downstream decision for unit $q$ depends on $\hat\theta_q$ and the resulting loss is smooth, a second-order Taylor expansion around the truth shows that, to second order, expected decision regret is proportional to $\sum_q w_q\,\bE[(\hat\theta_q - \theta_q^*)^2]$, where $w_q$ reflects the curvature of unit-level decision loss at the optimum. Minimizing weighted MSE thus serves as a local quadratic approximation to minimizing total decision regret across the portfolio of insight tasks.

The per-label cost for unit $q$ is $c_q>0$, and the total human-validation budget is $B$. Throughout, $q$ is the unit to which human validation is allocated; Table~\ref{tab:applications} lists examples beyond the survey setting used in our empirical sections. The framework extends in two directions. First, the allocation unit can be a \emph{question module} rather than a single question: many surveys group questions into thematic modules and assign each respondent to one or more modules, as in split questionnaire designs \citep{raghunathan1995split}. Second, the estimand can be vector-valued (e.g., choice shares, regression coefficients, or conjoint partworths), with the scalar difficulty derived from a sandwich covariance; Section~\ref{subsec:m_estimation_main} develops this $M$-estimation extension, with the full derivation and an allocation illustration for choice-share estimation in a multinomial logit (MNL) model in Appendix~\ref{appsec:M-estimation}.

\begin{table}[ht]

\centering
\footnotesize
\caption{Examples of AI-enabled human-validation allocation problems.}
\label{tab:applications}
\renewcommand{\arraystretch}{1.15}
\begin{tabular}{>{\raggedright\arraybackslash}p{0.19\linewidth}>{\raggedright\arraybackslash}p{0.17\linewidth}>{\raggedright\arraybackslash}p{0.19\linewidth}>{\raggedright\arraybackslash}p{0.15\linewidth}>{\raggedright\arraybackslash}p{0.15\linewidth}}
\toprule
Application & Allocation unit $q$ & AI prediction layer & Human ground truth & Estimand \\
\midrule
Demand sensing from reviews/chat logs & Product category, region, issue type & LLM-coded sentiment, issue, or demand signal & Human audit label & Mean, proportion \\
\addlinespace
Service-quality monitoring & Queue, agent team, issue class & AI quality score or escalation flag & Human QA audit & Defect or recovery-risk rate \\
\addlinespace
Product / concept testing & Concept, segment, choice task & LLM-predicted rating or choice & Human rating or choice & Mean preference, \emph{MNL/conjoint partworths} \\
\addlinespace
Content / promotion evaluation & Creative, campaign, segment & LLM persuasiveness or brand-safety score & Human evaluation or follow-up & Mean score, \emph{regression coefficients} \\
\addlinespace
Survey / digital-twin study \emph{(this paper)} & Question & LLM-predicted response & Human survey response & Mean, category probability \\
\bottomrule
\end{tabular}
\begin{tablenotes}
\footnotesize
\item Notes: In each instance the AI signal is cheap but of uneven, unknown reliability, human ground truth is costly and noisy, and the decision is where to direct limited human validation. Analytics at the operations--marketing interface is the focus; the same structure (cheap imperfect AI signal, costly noisy human label, $A_q/n_q$ uncertainty) applies more broadly to human-in-the-loop AI validation such as content moderation, AI-output auditing, and fraud/quality review. Rows with vector-valued estimands (italicized) are handled by the $M$-estimation extension (Section~\ref{subsec:m_estimation_main}).
\end{tablenotes}
\end{table}
\renewcommand{\arraystretch}{1.0}

\subsection{Prediction-Powered Inference (PPI and PPI++) and Rectification Difficulty}
\label{subsec:ppi_and_difficulty}

An LLM provides a low-cost predictor $f:\mathcal X\to\mathcal Y$\footnote{The notation $f(X)$ is a shorthand for the LLM prediction generated under the chosen prompting and decoding protocol. It does not require the prediction to be deterministic in $X$: if the LLM call has additional randomness, one may write $f(X,\xi)$, and all expectations, variances, and covariances involving $Y^{\mathrm{LLM}}$ are taken over both $X$ and this LLM randomness. We suppress $\xi$ for brevity.}, yielding a surrogate prediction $Y^{\mathrm{LLM}}_{q,i} := f(X_{q,i})$.  For each question $q$, the researcher has access to two data sources. The labeled sample consists of paired observations $(X_{q,i},Y_{q,i},Y^{\mathrm{LLM}}_{q,i})$, where $X_{q,i}\sim\cP_q^X$, $Y_{q,i}\sim\cP_q(\cdot\mid X_{q,i})$, and the LLM is queried on the same covariates, prompt, or respondent profile so that $Y^{\mathrm{LLM}}_{q,i}=f(X_{q,i})$. The LLM-only pool consists of LLM predictions generated from covariates drawn from the same target marginal distribution $\cP_q^X$, without collecting human responses. This pool need not be one-to-one matched with the labeled respondents, and it is only required to share the same target marginal covariate distribution $\cP_q^X$. Examples include digital-twin surveys, where the LLM is queried using the same respondent profile as the human respondent, and market research panels, where sampled consumer profiles are used both to collect human responses and to prompt the LLM. We write
\[
{
\bar Y_{q,n_q}:=\frac{1}{n_q}\sum_{i=1}^{n_q}Y_{q,i},\qquad
\bar Y^{\mathrm{LLM}}_{q,n_q}:=\frac{1}{n_q}\sum_{i=1}^{n_q}Y^{\mathrm{LLM}}_{q,i},\qquad
\bar Y^{\mathrm{LLM}}_{q,m_q}:=\frac{1}{m_q}\sum_{j=1}^{m_q}Y^{\mathrm{LLM}}_{q,j}.}
\]
 We first recall the original PPI construction and then introduce its tuned PPI++ form.

Prediction-powered inference (PPI) \citep{angelopoulos2023prediction} corrects the LLM-only mean using the discrepancy observed in the labeled sample. For mean estimation, the PPI estimator is
\begin{equation}
\label{eq:ppi_estimator_pf}
\hat{\theta}^{\mathrm{PPI}}_q
=
\bar Y_{q,n_q}
+
\left(\bar Y^{\mathrm{LLM}}_{q,m_q}-\bar Y^{\mathrm{LLM}}_{q,n_q}\right)
=
\bar Y^{\mathrm{LLM}}_{q,m_q}
+
\frac{1}{n_q}\sum_{i=1}^{n_q}\left(Y_{q,i}-Y^{\mathrm{LLM}}_{q,i}\right).
\end{equation}
The correction term removes LLM bias because the labeled-sample LLM mean and the LLM-only mean target the same population expectation. Thus the labeled-sample uncertainty in PPI is governed by the residual $Y_q-Y_q^{\mathrm{LLM}}$.

PPI++ \citep{angelopoulos2024ppiplus} generalizes this construction by allowing the LLM correction to be weighted. For a tuning parameter $\lambda$, it estimates $\theta_q^*=\bE[Y_q]$ by
\begin{equation}
\label{eq:ppipp_estimator_pf}
\hat{\theta}^{\mathrm{PPI++}}_q(\lambda)
=
\bar Y_{q,n_q}
+
\lambda\left(\bar Y^{\mathrm{LLM}}_{q,m_q}-\bar Y^{\mathrm{LLM}}_{q,n_q}\right)
=
\lambda\bar Y^{\mathrm{LLM}}_{q,m_q}
+
\frac{1}{n_q}\sum_{i=1}^{n_q}\left(Y_{q,i}-\lambda Y^{\mathrm{LLM}}_{q,i}\right).
\end{equation}

The first representation shows the human sample mean plus a prediction-powered correction; the second representation shows that the labeled-sample uncertainty is governed by the tuned residual
\[
\tilde Y_q(\lambda):=Y_q-\lambda Y_q^{\mathrm{LLM}}.
\]
For any fixed $\lambda$ and a nonadaptive sample drawn from the target distribution, Eq.~\eqref{eq:ppipp_estimator_pf} is unbiased because the LLM-only mean and the labeled-sample LLM mean have the same population expectation.

For fixed coefficients and counts with independent labeled and synthetic samples, the variance decomposes as
\begin{equation}
\label{eq:variance_decomp_pf}
\Var\!\big(\hat{\theta}^{\mathrm{PPI++}}_q(\lambda)\big)
=
\frac{\lambda^2}{m_q}\Var\!\big(Y^{\mathrm{LLM}}_q\big)
\;+\;
\frac{1}{n_q}\Var\!\big(Y_q-\lambda Y^{\mathrm{LLM}}_q\big).
\end{equation}

 For the labeled-sample component that governs allocation in the SDR regime below, the residual-variance-optimal PPI++ tuning parameter within $[0,1]$ is the projected regression coefficient:
\begin{equation}
\label{eq:lambda_star_pf}
\lambda_q^*
:=
\Pi_{[0,1]}\!\left(\frac{\Cov(Y_q,Y^{\mathrm{LLM}}_q)}{\Var(Y^{\mathrm{LLM}}_q)}\right)
.
\end{equation}
The \textit{rectification difficulty} of question $q$ is the tuned residual variance
\begin{equation}
\label{eq:tuned_difficulty_pf}
A_q
:=
A_q(\lambda_q^*)
:=
\Var\!\big(\tilde Y_q(\lambda_q^*)\big)
=
\Var\!\big(Y_q-\lambda_q^*Y^{\mathrm{LLM}}_q\big)
=\min_{\lambda\in[0,1]}\Var(Y_q-\lambda Y_q^{\mathrm{LLM}}).
\end{equation}
If $\Var(Y_q^{\mathrm{LLM}})=0$, we set $\lambda_q^*=0$, so that $A_q=\Var(Y_q)$. This tuned residual variance sets the labeled-sample variance of the PPI++ estimator \citep{angelopoulos2024ppiplus}, determines the optimal stratum sampling rates of stratified PPI \citep{fisch2024stratified}, and is called rectification difficulty by \citet{ye2025llm}.

We define $A_q$ relative to the chosen coefficient interval $\lambda\in[0,1]$, following the clipping option that \citet[Section~6.2]{angelopoulos2024ppiplus} use for general $M$-estimation. This interval includes the human-only and unit-weight PPI estimators. Appendix~\ref{subsec:finite_upper_bounds} extends the analysis to any fixed finite interval, with its corresponding oracle difficulty; on the Twin pool, $[0,1]$ gives lower mean realized MSE than $[0,2]$ or unrestricted coefficients at the four reported budgets (Appendix~\ref{app:twin-lambda-range}).

$A_q$ is the human-response variance remaining after the chosen AI signal is optimally used within the coefficient range $[0,1]$. It combines differences across respondents, noise in each respondent's own answers, and AI prediction error. Paired observations do not separate these sources, and the allocation does not need to, because the precision of the estimate depends on their total. Its operational role is the same across the examples in Table~\ref{tab:applications}: combined with $w_q$, $c_q$, and the current count, it determines the marginal value of another label.

\noindent\textbf{Why a point predictor, and where calibration fits.} The LLM prediction $f(X)$ enters our formulation only as a \emph{control variate}: the coefficient $\lambda_q^*$ is the variance-minimizing control-variate weight within $[0,1]$ \citep{angelopoulos2024ppiplus}. At the population optimum in the SDR benchmark, PPI++ weakly reduces the labeled-sample variance relative to the human-only estimator, since $\lambda{=}0$ is always feasible. The estimand $\theta_q^*$ and the unbiasedness of the correction do not require the LLM response distribution to be correct; richer LLM use only changes the predictor $f$. In particular, when the LLM is queried multiple times per profile, under conditional independence of decoding randomness and the human response given $X$, replacing the stochastic predictor by the conditional mean $\bE[f(X,\xi)\mid X]$ weakly lowers the population-optimal $A_q$. We treat the LLM as a fixed, pre-trained predictor (a frozen off-the-shelf model in our experiments), which avoids training leakage but does not by itself validate plug-in intervals under adaptive allocation; were $f$ fine-tuned on the same human labels, one would instead use cross-fitting \citep{zrnic2024cross}. \citet{wang2025finetune} study how to split a limited set of human labels between fine-tuning the predictor and rectification.

\noindent Calibration is complementary to this allocation problem rather than a substitute for it. A calibrated, fine-tuned, ensembled, or repeatedly queried LLM simply defines a different prediction protocol and can reduce the rectification difficulty $A_q$. Our problem begins after this protocol is fixed: given many target units and a limited human-validation budget, how should the firm allocate paired human labels when the resulting $A_q$ values are unknown? Improving the predictor and allocating scarce human validation across units are thus orthogonal levers, and our framework addresses the second.
Because LLM-only predictions can be generated cheaply for a large pool of target-sampled covariate profiles, we work in the \textit{synthetic-data-rich} (SDR) regime and set $m_q\to\infty$ for theoretical analysis. This isolates the human-label allocation problem: the synthetic pool size $m_q$ is not a decision variable and can be chosen independently of the human-label allocation $\{n_q\}$. As $m_q\to\infty$, the first term in Eq.~\eqref{eq:variance_decomp_pf} vanishes and the tuned PPI++ estimation variance reduces to
$$
\Var\!\big(\hat\theta^{\mathrm{PPI++}}_q(\lambda_q^*)\big) = \frac{A_q}{n_q},
$$
which implies that optimal survey allocation is governed by $\{A_q\}$.

\subsection{Online Learning Framework and Regret}
\label{subsec:online_learning_pf}

We now frame the budget allocation problem as an online learning problem. Under the SDR benchmark of Section~\ref{subsec:ppi_and_difficulty}, the design objective from Section~\ref{subsec:survey_setting} reduces to minimizing $\sum_q w_q \Var(\hat\theta^{\mathrm{PPI++}}_q(\lambda_q^*)) = \sum_q w_q A_q / n_q$. If the rectification difficulties $\{A_q\}$ were known, this becomes
$$
\min_{\{n_q\}} \sum_{q=1}^Q w_q \frac{A_q}{n_q}
\quad \text{s.t.}\quad
\sum_{q=1}^Q c_q n_q \le B.
$$
The optimal solution follows a Neyman-type square-root allocation rule:
\begin{equation}
\label{eq:offline_oracle_pf}
n_q^*
=
\frac{B}{\sum_{j=1}^Q \sqrt{w_j A_j c_j}}
\sqrt{\frac{w_q A_q}{c_q}},
\qquad q\in\mathcal T.
\end{equation}
Thus, questions that are more important (large $w_q$), harder to rectify (large $A_q$), or cheaper (small $c_q$) receive more human labels. We refer to Eq.~\eqref{eq:offline_oracle_pf} as the oracle allocation.

However, this oracle is not directly implementable because it requires knowledge of $\{A_q\}$. Since $A_q=\Var(Y_q-\lambda_q^*Y^{\mathrm{LLM}}_q)$ and $\lambda_q^*$ is itself unknown, the difficulty depends on the joint distribution of human responses and LLM predictions, is unknown before any human data are collected, and varies across questions, populations, and LLM models. One could run a pilot survey to estimate $\{A_q\}$ before committing to an allocation, but this requires choosing how much of the budget to spend on the pilot versus the main study, and pilot estimates from a small sample may be unreliable. Our online formulation instead uses each acquired pair to update the selected question's coefficient and rectification difficulty throughout acquisition. These updated estimates and their confidence bounds determine subsequent acquisitions. Thus the observations used to learn the allocation inputs depend on earlier allocation decisions, and the regret analysis must account for this feedback.

We model data collection as a sequential allocation process. After $t$ human labels have been acquired, a policy selects the next question $q_t\in\mathcal T$ based on the available observations. It acquires one additional human response for that question (paying cost $c_{q_t}$) and queries the LLM on the same covariates, obtaining the $(t+1)$st acquired pair $(Y_{q_t},Y^{\mathrm{LLM}}_{q_t})$. The new pair is added to that question's data, its coefficient is refitted as $\hat\lambda_{q_t,t+1}$, and all its residuals are recomputed using that coefficient. Let $n_{q,t}$ be the number of human--LLM pairs collected for question $q$ up to time $t$, and let $n_{q,B}$ be the final allocation to question $q$ when the algorithm terminates.

An adaptive policy $\pi=\{\pi_t\}$ uses $\pi_t$ to map the history $h_t$ of the first $t$ acquired pairs, including initialization, to the next question $q_t$. Because $\{A_q\}$ are unknown, the policy faces an exploration--exploitation trade-off: sampling a question to reduce uncertainty in its difficulty estimate versus allocating budget to questions that appear most impactful under current estimates.

We measure performance by regret, which is the gap between the expected terminal objective achieved by $\pi$ and that of the oracle allocation:
$$
\mathcal R(B)
:=
\bE_{\pi}\!\left[\sum_{q=1}^Q w_q\frac{A_q}{n_{q,B}}\right]
-
\sum_{q=1}^Q w_q\frac{A_q}{n_q^*}.
$$
The goal is to design an online allocation policy with vanishing terminal regret as $B$ grows. The next section develops TRD-UCB, a UCB-based algorithm that learns $\{A_q\}$ from PPI++ tuned residuals and achieves $\mathcal R(B)=O(\ln B/B^2)$ under the assumptions of Section~\ref{subsec:regret_analysis}.

We note that the LLM augmentation and the adaptive allocation address complementary aspects of the problem. In the PPI++ setting, the coefficient $\lambda_q^*$ determines the rectification difficulty $A_q$ for each question, and the adaptive policy learns both from the acquired pairs and assigns human labels accordingly.  From this point onward, the paper formulates and analyzes the allocation problem under the PPI++ setting, and $A_q$ always denotes the tuned PPI++ difficulty $A_q^{\mathrm{PPI++}}=\Var(Y_q-\lambda_q^*Y_q^{\mathrm{LLM}})$. The adaptive allocation problem then uses the single objective $\sum_q w_q A_q/n_q$, with heterogeneity in these tuned difficulties driving the value of online learning. %

\subsection{Beyond Means: General \texorpdfstring{$M$}{M}-Estimation Targets}
\label{subsec:m_estimation_main}

The scalar mean developed above is our running formulation and the cleanest way to present the allocation logic, but many analytics targets are decision models (category probabilities, demand regressions, service-quality response models, conjoint/choice partworths) written as $M$-estimators. For unit $q$, let $\bm\theta_q^*\in\argmin_{\bm\theta}\bE[\ell_q(X_q,Y_q;\bm\theta)]$ with score $\bm\psi_q=\nabla_{\bm\theta}\ell_q$, and AI surrogate score $\bm\psi_q^{\mathrm{AI}}(X;\bm\theta)=\bm\psi_q(X,f_q(X);\bm\theta)$. The PPI++ estimating equation corrects the human score by the AI score:
\begin{equation}
\label{eq:ppipp_mest_main}
\bm 0=\frac{1}{n_q}\sum_{i}\bm\psi_q(X_{q,i},Y_{q,i};\bm\theta)+\lambda_q\Big(\bE[\bm\psi_q^{\mathrm{AI}}(X_q;\bm\theta)]-\frac{1}{n_q}\sum_i\bm\psi_q^{\mathrm{AI}}(X_{q,i};\bm\theta)\Big).
\end{equation}
Under standard regularity and the synthetic-data-rich benchmark, $\sqrt{n_q}(\widehat{\bm\theta}_q-\bm\theta_q^*)\Rightarrow \mathcal N(\bm 0,\bm H_q^{-1}\bm V_q^{\Delta}(\lambda_q)\bm H_q^{-\top})$, where $\bm H_q$ is the population Hessian and $\bm V_q^{\Delta}(\lambda_q)=\Var(\bm\psi_q(X_q,Y_q;\bm\theta_q^*)-\lambda_q\bm\psi_q^{\mathrm{AI}}(X_q;\bm\theta_q^*))$. The $1/n_q$ structure is preserved: for any degree-one homogeneous scalar design criterion $g$ (a weighted trace for A-optimality, the $d$th root of the determinant for D-optimality), define the scalar difficulty $A_q=\min_{\lambda\in[0,1]}g(\bm H_q^{-1}\bm V_q^{\Delta}(\lambda)\bm H_q^{-\top})$. The oracle again minimizes $\sum_q w_q A_q/n_q$ under the validation budget, so the index of Section~\ref{sec:algorithm} can be computed in the same way once rectification difficulty is computed from a scalarized PPI++ sandwich covariance instead of a scalar residual variance; its confidence radius and regret guarantee do not yet cover this case. Appendix~\ref{appsec:M-estimation} gives the full derivation and illustrates the allocation for choice-share estimation in an MNL choice model.

The regret theorem in Section~\ref{sec:algorithm} is stated for scalar means, where finite-sample concentration is transparent. For general $M$-estimators the oracle structure and UCB implementation are identical once confidence radii for the scalar difficulty estimates are available; establishing those radii remains open.

\section{Algorithm Design and Regret Analysis}
\label{sec:algorithm}

This section develops TRD-UCB, the policy that implements the TRD framework of Section~\ref{sec:problem_formulation} with an upper confidence bound. We first motivate the algorithm design through a marginal-efficiency interpretation of the oracle allocation in Section~\ref{subsec:oracle_marginal}, then construct the UCB index that replaces unknown difficulties with confidence-adjusted estimates in Section~\ref{subsec:ucb_index}, and present the full algorithm in Section~\ref{subsec:algorithm_description}. Section~\ref{subsec:regret_analysis} establishes the regret guarantee.

\subsection{Oracle Case with Known \texorpdfstring{$A_q$}{Aq}: Marginal-Efficiency Interpretation}
\label{subsec:oracle_marginal}
Recall from Section~\ref{sec:problem_formulation} that $\Var(\hat\theta^{\mathrm{PPI++}}_q(\lambda_q^*))= A_q/n_q$ with $A_q=\Var(Y_q-\lambda_q^*Y_q^{\mathrm{LLM}})$. To motivate our adaptive approach, we first analyze the case where the rectification difficulties $A_q$ are known, but the budget is allocated sequentially. We frame this through the lens of \textit{Marginal Efficiency}.

 To build intuition, consider two survey questions with $A_1=4$, $A_2=1$ and equal weights and costs ($w_q=c_q=1$). The oracle allocates $n_1^*=2n_2^*$ (proportional to $\sqrt{A_q}$). Suppose both questions currently have $n_q=10$ samples. The marginal efficiency of adding one more sample to question~1 is $A_1/n_1^2 = 4/100 = 0.04$, versus $A_2/n_2^2 = 1/100 = 0.01$ for question~2. The greedy rule correctly directs the next sample to the harder question.

More generally, consider the current state at round \( t \), where each unit \( q \) has been allocated \( n_{q,t} \) samples. If we were to allocate one additional sample to unit \( q \), the objective function \( \sum_q w_q \frac{A_q}{n_q} \) would decrease. The \textit{Marginal Variance Reduction} is given by the negative partial derivative with respect to \( n_q \): \[ -\frac{\partial}{\partial n_q} \left( w_q \frac{A_q}{n_q} \right) = \frac{w_q A_q}{n_q^2} \]

Since each sample for unit \( q \) costs \( c_q \), we define the \textit{Marginal Efficiency Index} \( \mathcal{I}_{q,t} \) as the variance reduction per unit of currency spent: \[ \mathcal{I}_{q,t}^* = \frac{\text{Marginal Variance Reduction}}{\text{Marginal Cost}} = \frac{w_q A_q}{c_q n_{q,t}^2}. \] A policy that greedily selects $q = \arg\max \mathcal{I}_{q,t}^*$ at each step converges to the optimal oracle allocation $n_q^*$, because the greedy rule drives the system toward a state where the marginal efficiency $w_q A_q/(c_q n_q^2)$ is equalized across all $q$, which is precisely the first-order condition of the Lagrangian of the static optimization problem. Operationally, $\mathcal{I}_{q,t}^*$ measures the marginal value of one additional unit of human oversight per dollar spent on unit $q$, and the oracle allocation equalizes it across the portfolio of insight tasks.

\subsection{UCB Index Construction}
\label{subsec:ucb_index}

In practice, $A_q$ is unknown. The key role of the LLM is to produce cheap paired predictions $f(X)$ so that each costly human label immediately yields information about the tuned residual $\tilde Y(\lambda)=Y-\lambda f(X)$.  In the PPI++ construction used here, both the tuning coefficient and the tuned residual variance are learned online from the same paired observations. Let $t$ denote the total number of acquired human labels, including initialization, and let $n_{q,t}$ be the count for unit $q$. Define
\begin{equation}
\label{eq:lambda_hat_main}
\hat\lambda_{q,t}
:=
\Pi_{[0,1]}\!\left(
\frac{\widehat{\Cov}_t(Y_q,Y_q^{\mathrm{LLM}})}{\widehat{\Var}_t(Y_q^{\mathrm{LLM}})}
\right),
\end{equation}
where $\widehat{\Cov}_t$ and $\widehat{\Var}_t$ are empirical centered moments based on the $n_{q,t}$ paired samples collected for unit $q$ (using denominator $n_{q,t}$ in the covariance/variance ratio). If the empirical predictor variance is zero, we set $\hat\lambda_{q,t}=0$. All sampled predictions are then equal, so their empirical covariance with the human responses is also zero and every coefficient gives the same residual variance. In either case, $\hat\lambda_{q,t}$ minimizes the sample residual variance over $[0,1]$; using denominator $n_{q,t}-1$ instead of $n_{q,t}$ in both moments leaves the ratio unchanged.

Let $\widehat A_{q,t}(\hat\lambda_{q,t})$ be the sample variance of the tuned residuals:
\begin{equation}
\label{eq:Ahat_ppipp_main}
\widehat{A}_{q,t}(\hat\lambda_{q,t})
=
\frac{1}{n_{q,t} - 1}
\sum_{i=1}^{n_{q,t}}
\left(
\tilde{Y}_{q,i}(\hat\lambda_{q,t})
-
\bar{\tilde{Y}}_{q,t}(\hat\lambda_{q,t})
\right)^2.
\end{equation}

A naive plug-in greedy approach using the empirical variance $\widehat{A}_{q,t}(\hat\lambda_{q,t})$ can under-sample a unit whose true $A_q$ is high but whose initial $\widehat{A}_{q,t}(\hat\lambda_{q,t})$ was accidentally low.

To address this, we use the Upper Confidence Bound (UCB) principle. The PPI++ confidence radius $\rho^{++}_{q,t}:=r(t,n_{q,t})$, defined in Eq.~\eqref{eq:rho_ppipp_step1_abs}, controls the difference between the minimum empirical and population residual standard deviations. Both minima are over the same coefficient interval. Because the fitted coefficient attains the empirical minimum, this bound accounts for its data dependence without requiring a separate concentration bound for $\hat\lambda_{q,t}$ around $\lambda_q^*$.

On the event $\mathcal G_t$ established in Lemma~\ref{lem:good_event_ppipp_abs},
\[
A_q \leq A^{\mathrm{UCB}}_{q,t}(\hat\lambda_{q,t}),
\qquad
A^{\mathrm{UCB}}_{q,t}(\hat\lambda_{q,t})
:=
\left(
\sqrt{\widehat A_{q,t}(\hat\lambda_{q,t})}
+
\rho^{++}_{q,t}
\right)^2 .
\]
We define the TRD-UCB index to prioritize units that either have high estimated difficulty (exploitation) or high estimation uncertainty (exploration): \[ \mathcal{I}_{q,t}^{\mathrm{UCB}} = \frac{w_q A^{\mathrm{UCB}}_{q,t}(\hat\lambda_{q,t})}{c_q n_{q,t}^2} . \]

 The confidence bound above contains the question-dependent $\widehat{A}_{q,t}(\hat\lambda_{q,t})$: when a unit's estimated difficulty is small, its optimistic index falls quickly as its count grows, and the next labels go to units whose estimated difficulty or remaining uncertainty is larger. No unit is abandoned: as the other counts grow, the index of an easy unit rises again relative to theirs, so every unit retains a positive share, in line with the Neyman rule that allocates proportionally to $\sqrt{A_q}$.

\subsection{Algorithm Description}
\label{subsec:algorithm_description}

TRD-UCB (Algorithm~\ref{alg:ucb-time-indexed}) first acquires a fixed number $K\ge2$ of labels per unit to make $\widehat A_{q,t}(\hat\lambda_{q,t})$ well-defined, then repeatedly allocates the next human label to the unit with the largest UCB index. Each allocation uses the LLM on the same covariate/prompt to produce a paired prediction, refits the PPI++ tuning coefficient for that question, and updates the tuned residual variance. The confidence radius for every unit is recomputed at each selection time using the current $t$. It then proceeds sequentially until the remaining budget falls below $c_{\max}:=\max_q c_q$.

\begin{algorithm}[ht]
\caption{TRD-UCB: UCB-Based Online Human-Validation Allocation with PPI++}
\label{alg:ucb-time-indexed}
\begin{algorithmic}[1]
\State \textbf{Input:} Budget $B\ge K\sum_qc_q$, costs $\{c_q\}_{q=1}^Q$, weights $\{w_q\}_{q=1}^Q$, fixed confidence $\delta_0\in(0,1)$, fixed integer $K\ge2$, LLM predictor $f$, coefficient interval $[0,1]$.
\State \textbf{Confidence inputs:} Known outcome and prediction ranges from Assumption~\ref{ass:boundedpp}, used in $r(t,s)$ of Eq.~\eqref{eq:rho_ppipp_step1_abs}.
\For{each unit $q \in \{1, \dots, Q\}$} \Comment{Initialization Phase}
    \State Sample $K$ initial covariates $\{X_{q,i}\}_{i=1}^K \sim \cP^X_q$.
    \State Collect paired samples $\{(Y_{q,i},Y^{\mathrm{LLM}}_{q,i})\}_{i=1}^K$ where $Y^{\mathrm{LLM}}_{q,i}=f(X_{q,i})$.
    \State Fit $\hat\lambda_{q,QK}=\Pi_{[0,1]}\!\big(\widehat{\Cov}_{q,K}(Y_q,Y^{\mathrm{LLM}}_q)/\widehat{\Var}_{q,K}(Y^{\mathrm{LLM}}_q)\big)$, where the empirical moments use the $K$ pairs for unit $q$.
    \State Compute tuned residuals $\tilde Y_{q,i}(\hat\lambda_{q,QK})=Y_{q,i}-\hat\lambda_{q,QK}Y^{\mathrm{LLM}}_{q,i}$ for $i\le K$, set $n_{q,QK}=K$, and compute $\widehat A_{q,QK}(\hat\lambda_{q,QK})$ by Eq.~\eqref{eq:Ahat_ppipp_main}.
\EndFor
\State Set $t\gets QK$ and $B_t\gets B-\sum_{q=1}^Q c_q K$; use the zero-variance convention in Eq.~\eqref{eq:lambda_hat_main} for every fit.
\For{$t = QK, QK{+}1, \dots$ until $B_t < \max_q c_q$} \Comment{UCB Selection Phase}
    \State \textbf{Selection through UCB Indices:}
    \State  \quad For each $q$, set $\rho^{++}_{q,t}=r(t,n_{q,t})$ and compute $$\mathcal{I}_{q,t}^{\mathrm{UCB}} = \frac{w_q \cdot A_{q,t}^{\mathrm{UCB}}(\hat\lambda_{q,t})}{c_q n_{q,t}^2},\qquad
    A_{q,t}^{\mathrm{UCB}}(\hat\lambda_{q,t})=\left(\sqrt{\widehat A_{q,t}(\hat\lambda_{q,t})}+\rho^{++}_{q,t}\right)^2.$$
    \State \quad Select $q_t = \arg\max_{q} \mathcal{I}_{q,t}^{\mathrm{UCB}}$, breaking ties by a fixed ordering.
    \State \textbf{Query \& LLM Pair:}
    \State \quad {Sample} a new covariate $X_{q_t} \sim \cP^X_{q_t}$ from the target population.
    \State \quad Query human label $Y_{q_t}\sim\cP_{q_t}(\cdot\mid X_{q_t})$
    \State \quad {Generate} {LLM prediction} $Y_{q_t}^{\mathrm{LLM}} = f(X_{q_t})$ for the same $X_{q_t}$.

    \State \textbf{Update:}
    \State \quad Let $n_{q_t, t+1} = n_{q_t, t} + 1$ and $B_{t+1} = B_{t} - c_{q_t}$.
    \State \quad Refit $\hat\lambda_{q_t,t+1}=\Pi_{[0,1]}\!\big(\widehat{\Cov}_{q_t,t+1}(Y_{q_t},Y^{\mathrm{LLM}}_{q_t})/\widehat{\Var}_{q_t,t+1}(Y^{\mathrm{LLM}}_{q_t})\big)$ using all $n_{q_t,t+1}$ pairs for unit $q_t$.
    \State \quad Recompute tuned residuals $\tilde Y_{q_t,i}(\hat\lambda_{q_t,t+1})=Y_{q_t,i}-\hat\lambda_{q_t,t+1}Y^{\mathrm{LLM}}_{q_t,i}$ for $i\le n_{q_t,t+1}$, then update $\widehat A_{q_t,t+1}(\hat\lambda_{q_t,t+1})$.
    \State Carry each unselected unit's count and fitted statistics forward to $t+1$.
\EndFor
\State \textbf{Output:} For each $q$, report the  PPI++ estimate $\hat\theta^{\mathrm{PPI++}}_q=\bar Y_{q,n_q}+\hat\lambda_{q,B}\big(\bar Y^{\mathrm{LLM}}_{q,m_q}-\bar Y^{\mathrm{LLM}}_{q,n_q}\big)$ with $m_q\gg n_q$, where $\hat\lambda_{q,B}$ is the last updated estimate of $\lambda_q^*$ for unit $q$ at termination.
\end{algorithmic}
\end{algorithm}
For exposition, the policy updates after each human label. In batched operations (daily audit queues, labeling or survey waves, multi-round concept tests), the same index is recomputed at batch boundaries to allocate the next batch, so real-time per-instance adaptation is not required.

\subsection{Regret Analysis}
\label{subsec:regret_analysis}

We now state the main theoretical guarantee for TRD-UCB (Algorithm~\ref{alg:ucb-time-indexed}). Recall $A_q:=A_q(\lambda_q^*)=\Var(Y_q-\lambda_q^*Y_q^{\mathrm{LLM}})$. Let $n_{q,B}$ be the total number of human labels allocated to question $q$ when the algorithm stops under budget $B$. Define $c_{\min}:=\min_q c_q$, $c_{\max}:=\max_q c_q$, and $T_{\max}:=\lfloor B/c_{\min}\rfloor$. The oracle allocation is
$$
n_q^*=\frac{B}{\sum_{j=1}^Q \sqrt{w_jA_jc_j}}\sqrt{\frac{w_qA_q}{c_q}},\qquad \Lambda:=\sum_{j=1}^Q \sqrt{w_jA_jc_j},\qquad \beta_q:=\frac{n_q^*}{B}=\frac{\sqrt{w_qA_q/c_q}}{\Lambda}.
$$
The regret is defined as
\begin{equation}
\label{eq:regret_def}
\mathcal{R}(B):=\bE\!\left[\sum_{q=1}^Q \frac{w_qA_q}{n_{q,B}}\right]-\sum_{q=1}^Q \frac{w_qA_q}{n_q^*}.
\end{equation}

We impose two assumptions for the theoretical analysis of Algorithm~\ref{alg:ucb-time-indexed}. Throughout, we assume $w_q>0$ and $c_q>0$ for all $q$, and that the budget is large enough for initialization: $B\ge \sum_{q=1}^Q c_q K$.

\begin{assumption}%
\label{ass:boundedpp}
For each $q\in[Q]$,
$$
Y_{q}\in[L_Y,U_Y],\qquad Y^{\mathrm{LLM}}_{q}\in[L_{Y^{\mathrm{LLM}}},U_{Y^{\mathrm{LLM}}}]\qquad\text{almost surely.}
$$
Define $R_Y:=U_Y-L_Y$, $R_{Y^{\mathrm{LLM}}}:=U_{Y^{\mathrm{LLM}}}-L_{Y^{\mathrm{LLM}}}$, $R:=R_Y+R_{Y^{\mathrm{LLM}}}$.
Then $\tilde Y_q(\lambda)$ has range at most $R$ for every $\lambda\in[0,1]$.
\end{assumption}

\begin{assumption}%
\label{ass:nondeg_ppipp}
We assume $A_{\min}:=\min_q A_q>0$, where $A_q=\min_{\lambda\in[0,1]}\Var(Y_q-\lambda Y_q^{\mathrm{LLM}})$.
\end{assumption}
Assumption~\ref{ass:boundedpp} requires outcomes and LLM predictions to be bounded, which ensures uniform concentration of the tuned residual variance over $\lambda\in[0,1]$. This holds for bounded survey responses and LLM-coded predictions, as in Likert scales, binary choices, or bounded-range ratings.
Assumption~\ref{ass:nondeg_ppipp} requires positive residual variance for every question after optimal correction within $[0,1]$, ensuring positive oracle allocation shares.

Because the adaptive loop of Algorithm~\ref{alg:ucb-time-indexed} stops once the remaining budget falls below $c_{\max}$, every question is affordable at every selection time, with no extra assumption on heterogeneous costs.

With these assumptions, we have the following theoretical guarantee for the regret of Algorithm~\ref{alg:ucb-time-indexed}.
\begin{theorem}
\label{thm:regret}

\par\begingroup

Suppose Assumptions~\ref{ass:boundedpp}--\ref{ass:nondeg_ppipp} hold. Hold $Q$, the positive weights and costs, the question distributions, and the outcome and prediction bounds fixed. Run Algorithm~\ref{alg:ucb-time-indexed} with fixed $\delta_0\in(0,1)$ and fixed integer $K\ge2$. There exist finite constants $C$ and $B_0$, depending only on these fixed quantities, such that for every $B\ge B_0$,
\[
\mathcal R(B)\le C\,\frac{\ln\!\big(2Q(T_{\max}+1)^5/\delta_0\big)}{B^2}
=O\!\left(\frac{\ln B}{B^2}\right).
\]

\par\endgroup
\end{theorem}

The same regret order holds for any fixed finite coefficient interval, provided every question has positive optimal residual variance within that interval (Appendix~\ref{subsec:finite_upper_bounds}). The proof proceeds in three steps. We present the key constructions and results here; complete algebraic details are in Appendix~\ref{appsec:proof_thm1}.

\subsubsection*{Step~I: PPI++ Good Event and Optimism}

\par\begingroup
The fitted coefficient minimizes the empirical residual variance, just as $\lambda_q^*$ minimizes the population residual variance over $[0,1]$. The difference between the two minimum standard deviations is at most their largest pointwise difference. Uniform concentration over coefficients and deterministic sample counts therefore yields events $\mathcal G_t$ with $\bP(\mathcal G_t^c)\le\delta_0/(t+1)^3$. Lemma~\ref{lem:good_event_ppipp_abs} also gives
\[
\bP\!\left(\bigcup_{t\ge m}\mathcal G_t^c\right)\le\frac{\delta_0}{2m^2},\qquad m\ge2.
\]
\par\endgroup

\paragraph{Optimism.} On $\mathcal G_t$, the confidence radius $\rho^{++}_{q,t}$ satisfies
$$
\sqrt{A_q}\le \sqrt{\widehat A_{q,t}(\hat\lambda_{q,t})}+\rho^{++}_{q,t},
$$
and hence $A_q\le A^{\mathrm{UCB}}_{q,t}(\hat\lambda_{q,t})$. Therefore the true marginal-efficiency index is upper-bounded by the optimistic index:
$$
\mathcal{I}^*_{q,t} = \frac{w_q A_q}{c_q n_{q,t}^2} \le \mathcal{I}^{\mathrm{UCB}}_{q,t} := \frac{w_q A^{\mathrm{UCB}}_{q,t}(\hat\lambda_{q,t})}{c_q n_{q,t}^2}.
$$
This establishes optimism; allocation tracking additionally uses the selection comparison and stopping rule stated below.

\subsubsection*{Step~II: Allocation Gap Bound}

Write $C_0:=K\sum_qc_q$, $D_B:=B-C_0-c_{\max}$, and $\beta_{\min}:=\min_q\beta_q>0$. The stopping rule gives a question $q_a$ with $n_{q_a,B}-K\ge\beta_{q_a}D_B$. Just before its last adaptive acquisition, the total count $t_a$ satisfies $t_a\ge\beta_{\min}D_B$. Define $m_B:=\lfloor\beta_{\min}D_B\rfloor$ and $\mathcal H_B:=\bigcap_{t\ge m_B}\mathcal G_t$. For sufficiently large $B$, $\bP(\mathcal H_B^c)=O(B^{-2})$. On $\mathcal H_B$, optimism holds at $t_a$, where every candidate is affordable, and the UCB selection comparison yields
\begin{equation}
\label{eq:allocation_gap_bound}
|n_{q,B} - n_q^*|
=
O\!\left(\sqrt{B\ln B}\right),\qquad \forall q\in[Q].
\end{equation}

\subsubsection*{Step~III: From Allocation Gaps to Regret}

We convert the allocation gaps into regret via a second-order Taylor expansion of $f_q(n)=w_qA_q/n$ around $n_q^*$. Summing over questions:
\begin{equation*}
\sum_{q=1}^Q\left[f_q(n_{q,B})-f_q(n_q^*)\right] = \underbrace{-\sum_{q=1}^Q\frac{w_qA_q}{(n_q^*)^2}(n_{q,B}-n_q^*)}_{\text{Term 1}} + \underbrace{\sum_{q=1}^Q \frac{w_qA_q}{\xi_q^3}(n_{q,B}-n_q^*)^2}_{\text{Term 2}}.
\end{equation*}
Term~1 telescopes by the optimality condition $w_qA_q/(n_q^*)^2=(\Lambda^2/B^2)c_q$, yielding $\text{Term 1}=(\Lambda^2/B^2)\mathrm{Rem}_B$ where $\mathrm{Rem}_B\in[0,c_{\max})$ is the leftover budget, contributing $O(1/B^2)$. Term~2 is bounded using Eq.~\eqref{eq:allocation_gap_bound}: each squared gap is $O(B\ln B)$, and the $1/\xi_q^3\approx 1/(n_q^*)^3\sim 1/B^3$ prefactor yields a total of $O(\ln B/B^2)$. For sufficiently large $B$, Eq.~\eqref{eq:allocation_gap_bound} ensures $\xi_q\ge n_q^*/2$ on $\mathcal H_B$. On its complement, the objective gap is bounded by $V_{\max}:=\sum_qw_qA_q/K$. Thus $\bP(\mathcal H_B^c)=O(B^{-2})$ contributes only $O(B^{-2})$ to expected regret.

\begingroup
\section{Synthetic Data Analysis}
\label{sec:synthetic}

We first study a controlled setting with rectification difficulties known only for evaluation. Each condition uses 500 replications, and all policies see the same draws within a replication.

\subsection{Experiment Setup}
\label{subsec:synthetic_setup}

We simulate $Q=100$ survey questions with equal normalized weights $w_q=1/Q$ and unit costs $c_q=1$. Question $q$ has variance scale
\begin{equation}
\label{eq:var_loguni}
v_q=\frac{h}{\sinh h}e^{U_q},\qquad U_q\overset{\mathrm{i.i.d.}}{\sim}\mathrm{Uniform}(-h,h).
\end{equation}
The normalization keeps $\bE[v_q]=1$ as $h$ changes, while the dispersion of $v_q$ increases with $h$. At the baseline $h=2$, the coefficient of variation of $v_q$ is approximately 1.04. Each replication draws an independent collection of question characteristics. For each question, paired human and AI observations follow
\begin{align}
S&{}=\sqrt{v_q}\bigl(\rho_q Z+\sqrt{1-\rho_q^2}\,\eps\bigr),\nonumber\\
Y&=\sqrt{v_q}(Z+\sigma_\eta\eta),\nonumber
\end{align}

where $Z,\eps,\eta$ are independent standard normals, $\sigma_\eta=0.5$, and $\rho_q\overset{\mathrm{i.i.d.}}{\sim}\mathrm{Beta}(2,1)$ at baseline. Here $S$ denotes the AI prediction $Y_q^{\mathrm{LLM}}$. The target mean is $\theta_q^*=0$. Before the observations are clipped to $[-10,10]$, the optimal coefficient is $\lambda_q^*=\rho_q$ and the rectification difficulty is $A_q=v_q(1-\rho_q^2+\sigma_\eta^2)$.

\begin{table}[ht]
\centering\footnotesize
\caption{Baseline synthetic design. Each condition uses 500 independent replications.}
\label{tab:params}
\begin{tabular}{lll}\toprule
Feature & Baseline & Sensitivity or role\\\midrule
Questions and weights & $Q=100$, $w_q=1/Q$, $c_q=1$ & Weighted per-question loss\\
AI quality & $\rho_q\sim\mathrm{Beta}(2,1)$ & Mean $0.3,0.5,2/3,0.8,0.9$ at concentration 3\\
Variance heterogeneity & $h=2$ & $h=0.5,1,1.5,2$\\
Human noise & $\sigma_\eta=0.5$ & Fixed across conditions\\
Paid initialization & $K=3$ labels per question & Fixed across conditions\\
Budget & $B\le2{,}000$ & Sensitivities at $B=2{,}000$\\
Policy parameters & $\gamma=0.08$; $\nu_0=2$; $c=5$; $\alpha=0.5$ & Fixed across conditions; selection: Appendix~\ref{app:benchmark-details}\\
\bottomrule\end{tabular}
\end{table}

\begingroup
\begin{figure}[ht]\centering\includegraphics[width=.98\textwidth]{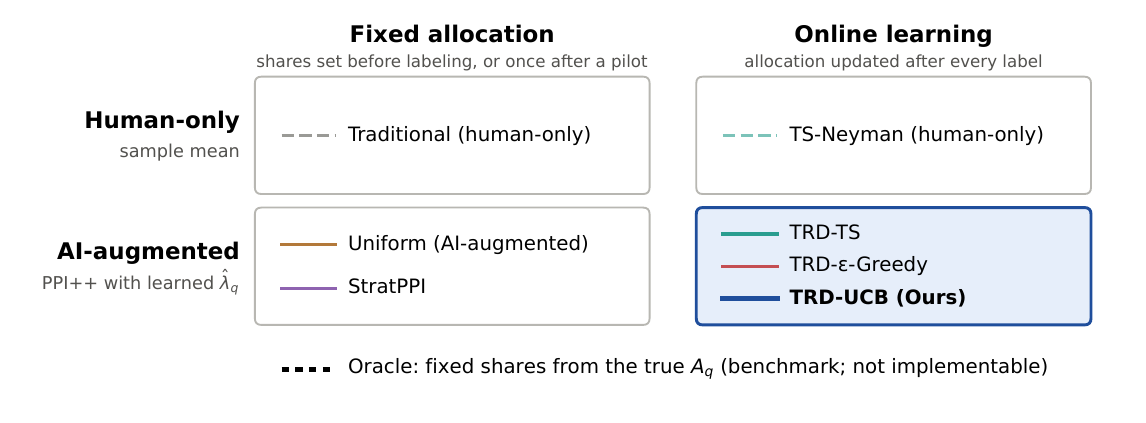}
\caption{Methods compared in Sections~\ref{sec:synthetic} and~\ref{sec:digital_twin}. Rows distinguish human-only and AI-augmented estimation; columns distinguish allocation rules. The TRD policies share the residual-variance criterion and differ in exploration. Line styles are used in subsequent comparison figures.}\label{fig:framework}
\end{figure}
\paragraph{Policy definitions.} Figure~\ref{fig:framework} organizes the methods by their allocation rule and final estimator. Comparisons with human-only baselines assess the combined gains from AI-based correction and allocation, whereas comparisons with Uniform retain the same PPI++ estimator and assess allocation gains. The three TRD policies share the residual-variance criterion and differ in their exploration rules. Below, Traditional and TS-Neyman denote the human-only methods, and Uniform denotes uniform allocation with PPI++ estimation.

All policies start with $K=3$ paid labels per question. Let $n_q$ denote the current sample count, $\widehat A_q$ the sample variance of the fitted residuals $Y-\hat\lambda_qS$ (Eq.~\eqref{eq:Ahat_ppipp_main}), and $b$ the budget spent so far. Human-only methods report the sample mean; AI-augmented methods report $\bar Y_q+\hat\lambda_q(\mu_{S,q}-\bar S_q)$, with the coefficient fitted using all acquired pairs and $\mu_{S,q}$ the target-population AI mean.

Traditional and Uniform acquire the next label from a question with the fewest observations and differ only in their final estimators. TS-Neyman, proposed by \citet{morikawa2026tsneyman}, uses posterior sampling of human-response variances to guide adaptive allocation. We include it as a human-only benchmark, with task means estimated from the acquired human labels. StratPPI adapts the design of \citet{fisch2024stratified}. Because our data lack autorater confidence scores, a uniform paid pilot uses a fraction $\alpha$ of the budget remaining after initialization to estimate coefficients and residual variances. Subsequent acquisitions follow fixed shares $p_q\propto\sqrt{w_q\widehat A_q/c_q}$, while final coefficients are refitted using all acquired pairs.

In the experiments, TRD-UCB uses a practical version of the index of Algorithm~\ref{alg:ucb-time-indexed}: it selects the affordable question maximizing
\begin{equation}
\label{eq:practical_ucb}
\frac{w_q\{\widehat A_q+\kappa_q\}}{c_qn_q^2},\qquad
\kappa_q=\gamma\left\{\frac{1}{\bar R}\sqrt{\frac{2\widehat A_q\ln(b+1)}{n_q-1}}+\frac{\ln(b+1)}{n_q-1}\right\},
\end{equation}
where $\bar R$ is a fixed normalization constant. The bonus is an empirical-Bernstein radius scaled by the exploration multiplier $\gamma$, replacing the theoretical radius $r(t,s)$ in Eq.~\eqref{eq:rho_ppipp_step1_abs}; $\gamma=0$ gives the plug-in greedy rule.

TRD-TS adapts TS-Neyman to the TRD framework by replacing raw human-response variance with learned rectification difficulty and using the common PPI++ estimator. It draws $\tilde A_q$ from a working posterior and selects the affordable question maximizing $w_q\tilde A_q/(c_qn_q^2)$. Whereas TRD-UCB explores through a confidence bonus, TRD-TS explores through posterior sampling. TRD-$\varepsilon$-Greedy follows the decaying-exploration form analyzed by \citet{auer2002finite}: it selects an affordable question uniformly at random with probability $\varepsilon_\tau=\min\{1,cQ/\tau\}$, and otherwise maximizes the plug-in index $w_q\widehat A_q/(c_qn_q^2)$, where $\tau$ counts acquisitions after initialization.

The Oracle uses population difficulties to track target shares proportional to $\sqrt{w_qA_q/c_q}$, subject to acquisition feasibility. It serves as a benchmark rather than an implementable policy. Policy parameters are reported in Table~\ref{tab:params} and held fixed across the main comparisons. Appendix~\ref{app:benchmark-details} provides implementation and parameter-selection details for the methods.
\endgroup

The allocation-quality metric at budget $B$ is
\begin{equation}
\label{eq:expected_mse}
L_{\mathrm{alloc}}(B)=\sum_{q=1}^{Q}\frac{w_qA_q}{n_q(B)}.
\end{equation}
For human-only output, expected MSE uses $\Var(Y_q)$ in place of rectification difficulty. The realized metric is $\sum_qw_q(\widehat\theta_q(B)-\theta_q^*)^2$ for the fitted estimator, averaged over replications. A gap to the Oracle compares a method with the Oracle on the same metric and replication; figures report it as $100(\overline{L}_{m}/\overline{L}_{\mathrm{Oracle}}-1)$. The regret of Theorem~\ref{thm:regret} is measured against the continuous Neyman benchmark $(\sum_q\sqrt{w_qA_qc_q})^2/B$, shown in panel (c) of Figure~\ref{fig:mse_curves}.

\subsection{Numerical Results}
\label{subsec:synthetic_results}

\paragraph{Budget comparison.} Figure~\ref{fig:mse_curves} compares the policies on both measures. On both measures, TRD-UCB has the smallest gap to the Oracle among the implementable policies at every budget shown, and it reaches the realized error of Uniform with $17.9\%$ to $22.7\%$ fewer labels at Uniform budgets from $1{,}000$ to $2{,}000$. The two comparisons of Section~\ref{subsec:synthetic_setup} separate at $B=2{,}000$: among the TRD methods, which learn the same difficulty, the allocation-objective gap is $3.4\%$ for TRD-UCB against $6.9\%$ for TRD-TS and $9.0\%$ for TRD-$\varepsilon$-Greedy, whereas Uniform, which forms the same PPI++ estimate but keeps uniform allocation, is $36.4\%$ above the Oracle (Table~\ref{tab:syn-endpoint}). The two human-only methods are 75\% to 128\% above the Oracle at $B=2{,}000$, and TRD-UCB's realized error is about half that of Traditional. Table~\ref{tab:syn-endpoint} reports both measures at $B=1{,}000$ and $2{,}000$ and the labels each policy saves relative to Traditional and to Uniform, and Table~\ref{tab:ec-primary-paired} reports the paired realized-error differences. Panel (c) reports TRD-UCB's terminal allocation regret against the continuous Neyman benchmark at five budgets; between $B=2{,}000$ and $10{,}000$ it falls slightly faster than the $\ln B/B^2$ reference line.

\begin{figure}[ht]\centering\includegraphics[width=\textwidth]{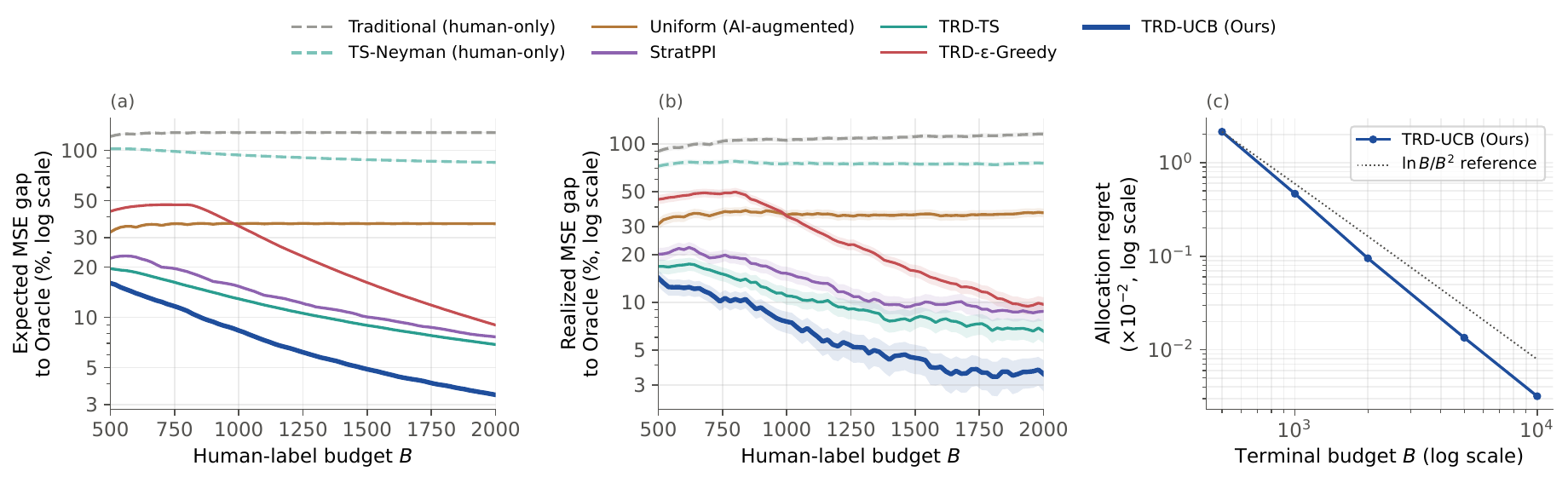}\caption{Synthetic budget comparison
($Q=100$, $\rho_q\sim\mathrm{Beta}(2,1)$, $h=2$;
500 replications).
(a) Expected-MSE and (b) realized-MSE gaps to the Oracle (\%),
on logarithmic axes, with paired 95\% Monte Carlo bands.
(c) TRD-UCB allocation regret against the continuous Neyman
benchmark, with a $\ln B/B^2$ reference line. StratPPI is run separately at each budget.}\label{fig:mse_curves}\end{figure}

\paragraph{AI quality across questions.} Figure~\ref{fig:mse_vs_rho} varies the mean $\mu_\rho$ of $\rho_q\sim\mathrm{Beta}(3\mu_\rho,3(1-\mu_\rho))$ while keeping the concentration at three. The center setting is exactly the baseline $\mathrm{Beta}(2,1)$. Since quality varies by question, changing $\mu_\rho$ changes relative rectification difficulties as well as the overall error scale. Across the five displayed settings, higher mean quality lowers the realized error of every AI-augmented method, whereas the human-only methods do not use the predictions and are unaffected. At all five means, TRD-UCB has the smallest gap to the Oracle among the implementable policies, and its realized MSE is $21.3\%$ to $26.0\%$ below that of Uniform. The expected-MSE counterpart is in Appendix~\ref{app:ts-synthetic}.

\begin{figure}[ht]\centering\includegraphics[width=\textwidth]{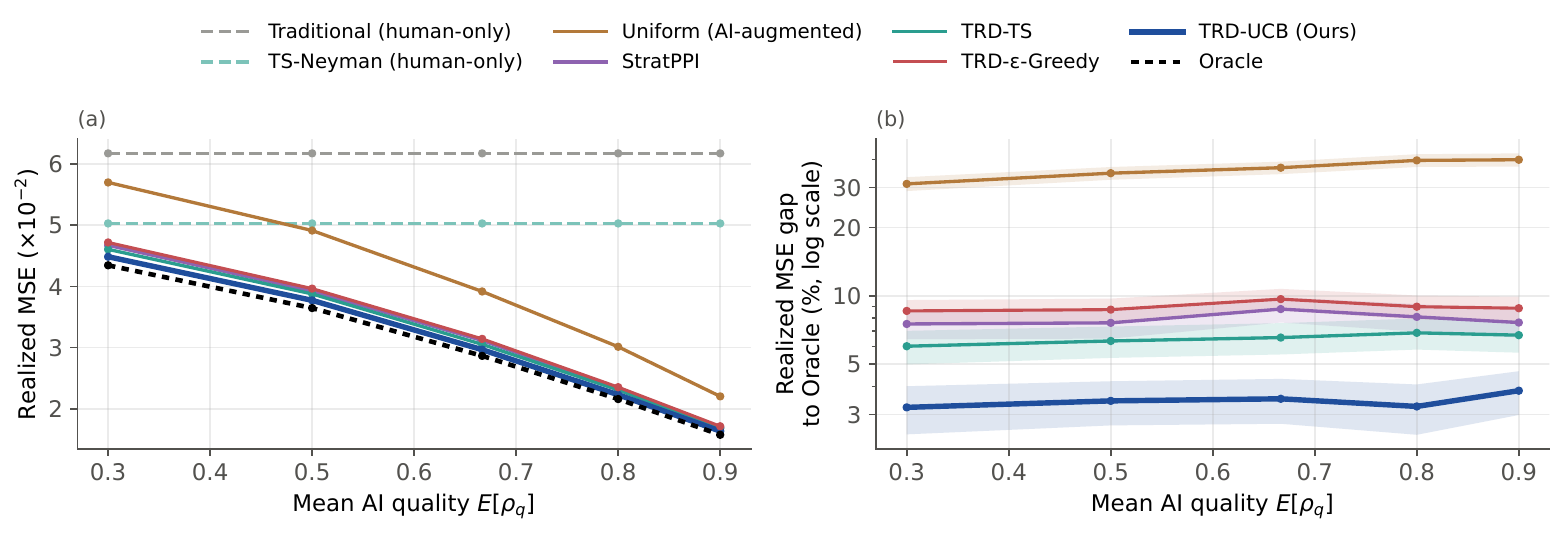}\caption{Synthetic AI-quality sensitivity at $B=2{,}000$ and $h=2$ (500 replications per setting). (a) Realized MSE ($\times10^{-2}$); (b) realized-MSE gap to the Oracle (\%) of the AI-augmented methods, on a logarithmic axis, with paired 95\% Monte Carlo bands; the human-only methods are omitted from (b). Beta concentration is three at all five means.}\label{fig:mse_vs_rho}\end{figure}

\paragraph{Difficulty heterogeneity.} Figure~\ref{fig:mse_vs_h} changes $h$ while retaining the baseline AI-quality distribution. At $h=0.5$ every gap is small. The opportunity for targeted allocation grows with dispersion: the realized gap of Uniform rises from $6.4\%$ to $36.7\%$, while TRD-UCB's stays between $3.2\%$ and $3.6\%$. TRD-UCB's mean realized MSE is below that of Uniform at every $h$, by $2.7\%$ to $24.3\%$. It is the lowest among the implementable policies for $h\ge1$. The expected-MSE counterpart is in Appendix~\ref{app:ts-synthetic}.

\begin{figure}[ht]\centering\includegraphics[width=\textwidth]{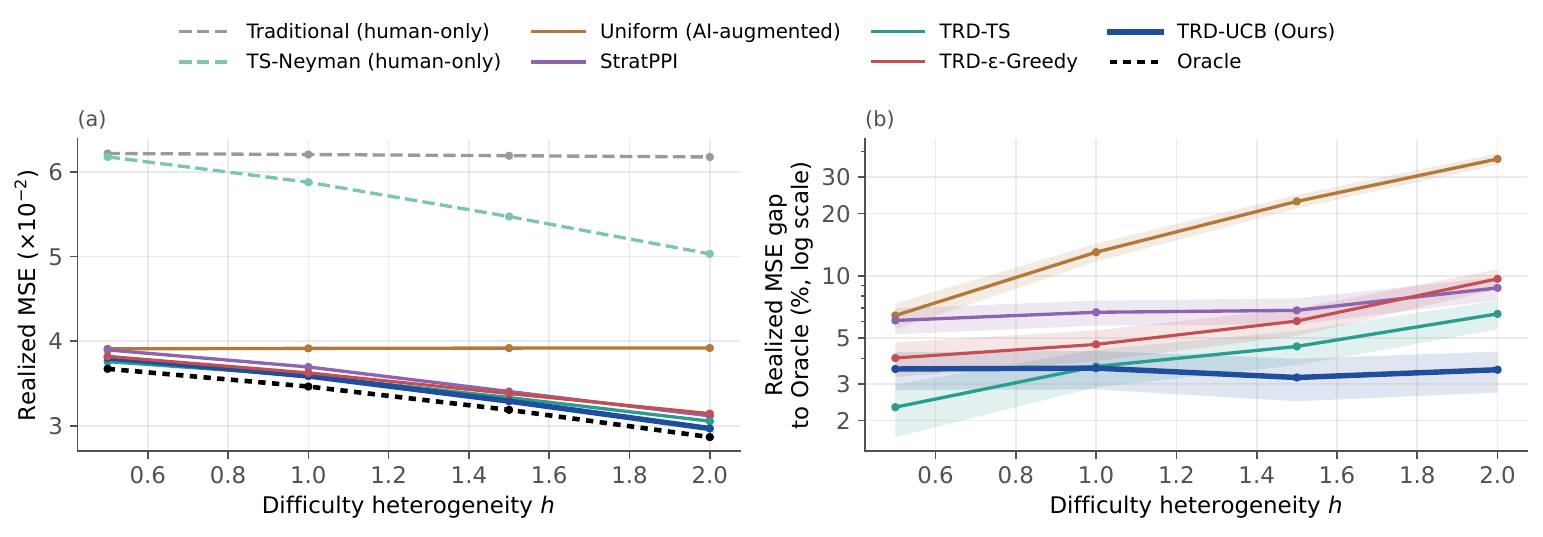}\caption{Synthetic difficulty-heterogeneity sensitivity at $B=2{,}000$ and $\rho_q\sim\mathrm{Beta}(2,1)$ (500 replications per $h$). (a) Realized MSE ($\times10^{-2}$); (b) realized-MSE gap to the Oracle (\%) of the AI-augmented methods, on a logarithmic axis, with paired 95\% Monte Carlo bands; the human-only methods are omitted from (b).}\label{fig:mse_vs_h}\end{figure}

\FloatBarrier

\section{Digital-Twin Data Analysis}
\label{sec:digital_twin}

We next evaluate the allocation decision on the Twin-2K-500 digital-twin survey \citep{toubia2025database}. Unlike the synthetic design, the paired human and AI responses and their rectification difficulty structure come from a fixed empirical question pool.

\subsection{Data and Setup}
\label{subsec:dt_setup}

We use the Wave 4 retest battery's $Q=68$ multiple-choice questions across 14 task types. Each question has 651--2,058 available respondents (mean 1,271), paired with GPT-4.1-mini digital-twin answers; human and AI responses are scaled to $[0,1]$.
The same respondents answered these questions in an earlier wave, so the data also show how stable human answers are. Averaged over respondents and questions, the agreement $1-|y-y'|$ between a respondent's earlier and Wave~4 answers is $0.817$, against $0.757$ between the digital-twin prediction and the Wave~4 answer, and questions with lower digital-twin agreement also tend to have lower test--retest agreement (correlation $0.62$ across questions). Human answers are therefore noisy themselves, and part of each $A_q$ is noise that no predictor could remove. Responses from the same person are correlated across questions. For given counts, the allocation objective has no cross-question covariance terms, and Theorem~\ref{thm:regret} does not require independence across questions.

Figure~\ref{fig:dt-scatter} displays each question's rectification difficulty against the coefficient $\lambda_q^*$ constrained to $[0,1]$. The difficulty ranges from $0.024$ to $0.239$ (CV $0.63$).
Twenty-one questions have constant AI predictions, making the slope unidentified, and five have negative unconstrained slopes clipped to zero.  In both cases $A_q$ equals the human-response variance, and all 68 questions satisfy Assumption~\ref{ass:nondeg_ppipp}. Appendix~\ref{app:twin-lambda-range} reports each question's variances, human--AI correlation, and coefficients.

\begin{figure}[ht]
\centering
\includegraphics[width=.82\textwidth]{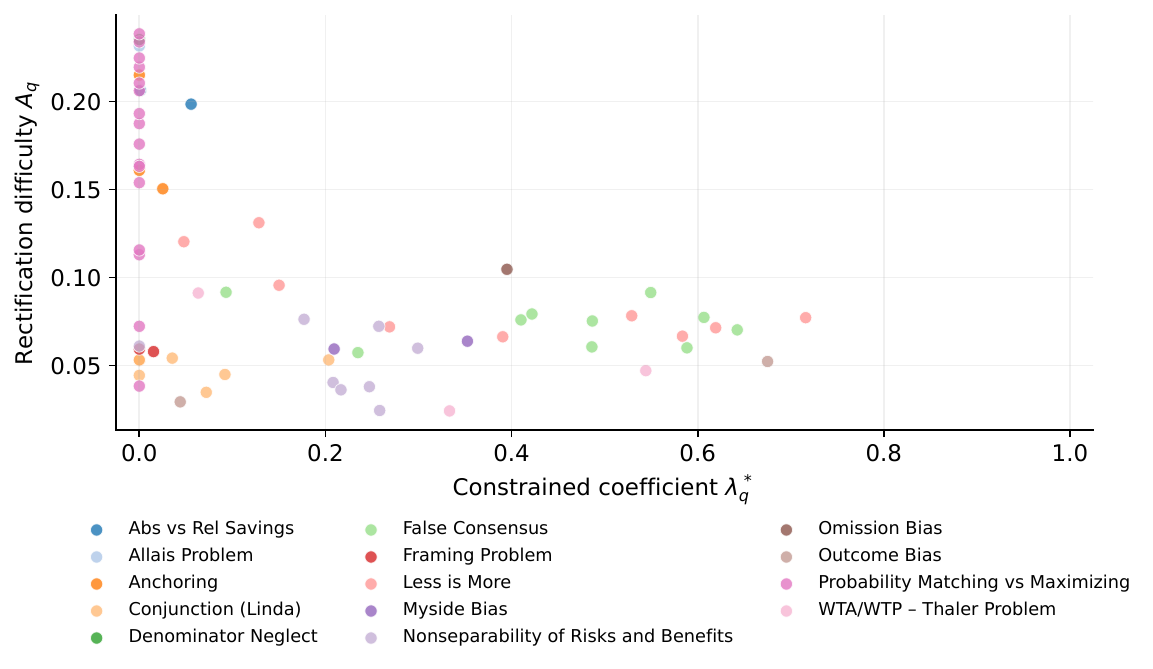}
\caption{Rectification difficulty $A_q$ versus constrained optimal PPI++ coefficient $\lambda_q^*$ for the 68 Wave 4 questions, colored by task type. Zero coefficients include constant AI predictions and negative unconstrained slopes clipped to zero.}
\label{fig:dt-scatter}
\end{figure}

The Twin comparison uses the policies and constants of Section~\ref{subsec:synthetic_setup}, starting from $K=3$ labels per question ($B=204$). Each result averages 500 replications in which every question's pairs are resampled with replacement from that question's pool, and all policies see the same draws.
The Oracle allocates with the true difficulties and uses the same estimator as the AI-augmented policies. Allocation MSE is summed over the 68 questions; in the text, realized MSE is reported per question.

\subsection{Policy Comparison}\label{subsec:dt_comparison}
Table~\ref{tab:dt-budget} reports realized MSE at four budgets, and Table~\ref{tab:ec-twin-budget-obj} the allocation objective. TRD-UCB has the lowest objective among the implementable policies at every evaluated budget from $B=500$ to $4{,}000$ (Figure~\ref{fig:dt-curves}), with gaps to the Oracle of $5.2\%$ at $B=500$ and $3.8\%$, $2.1\%$, $1.5\%$, and $1.2\%$ at $B=1{,}000$ to $4{,}000$. StratPPI's gap falls from $6.4\%$ to $1.9\%$ and TRD-TS's from $4.8\%$ to $1.5\%$, and the gap of Uniform stays near $10\%$. The gap of TRD-$\varepsilon$-Greedy levels off at about $3.5\%$ from $B=2{,}000$ on because its random exploration rate $cQ/\tau$ falls quickly, so questions whose difficulty was underestimated early are rarely revisited.

Figure~\ref{fig:dt-curves} shows the corresponding budget path for the allocation objective and the fitted estimator. At $B=2{,}000$, TRD-UCB's mean realized MSE is $2.7\%$ below that of Uniform and $5.1\%$ below that of Traditional. It is also lower than that of StratPPI and TS-Neyman at every displayed budget, with paired 95\% Monte Carlo intervals excluding zero (Table~\ref{tab:ec-primary-paired}). Uniform has lower realized MSE than TRD-UCB below about $1{,}650$ labels (about 24 per question) and higher from $1{,}700$ labels on. Table~\ref{tab:ec-primary-paired} reports the paired differences with the other benchmarks (with TRD-$\varepsilon$-Greedy in Table~\ref{tab:supp-baseline-paired}), and Appendix~\ref{app:gain-attribution} decomposes the gain over Traditional into AI correction, adaptive allocation, and the residual index. On this pool, allocating with known human-response variances instead of known $A_q$ would raise the Oracle's allocation objective by $0.24\%$, because $A_q/\Var(Y_q)$ varies little across questions; in the Synthetic design, where AI quality varies across questions, the residual index lowers realized error by $4.2\%$ relative to the raw-variance index at $B=2{,}000$.

\begin{table}[ht]\centering\footnotesize
\caption{Twin realized MSE ($\times10^{-2}$, 68-question sum) across budgets, 500 paired streams. Gap to Oracle (\%) in parentheses.}\label{tab:dt-budget}
\setlength{\tabcolsep}{3.5pt}
\setlength{\tabcolsep}{2.6pt}\begin{tabular}{rcccccccc}\toprule
$B$ & Oracle & \makecell{TRD-UCB\\(Ours)} & \makecell{Traditional\\(human-only)} & \makecell{TS-Neyman\\(human-only)} & \makecell{Uniform\\(AI-aug.)} & StratPPI & TRD-TS & \makecell{TRD-$\varepsilon$-\\Greedy}\\\midrule
1{,}000 & 45.5 & 50.8 (11.8\%) & 51.2 (12.5\%) & 53.4 (17.3\%) & \textbf{50.3 (10.6\%)} & 53.7 (18.0\%) & 52.5 (15.4\%) & 50.8 (11.6\%) \\
2{,}000 & 22.9 & \textbf{24.5 (7.3\%)} & 25.8 (13.0\%) & 26.2 (14.5\%) & 25.2 (10.2\%) & 25.3 (10.5\%) & 25.4 (11.2\%) & 25.5 (11.6\%) \\
3{,}000 & 15.2 & \textbf{16.0 (5.2\%)} & 17.1 (12.6\%) & 16.9 (11.3\%) & 16.7 (9.6\%) & 16.2 (6.2\%) & 16.4 (7.9\%) & 17.1 (12.7\%) \\
4{,}000 & 11.4 & \textbf{11.9 (3.8\%)} & 12.8 (12.0\%) & 12.5 (9.0\%) & 12.4 (8.9\%) & 12.0 (5.3\%) & 12.0 (5.1\%) & 12.8 (12.2\%) \\
 
\bottomrule\end{tabular}
\par\vspace{2pt}{\footnotesize\raggedright Notes: Human-only methods report sample means; the other methods report learned-coefficient PPI++ estimates. StratPPI ($\alpha=0.5$) is run separately at each budget; the other methods are evaluated along runs to $B=4{,}000$. Bold marks the lowest non-Oracle realized MSE in each row. Monte Carlo standard errors of the means are at most 0.5 in the displayed units. The expected-MSE counterpart is in Table~\ref{tab:ec-twin-budget-obj}.\par}
\end{table}

\begin{figure}[ht]\centering\includegraphics[width=\textwidth]{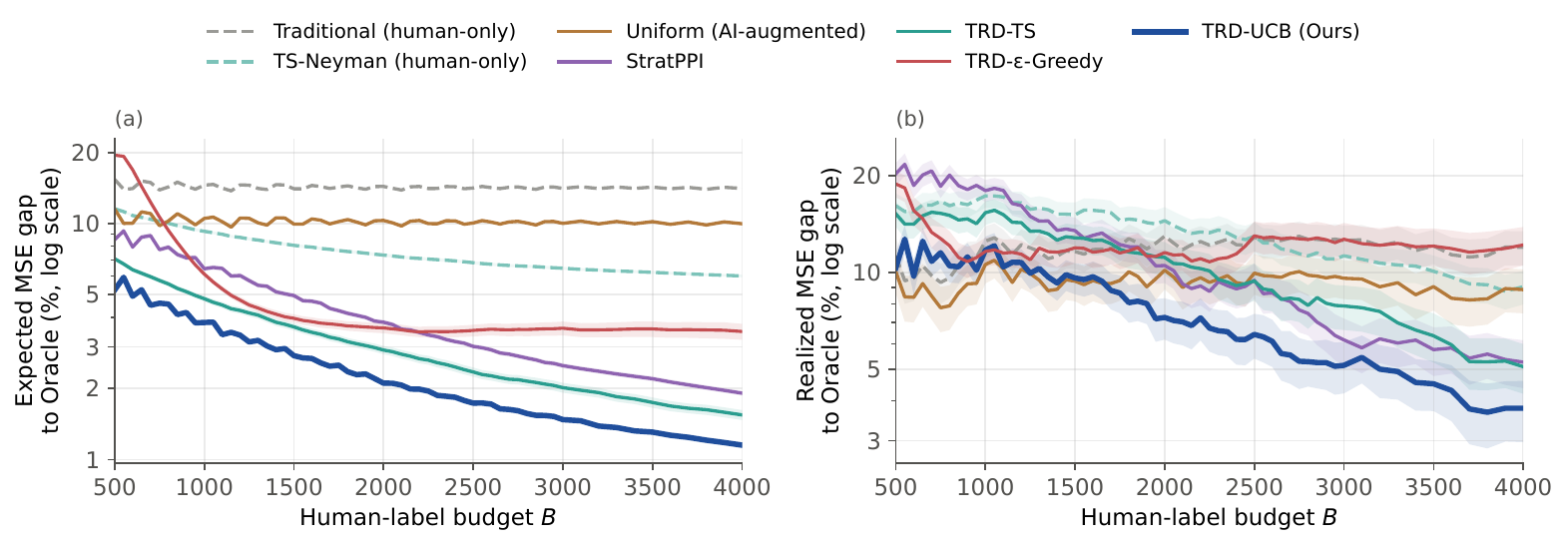}\caption{Twin budget comparison (68 questions, 500 replications). (a) Expected-MSE and (b) realized-MSE gap to the Oracle (\%), on logarithmic axes, with paired 95\% Monte Carlo bands. StratPPI is run separately at each budget.}\label{fig:dt-curves}\end{figure}

Figure~\ref{fig:dt-budget-saved} shows how many labels each method needs to match the mean realized error of Traditional at each displayed budget; with Uniform as the target, which isolates the allocation, TRD-UCB needs $2.5\%$ to $4.5\%$ fewer labels from $2{,}000$ labels on (Appendix~\ref{app:ts-twin-budget}). At the Traditional budgets $1{,}000$, $2{,}000$, $3{,}000$, and $4{,}000$, TRD-UCB needs $0.8\%$, $4.4\%$, $6.7\%$, and $6.9\%$ fewer labels for the same mean error. The Oracle saves $10.2$--$11.6\%$. At $B=2{,}000$, TRD-UCB's realized error exceeds its allocation objective by $8\%$, versus $2.6\%$ for Uniform. The excess comes mainly from the 25 binary questions: when such a question's running mean drifts toward 0 or 1, its variance estimate falls and TRD-UCB acquires fewer further labels while its estimate is still inaccurate. This excess shrinks as labels accumulate, from $9.9\%$ at $B=1{,}000$ to $5.3\%$ at $B=4{,}000$, which is why Uniform has lower realized error at small budgets; a diagnostic that redraws labels at the same counts removes most of it (Table~\ref{tab:fixed-count-check}). A larger exploration multiplier or a shrunk variance estimate in the index reduces the excess (Appendix~\ref{app:supp-feedback}). With a moderated $t$ interval, nominal 95\% coverage is 92.4\% to 94.8\% for every AI-augmented policy at budgets from 1,000 to 8,000; Appendix~\ref{app:supp-coverage} reports coverage by terminal count.

\begin{figure}[ht]\centering\includegraphics[width=.85\textwidth]{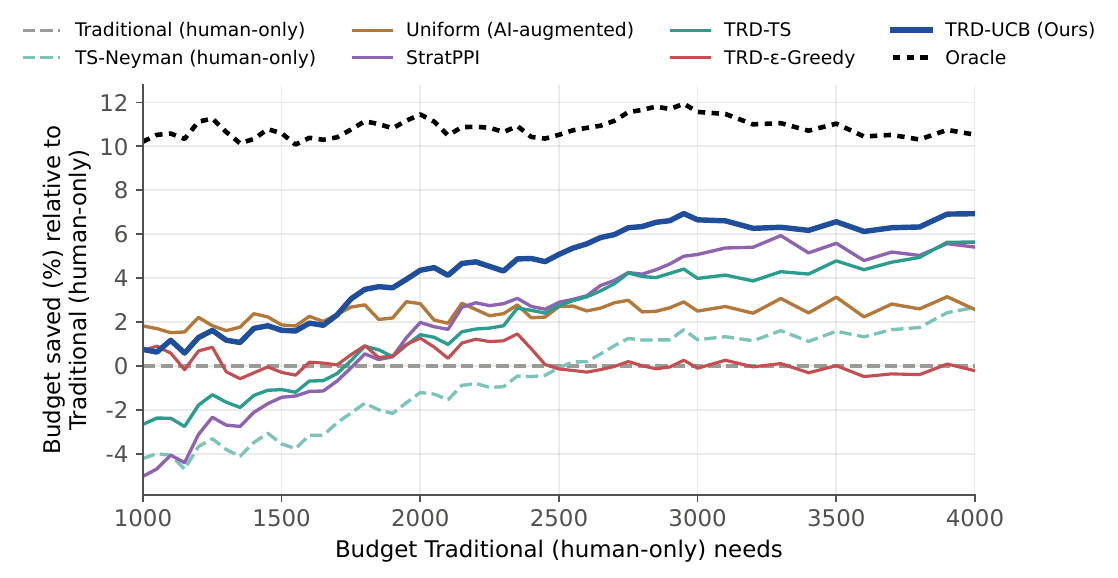}\caption{Twin labels saved at equal mean realized MSE relative to Traditional (human-only), 500 replications. A method's required budget is the earliest budget from which its mean realized MSE stays at or below the target through $B=5{,}000$, interpolated linearly; targets below $B=1{,}000$ are omitted.}\label{fig:dt-budget-saved}\end{figure}

\subsection{Weight and Cost Sensitivity}
\label{subsec:dt_sensitivity}

Operational priorities and acquisition costs need not be equal. Table~\ref{tab:dt-sensitivity} changes the spread of constructed question weights or costs at the main budget $B=2{,}000$, holding the empirical human--AI pool fixed. The inputs follow $w_q=1+a z_q$ or $c_q=\max\{1+a z_q,10^{-3}\}$ on their respective panels; they are decision scenarios, not measured monetary prices, because observed priorities and prices are not available for this pool. In practice a weight can reflect the decision value of accuracy for a question, such as the revenue or risk exposure of the decision it informs, and a cost the price of one human label, such as respondent incentives or expert time. TRD-UCB has the lowest non-Oracle expected allocation objective in the displayed weight and cost rows (Table~\ref{tab:ec-twin-weights-obj}). It also has the lowest realized MSE among the implementable policies in every weight and cost scenario of Table~\ref{tab:dt-sensitivity}. Its realized error is $3.2\%$ to $15.4\%$ below that of Uniform for weight spreads $a$ from $0.25$ to $1$, and $1.8\%$ to $8.0\%$ below for the same cost spreads. Appendices~\ref{app:supp-dependence}, \ref{app:supp-shift}, and~\ref{app:supp-task-split} add three checks: weights, costs, and difficulty that move together; shifts in the respondent mix; and two disjoint task halves. TRD-UCB has the smallest allocation-objective gap among the implementable policies in all of them, and its realized error is up to $40\%$ below that of Uniform. Appendix~\ref{app:syn-joint} repeats the joint scenarios in the Synthetic design.

\begin{table}[ht]\centering\footnotesize
\caption{Twin realized MSE ($\times10^{-2}$) under weight and cost heterogeneity, $B=2{,}000$, 500 paired streams. Gap to Oracle (\%) in parentheses.}\label{tab:dt-sensitivity}
\setlength{\tabcolsep}{3.5pt}
\textit{(a) Varying weights, $c_q=1$}\par\smallskip
\setlength{\tabcolsep}{2.6pt}\begin{tabular}{rcccccccc}\toprule
$a$ & Oracle & \makecell{TRD-UCB\\(Ours)} & \makecell{Traditional\\(human-only)} & \makecell{TS-Neyman\\(human-only)} & \makecell{Uniform\\(AI-aug.)} & StratPPI & TRD-TS & \makecell{TRD-$\varepsilon$-\\Greedy}\\\midrule
0.00 & 22.9 & \textbf{24.5 (7.3\%)} & 25.8 (13.0\%) & 26.2 (14.5\%) & 25.2 (10.2\%) & 25.3 (10.5\%) & 25.4 (11.2\%) & 25.5 (11.6\%) \\
0.25 & 22.4 & \textbf{24.1 (7.4\%)} & 25.6 (13.8\%) & 25.5 (13.8\%) & 24.9 (11.0\%) & 24.8 (10.5\%) & 24.8 (10.4\%) & 25.0 (11.6\%) \\
0.50 & 21.6 & \textbf{23.3 (7.7\%)} & 25.3 (16.8\%) & 24.7 (14.2\%) & 24.6 (13.9\%) & 24.0 (11.1\%) & 24.1 (11.2\%) & 24.2 (11.9\%) \\
0.75 & 20.6 & \textbf{22.1 (7.7\%)} & 25.0 (21.6\%) & 23.4 (14.0\%) & 24.4 (18.5\%) & 23.0 (11.7\%) & 22.8 (11.1\%) & 23.1 (12.3\%) \\
1.00 & 19.0 & \textbf{20.4 (7.2\%)} & 24.7 (30.0\%) & 21.6 (13.6\%) & 24.1 (26.7\%) & 21.5 (13.1\%) & 20.9 (10.1\%) & 21.7 (13.9\%) \\
 
\bottomrule\end{tabular}
\par\medskip
\textit{(b) Varying costs, $w_q=1$}\par\smallskip
\setlength{\tabcolsep}{2.6pt}\begin{tabular}{rcccccccc}\toprule
$a$ & Oracle & \makecell{TRD-UCB\\(Ours)} & \makecell{Traditional\\(human-only)} & \makecell{TS-Neyman\\(human-only)} & \makecell{Uniform\\(AI-aug.)} & StratPPI & TRD-TS & \makecell{TRD-$\varepsilon$-\\Greedy}\\\midrule
0.00 & 22.9 & \textbf{24.5 (7.3\%)} & 25.8 (13.0\%) & 26.2 (14.5\%) & 25.2 (10.2\%) & 25.3 (10.5\%) & 25.4 (11.2\%) & 25.5 (11.6\%) \\
0.25 & 23.1 & \textbf{24.8 (7.4\%)} & 25.9 (12.2\%) & 26.4 (14.2\%) & 25.3 (9.4\%) & 25.5 (10.5\%) & 25.6 (10.7\%) & 25.7 (11.1\%) \\
0.50 & 23.0 & \textbf{24.8 (7.7\%)} & 26.0 (12.9\%) & 26.2 (13.9\%) & 25.3 (10.2\%) & 25.4 (10.3\%) & 25.4 (10.6\%) & 25.6 (11.5\%) \\
0.75 & 22.7 & \textbf{24.4 (7.5\%)} & 26.1 (15.1\%) & 25.9 (14.3\%) & 25.4 (12.2\%) & 25.0 (10.1\%) & 25.2 (11.1\%) & 25.0 (10.1\%) \\
1.00 & 21.7 & \textbf{23.5 (8.3\%)} & 26.2 (20.8\%) & 24.9 (14.5\%) & 25.6 (17.8\%) & 23.9 (10.1\%) & 24.1 (11.1\%) & 23.8 (9.7\%) \\
 
\bottomrule\end{tabular}
\par\vspace{2pt}{\footnotesize\raggedright Notes: $w_q=1+az_q$ or $c_q=\max\{1+az_q,10^{-3}\}$, with $z_q$ fixed across $a$ within each panel. At $a=0$, both panels repeat the $B=2{,}000$ row of Table~\ref{tab:dt-budget}. Weights are not renormalized. Bold marks the lowest non-Oracle realized MSE in each row. Strategy settings match Table~\ref{tab:dt-budget}.\par}
\end{table}

\subsection{Effect of Difficulty Heterogeneity}
\label{subsec:dt_heterogeneity}
Table~\ref{tab:dt-heterogeneity} redistributes the empirical rectification difficulties while holding their mean fixed; $h=1$ is the empirical pool, and the other values of $h$ are constructed.
Writing $\bar A=Q^{-1}\sum_q A_q$ and $\bar\ell=Q^{-1}\sum_q\log A_q$, the transformed difficulties are
\[
A_q^{(h)}=\bar A\,\frac{\exp\{h(\log A_q-\bar\ell)\}}{Q^{-1}\sum_j\exp\{h(\log A_j-\bar\ell)\}}.
\]

Thus $A_q^{(0)}=\bar A$ and $A_q^{(1)}=A_q$. Writing $S_q=Y_q^{\mathrm{LLM}}$, each available pair is transformed as
\[
Y_q^{(h)}=\lambda_q^* S_q+\sqrt{A_q^{(h)}/A_q}\,(Y_q-\lambda_q^* S_q),\qquad S_q^{(h)}=S_q.
\]
The target means are recomputed from the transformed pairs;
$\lambda_q^*$ remains optimal over $[0,1]$, and the difficulty
becomes $A_q^{(h)}$.
Table~\ref{tab:dt-heterogeneity} reports realized-error gaps;
allocation-objective gaps and the additional setting $h=0.25$
are reported in Appendix~\ref{app:ts-twin-sensitivity}.
At $h=1.5$ and $2$, TRD-UCB's allocation-objective gaps are
$3.6\%$ and $6.3\%$, compared with $23\%$ and $40\%$ for Uniform;
StratPPI and TRD-TS are within 1.5 percentage points of TRD-UCB.
At $h=2$, difficulties span a factor of about 100, and the
additive exploration term $\gamma\ln(b+1)/(n_q-1)$ remains
larger than the smallest difficulties, leaving the ten easiest
questions with 19 labels each on average, versus 10 under the Oracle.
TRD-UCB has the lowest realized MSE among the implementable
policies for $h\ge1$, improving on Uniform by $18.4\%$ at $h=2$.
For $h\le0.5$, Uniform has the lowest realized MSE, and the
Oracle's realized error is within $3\%$ of Uniform's.

\begin{table}[ht]\centering\footnotesize
\caption{Twin realized-MSE gaps to Oracle ($\times10^{-2}$) under varying $A_q$ heterogeneity, $B=2{,}000$, 500 paired streams. Mean $A_q$ is held constant.}\label{tab:dt-heterogeneity}
\setlength{\tabcolsep}{3.5pt}
\setlength{\tabcolsep}{2.6pt}\begin{tabular}{rcccccccc}\toprule
$h$ & Oracle & \makecell{TRD-UCB\\(Ours)} & \makecell{Traditional\\(human-only)} & \makecell{TS-Neyman\\(human-only)} & \makecell{Uniform\\(AI-aug.)} & StratPPI & TRD-TS & \makecell{TRD-$\varepsilon$-\\Greedy}\\\midrule
0.00 & 25.1 & 1.22 (4.9\%) & 0.66 (2.6\%) & 2.50 (10.0\%) & \textbf{0.04 (0.2\%)} & 2.74 (10.9\%) & 1.93 (7.7\%) & 2.35 (9.4\%) \\
0.50 & 24.4 & 1.35 (5.5\%) & 1.33 (5.4\%) & 2.89 (11.8\%) & \textbf{0.70 (2.9\%)} & 2.49 (10.2\%) & 2.23 (9.1\%) & 2.61 (10.7\%) \\
1.00 & 22.9 & \textbf{1.66 (7.3\%)} & 2.97 (13.0\%) & 3.32 (14.5\%) & 2.33 (10.2\%) & 2.41 (10.5\%) & 2.55 (11.2\%) & 2.64 (11.6\%) \\
1.50 & 20.6 & \textbf{2.12 (10.3\%)} & 5.30 (25.8\%) & 3.63 (17.7\%) & 4.63 (22.5\%) & 2.27 (11.0\%) & 2.73 (13.3\%) & 2.67 (13.0\%) \\
2.00 & 18.1 & \textbf{2.47 (13.6\%)} & 7.80 (43.1\%) & 3.98 (22.0\%) & 7.09 (39.2\%) & 2.49 (13.7\%) & 2.84 (15.7\%) & 2.84 (15.7\%) \\
 
\bottomrule\end{tabular}
\par\vspace{2pt}{\footnotesize\raggedright Notes: The Oracle column gives the Oracle's realized MSE; the other columns give each method's realized MSE minus the Oracle's, with the relative gap in parentheses. $h=0$ makes all $A_q$ equal; $h=1$ is the empirical Twin difficulty and repeats the $B=2{,}000$ row of Table~\ref{tab:dt-budget}. Bold marks the lowest non-Oracle gap in each row. Strategy settings match Table~\ref{tab:dt-budget}.\par}
\end{table}

\endgroup

\FloatBarrier
\section{Conclusion}
\label{sec:conclusion}

We model AI-enabled analytics as an online human-validation allocation problem. Within the Tuned Rectification-Difficulty (TRD) framework, the correction coefficients and rectification difficulties are learned jointly from acquired human--AI pairs, so that each label informs both estimation and subsequent allocation. The framework also extends to general $M$-estimation targets, subject to their loss definitions, and $q$ can be any unit, such as a stratum or segment, whose uncertainty falls at rate $A_q/n_q$. The cost-adjusted Neyman objective identifies where a label has greater expected estimation value, and the confidence-radius version of our TRD-UCB policy has vanishing allocation regret under its assumptions. In the synthetic design the practical TRD-UCB policy cuts realized error by about one quarter relative to uniform allocation with the same AI-augmented estimate and by about half relative to the traditional human-only practice, at 2,000 labels. On the real digital-twin survey, it attains the lowest allocation objective among the implementable policies at every evaluated budget, closing most of the gap between uniform allocation and the Oracle. It matches the estimation error of uniform allocation and the traditional human-only practice with up to 4.5\% and 6.9\% fewer human labels, respectively. In constructed scenarios where differences in importance, cost, and difficulty make the optimal allocation more uneven, its realized error is lower than that of uniform allocation by up to 18\% when one of them varies and by up to 40\% when they vary together. At small budgets, realized gains on bounded survey items are smaller than the allocation objective suggests, because the same acquired labels both guide later acquisitions and enter the estimates; the shortfall shrinks as labels accumulate. A larger exploration multiplier or a shrunk variance estimate reduces the shortfall (Appendix~\ref{app:supp-feedback}).

The potential allocation gain depends on the optimal shares implied jointly by rectification difficulties, importance weights, and label costs: when the known difficulties, weights, and costs imply uniform optimal shares, adaptive learning has little allocation benefit to recover. The policy starts afresh within each target task collection and learns from the labels acquired there; it does not require a pretrained allocation predictor. The present empirical results concern fixed pools, so performance under external population or model shift remains to be evaluated. The general $M$-estimation confidence analysis and fully justified adaptive confidence intervals also remain open.

More broadly, the operational value of AI depends not only on model accuracy but also on how human oversight is allocated.
Future directions include decision-specific losses, batched acquisitions, and adaptation to response distributions, populations, and predictors that change while labels are being collected.

\section*{Acknowledgments}
The authors contributed equally to this work.

\bibliographystyle{pomsref}
\bibliography{mybib}

\newpage
\ECSwitch
\makeatletter
\setlength{\@fpsep}{16pt plus 4pt}
\renewcommand{\floatpagefraction}{0.5}
\renewenvironment{table}[1][]{\let\@makecaption\@maketablecaption\@float{table}[!htbp]}{\end@float}
\makeatother
\renewcommand{\theHsection}{EC\arabic{section}}
\renewcommand{\theHequation}{EC\arabic{equation}}
\renewcommand{\theHfigure}{EC\arabic{figure}}
\renewcommand{\theHtable}{EC\arabic{table}}
\renewcommand{\theHtheorem}{EC\arabic{theorem}}
\renewcommand{\theHlemma}{EC\arabic{lemma}}

\begin{center}
\textbf{\Large Supplementary Appendix:}\
\textbf{\Large Allocating Human Oversight in AI-Enabled Analytics}
\end{center}

\section{Extension to General \texorpdfstring{$M$}{M}-Estimation}
\label{appsec:M-estimation}

This appendix provides the formal details for the $M$-estimation extension summarized in Section~\ref{subsec:m_estimation_main}. The main text focuses on population means ($\theta_q^*=\bE[Y_q]$), but many applications target richer estimands such as category choice probabilities, regression coefficients, conjoint/choice partworths, risk- or fraud-score coefficients, and AI-output calibration. Here we show that the $A_q/n_q$ variance structure, and with it the oracle allocation rule, carries over to any estimand defined through an $M$-estimation problem once the multi-dimensional asymptotic covariance is reduced to a scalar difficulty index; the regret guarantee of Theorem~\ref{thm:regret} does not extend to this case. We fix a question $q$ and suppress subscripts throughout, writing $\cP$, $\bm\theta^*$, $d$, $n$, etc.

\subsection{Setup and examples}

Consider a convex, twice-differentiable loss $\ell(\bm X,Y;\bm\theta)$ whose population risk has a unique minimizer $\bm\theta^* = \argmin_{\bm\theta\in\mathbb R^d}\,\bE[\ell(\bm X,Y;\bm\theta)]$, with score function $\bm\psi(\bm X,Y;\bm\theta):=\nabla_{\bm\theta}\ell(\bm X,Y;\bm\theta)$ satisfying $\bE[\bm\psi(\bm X,Y;\bm\theta^*)]=\bm 0$. Let $\bm H:=\bE[\nabla_{\bm\theta}\bm\psi(\bm X,Y;\bm\theta^*)]$ denote the nonsingular population Hessian.

This setup covers four common cases. (i)~\emph{Mean estimation}: $d=1$, $\ell=\tfrac{1}{2}(Y-\theta)^2$, recovering $\theta^*=\bE[Y]$. (ii)~\emph{Categorical responses}: for $Y\in\{1,\dots,J\}$, the loss $\ell=\sum_k \tfrac{1}{2}(\bI\{Y{=}k\}-\theta_k)^2$ yields category probabilities $\theta_k^*=\Pr(Y{=}k)$. (iii)~\emph{Linear regression}: $\ell=\tfrac{1}{2}(Y-\bm X^\top\bm\theta)^2$ gives $\bm\theta^*=\bE[\bm X\bm X^\top]^{-1}\bE[\bm X Y]$. (iv)~\emph{Multinomial logit}: the MNL negative log-likelihood $\ell=\log(\sum_k e^{\bm X_k^\top\bm\theta}) - \sum_k (\bm X_k^\top\bm\theta)\bI\{Y{=}k\}$ yields conjoint partworths.

\subsection{PPI++ \texorpdfstring{$M$}{M}-estimator and variance structure}

Given an LLM predictor $f(\bm X)$, define the surrogate score $\bm\psi^{\mathrm{LLM}}(\bm X;\bm\theta):=\bm\psi(\bm X,f(\bm X);\bm\theta)$. For tuning parameter $\lambda\in[0,1]$, the PPI++ $M$-estimator solves
\begin{equation}
\label{eqn:Mestimator_score_pp}
\bm 0 = \frac{1}{n}\sum_{i=1}^n \bm\psi(\bm X_i,Y_i;\bm\theta) + \lambda\left(\bE[\bm\psi^{\mathrm{LLM}}(\bm X;\bm\theta)] - \frac{1}{n}\sum_{i=1}^n \bm\psi^{\mathrm{LLM}}(\bm X_i;\bm\theta)\right),
\end{equation}
where the expectation is approximated from a large synthetic pool. The surrogate terms cancel in expectation at $\bm\theta^*$. For fixed $\lambda$, a consistent root has the usual asymptotic interpretation under regularity and nonadaptive sampling; cancellation alone does not validate adaptive plug-in inference.

In the synthetic-data-rich regime, linearizing Eq.~\eqref{eqn:Mestimator_score_pp} around $\bm\theta^*$ yields $\sqrt{n}(\widehat{\bm\theta}(\lambda)-\bm\theta^*) \xrightarrow{d} \mathcal N(\bm 0,\bm\Sigma(\lambda))$, where the sandwich covariance is
$$
\bm\Sigma(\lambda) = \bm H^{-1}\bm V^{\Delta(\lambda)}\bm H^{-\top}, \qquad \bm V^{\Delta(\lambda)} := \Var\!\big(\bm\psi(\bm X,Y;\bm\theta^*)-\lambda\,\bm\psi^{\mathrm{LLM}}(\bm X;\bm\theta^*)\big).
$$
As in the scalar case, $\Var(\widehat{\bm\theta}(\lambda))\approx \bm\Sigma(\lambda)/n$: the labeled sample size $n$ enters only through a $1/n$ prefactor.

For the GLM-type losses considered here, write $\ell(\bm X,Y;\bm\theta)=\phi(\bm X;\bm\theta)-Y\,t(\bm X;\bm\theta)$, where $\phi$ is convex and $t$ is linear in $\bm\theta$ (with the analogous vector form for multinomial outcomes). The corrected empirical objective has the form
$$L^{\mathrm{PP}}_\lambda(\bm\theta)=(1-\lambda)\overline\phi_n(\bm\theta)+\lambda\overline\phi_N(\bm\theta)+\text{a term linear in }\bm\theta.$$
Thus $\lambda\in[0,1]$ ensures convexity in $\bm\theta$ and supports the computational extension to these convex $M$-estimation problems \citep{angelopoulos2024ppiplus}. Convexity of an arbitrary original loss alone would not suffice, because the PPI objective contains a subtracted surrogate loss. For squared-loss mean estimation, the quadratic curvature is independent of $\lambda$, so this convexity argument imposes no truncation. Convexity alone also does not establish a general adaptive-inference or allocation-regret guarantee. This convexity argument is separate from the choice of coefficient class for scalar means; the extension of the scalar-mean analysis to $[0,\bar\lambda]$ in Appendix~\ref{subsec:finite_upper_bounds} does not imply that a general corrected $M$-estimation objective stays convex for $\lambda>1$.

\subsection{Scalar difficulty index for allocation}

The allocation framework requires a scalar difficulty per question. Two standard design criteria reduce $\bm\Sigma(\lambda)$ to a scalar while preserving the $1/n$ structure. The {A-optimal} (trace) criterion defines $A^{(A)}(\lambda) := \mathrm{tr}(\bm\Omega\,\bm\Sigma(\lambda))$ for a weight matrix $\bm\Omega=\bm L^\top\bm L\succeq\bm 0$ reflecting which linear functions of $\bm\theta$ matter most. The D-optimal (determinant) criterion defines $A^{(D)}(\lambda) := \det(\bm\Sigma(\lambda))^{1/d}$, which scales with the chosen parameter coordinates; cross-task comparisons require a common scale or corresponding loss weights. For $J$ category probabilities, which sum to one, the covariance in all $J$ coordinates is singular; the D-criterion is evaluated on $J-1$ free coordinates, with $d=J-1$ in the determinant root. For PPI++, each question uses $\lambda_q^*\in\argmin_{\lambda\in[0,1]}A_q(\lambda)$ under the chosen scalarization, and we write $A_q:=A_q(\lambda_q^*)$. Under either criterion, the per-question objective takes the form $A(\lambda_q^*)/n$, so the same TRD-UCB allocation algorithm applies after replacing the scalar rectification difficulty with the chosen scalarization. A full finite-sample regret guarantee for general $M$-estimators would additionally require concentration bounds for the plug-in sandwich-covariance estimate; the main theorem (Section~\ref{sec:algorithm}) states the guarantee for scalar means, and the design illustration in Appendix~\ref{sec:mnl_extension} below explains the broader allocation structure. We leave the general regret analysis to future work.

\subsection{Allocation illustration: choice-share estimation in a multinomial logit model}
\label{sec:mnl_extension}

We illustrate the $M$-estimation extension in a choice-model setting. Conjoint and choice models inform product design, assortment, pricing, and launch decisions; the human input is a costly observed choice, the AI prediction is a low-cost predicted choice probability vector, and the estimand is a vector of choice shares, case~(ii) above, rather than a scalar mean.

We use a synthetic multinomial logit (MNL) design with $Q=50$ choice tasks, each with $J=3$ alternatives described by $d=2$ attributes. For task~$q$, the attribute matrix $\bm X_q\in\mathbb R^{J\times d}$ is a randomly rotated balanced design. Respondent $i$ has partworths $\bm\beta_{q,i}=\bm b_q+\Gamma_q\bm z_i+\bm\eta_i$, with observed respondent features $\bm z_i\sim\mathcal N(\bm 0,\bm I_d)$, idiosyncratic taste $\bm\eta_i\sim\mathcal N(\bm 0,0.1^2\bm I_d)$, a scalar loading $\Gamma_q\sim\mathrm{Uniform}(5,10)$, and a mean partworth $\bm b_q$ with uniformly random direction and log-uniform norm on $[0.5,2]$. The human choice $\bm Y_{q,i}\in\{0,1\}^J$ is the one-hot outcome of $\mathrm{MNL}(\bm X_q\bm\beta_{q,i})$. The AI observes $\bm z_i$ and predicts the choice probabilities $\bm S_{q,i}=\mathrm{softmax}\{\bm X_q(\bm b_q+\bm\delta_q+\zeta_q\Gamma_q\bm z_i)\}$, where $\zeta_q\sim\mathrm{Uniform}(0,1)$ is the fraction of respondent heterogeneity the AI captures and $\bm\delta_q\sim\mathcal N(\bm 0,0.2^2\bm I_d)$ is a task-level bias. Tasks therefore differ mainly in AI reliability rather than in outcome variance.

The estimand is the share vector $\bm\theta_q=\bE[\bm Y_q]$, for which $\bm H_q=\bm I_J$. With the one-hot indicator replaced by $\bm S_q$ in the loss of case~(ii), the AI score is $\bm\theta-\bm S_q$, so PPI++ corrects all $J$ coordinates with one coefficient, $\widehat{\bm\theta}_q=\bar{\bm Y}_q+\hat\lambda_q(\bm\mu_{S,q}-\bar{\bm S}_q)$, where $\bm\mu_{S,q}=\bE[\bm S_q]$ comes from the AI pool and $\hat\lambda_q$ minimizes the trace of the empirical residual covariance over $[0,1]$. Under the trace criterion the difficulty is $A_q=\min_{\lambda\in[0,1]}\mathrm{tr}\,\Var(\bm Y_q-\lambda\bm S_q)$, computed from $200{,}000$ simulated respondents per task, and the policies use the same trace of the sample residual covariance as $\widehat A_q$. The resulting $A_q$ range from $0.075$ to $0.657$ (coefficient of variation $0.73$), and every $\lambda_q^*$ lies within $0.001$ of one. The ratio $A_q/\mathrm{tr}\,\Var(\bm Y_q)$ ranges from $0.11$ to $0.99$, whereas $\mathrm{tr}\,\Var(\bm Y_q)$ lies between $0.651$ and $0.666$.

All policies follow the definitions of Section~\ref{sec:synthetic} with the constants of the main comparison: three initial labels per task, $\gamma=0.08$ for TRD-UCB, $\nu_0=2$ for TRD-TS and TS-Neyman (human-only), $c=5$ for TRD-$\varepsilon$-Greedy, and $\alpha=0.5$ for StratPPI. Because choice indicators are bounded, TRD-UCB uses $\bar R=2$, as in the Twin. On this fixed task portfolio, the policies use $B=2{,}000$ and 500 independent replications, with all policies drawing the same observations for each task. Table~\ref{tab:mnl-allocation} evaluates each policy's final counts in two ways. Panel~A reports the allocation objective $\sum_q A_q/n_q$ for AI-augmented policies and $\sum_q\mathrm{tr}\,\Var(\bm Y_q)/n_q$ for the human-only policies, Traditional and TS-Neyman. The objective for Traditional is computed exactly because its final counts are fixed. Panel~B keeps the final counts, draws 20 fresh label sets for each replication independently of the labels that guided the allocation, re-estimates the shares, and reports the mean of $\sum_q\|\widehat{\bm\theta}_q-\bm\theta_q\|^2$, so it measures the estimation error attributable to the allocation itself.

\begin{table}[ht]\centering\small
\caption{Allocation illustration for MNL choice-share estimation ($Q=50$, $B=2{,}000$, 500 paired replications).}
\label{tab:mnl-allocation}
\begin{tabular}{lrrrrrr}\toprule
 & \multicolumn{3}{c}{A: Allocation objective} & \multicolumn{3}{c}{B: Fresh-label error} \\
\cmidrule(lr){2-4}\cmidrule(lr){5-7}
Policy & $\times10^{-2}$ & MC SE & Gap (\%) & $\times10^{-2}$ & MC SE & Gap (\%) \\ \midrule
Feasible Oracle & 21.832 & 0.000 & -- & 21.897 & 0.032 & -- \\
TRD-UCB & \textbf{22.210} & 0.004 & 1.73 & \textbf{22.291} & 0.035 & 1.80 \\
TRD-TS & 22.247 & 0.004 & 1.90 & 22.355 & 0.035 & 2.09 \\
StratPPI & 22.511 & 0.006 & 3.11 & 22.626 & 0.036 & 3.33 \\
TRD-$\varepsilon$-Greedy & 22.669 & 0.012 & 3.83 & 22.889 & 0.038 & 4.53 \\
Uniform (AI-augmented) & 24.374 & 0.000 & 11.64 & 24.434 & 0.043 & 11.59 \\
TS-Neyman (human-only) & 82.932 & 0.001 & 279.86 & 83.180 & 0.122 & 279.87 \\
Traditional (human-only) & 82.866 & 0.000 & 279.56 & 82.896 & 0.113 & 278.58 \\
 
\bottomrule\end{tabular}
\par\smallskip{\footnotesize\raggedright Notes: Panel~A evaluates $\sum_q A_q/n_q$ for AI-augmented policies and $\sum_q\mathrm{tr}\,\Var(\bm Y_q)/n_q$ for human-only policies at their final counts. Panel~B is the mean squared error of the share estimates, summed over tasks, recomputed from 20 fresh label sets at the same counts. MC SEs are across the 500 replications and in the same units. Gaps are relative to the Feasible Oracle in the same panel, which allocates integer counts with the true difficulties after three initial labels per task. Traditional (human-only) is computed exactly in Panel~A. Bold marks the lowest non-Oracle value in each panel.\par}
\end{table}

In Panel~A of Table~\ref{tab:mnl-allocation}, TRD-UCB has a $1.7\%$ gap to the Oracle, compared with $1.9\%$ for TRD-TS, $3.1\%$ for StratPPI, $3.8\%$ for TRD-$\varepsilon$-Greedy, $11.6\%$ for Uniform (AI-augmented), and about $280\%$ for the two human-only policies. Panel~B gives the same ordering, with gaps of $1.8\%$ for TRD-UCB, $2.1\%$ for TRD-TS, $3.3\%$ for StratPPI, $4.5\%$ for TRD-$\varepsilon$-Greedy, and $11.6\%$ for Uniform (AI-augmented); for every policy, the two panels differ by at most $1$ percentage point. Because the difficulty differences come from AI reliability, which $\mathrm{tr}\,\Var(\bm Y_q)$ does not reveal, approaching the oracle requires learning the residual variances. The scalarized difficulty preserves the oracle allocation form; finite-sample guarantees for general $M$-estimators require separate analysis.

\section{Proof of Theorem~\ref{thm:regret}}
\label{appsec:proof_thm1}

Throughout this proof, write $A_q:=A_q(\lambda_q^*)=\Var(Y_q-\lambda_q^*Y_q^{\mathrm{LLM}})$, $\Lambda:=\sum_j\sqrt{w_jA_jc_j}$, and $\beta_q:=\sqrt{w_qA_q/c_q}/\Lambda$.

\subsection{Step I: PPI++ Good Event and Optimism}
\label{subsec:step1_ppipp_abs}

\par\begingroup

For a deterministic count $s\ge2$, let $\tilde\lambda_{q,s}$ and $\tilde\sigma_{q,s}(\lambda)$ be the fitted coefficient and sample residual standard deviation computed from the first $s$ paired observations of question $q$. Write
\[
\sigma_q(\lambda):=\sqrt{\Var(Y_q-\lambda Y_q^{\mathrm{LLM}})},\qquad
\widetilde A_{q,s}(\tilde\lambda_{q,s}):=\tilde\sigma_{q,s}(\tilde\lambda_{q,s})^2.
\]
The algorithm's statistics satisfy $\hat\lambda_{q,t}=\tilde\lambda_{q,n_{q,t}}$ and $\widehat A_{q,t}(\hat\lambda_{q,t})=\widetilde A_{q,n_{q,t}}(\tilde\lambda_{q,n_{q,t}})$. The sample variance uses denominator $s-1$; the covariance-to-variance ratio is unchanged if both moments use denominator $s$.

Fix $\delta_0\in(0,1)$. For integers $t\ge2$ and $2\le s\le t$, define $\eta_t:=\delta_0/[Q(t+1)^4]$ and
\begin{equation}
\label{eq:rho_ppipp_step1_abs}
r(t,s):=R\sqrt{\frac{2\ln\!\big(2Q(t+1)^5/\delta_0\big)}{s-1}}
+\frac{\sqrt{2}\,R_{Y^{\mathrm{LLM}}}}{t},\qquad
\rho^{++}_{q,t}:=r(t,n_{q,t}).
\end{equation}
Let
\begin{equation}
\label{eq:good_event_v5}
\mathcal G_t:=\bigcap_{q=1}^Q\bigcap_{s=2}^t
\left\{\sup_{\lambda\in[0,1]}|\tilde\sigma_{q,s}(\lambda)-\sigma_q(\lambda)|\le r(t,s)\right\}.
\end{equation}
\begin{lemma}
\label{lem:good_event_ppipp_abs}
Under Assumption~\ref{ass:boundedpp}, for every deterministic integer $t\ge2$,
\[
\bP(\mathcal G_t^c)\le\frac{\delta_0}{(t+1)^3}.
\]
On $\mathcal G_t$, for every $q$ and $2\le s\le t$,
\[
\left|\sqrt{\widetilde A_{q,s}(\tilde\lambda_{q,s})}-\sqrt{A_q}\right|\le r(t,s).
\]
Consequently, for every deterministic integer $m\ge2$,
\begin{equation}
\label{eq:tail_good_v5}
\bP\!\left(\bigcup_{t\ge m}\mathcal G_t^c\right)\le\frac{\delta_0}{2m^2}.
\end{equation}
\end{lemma}
The proof uses uniform concentration of residual standard deviations and exact minimization of the empirical and population variances over the same interval.

\par\endgroup
\subsubsection{Auxiliary Lemmas}

\begin{lemma}[Uniform concentration]%
\label{lem:unif_var_step1_abs}
Under Assumption~\ref{ass:boundedpp}, fix $q$ and integers $t\ge s\ge2$. Define
$$
\sigma_q(\lambda):=\sqrt{\Var(Y_q-\lambda Y^{\mathrm{LLM}}_q)},\qquad
\tilde\sigma_{q,s}(\lambda):=\sqrt{\widetilde{\Var}_s(Y_q-\lambda Y^{\mathrm{LLM}}_q)},
$$
where $\widetilde{\Var}_s$ is computed from the first $s$ i.i.d.\ pairs for question $q$.
Let $N:=t+1$ and $\mathcal L_t:=\{0,1/t,\ldots,1\}$. Then for any $\delta\in(0,1)$,
with probability at least $1-\delta$,
\begin{equation}
\label{eq:unif_var_bound_step1_abs}
\sup_{\lambda\in[0,1]}
\left|\sigma_q(\lambda)-\tilde\sigma_{q,s}(\lambda)\right|
\le
R\sqrt{\frac{2\ln\!\big(2N/\delta\big)}{s-1}}
+
\frac{\sqrt{2}\,R_{Y^{\mathrm{LLM}}}}{t}.
\end{equation}
\end{lemma}
\proof{Proof of Lemma~\ref{lem:unif_var_step1_abs}.}
 For each fixed $\lambda\in\mathcal L_t$, the random variable $Y_q-\lambda Y^{\mathrm{LLM}}_q$ has range at most $R$. Applying the square-root variance bounds of \citet[Theorem~10]{maurer2009empirical} (the unbiased sample variance has expectation $\sigma_q(\lambda)^2$) to the first $s$ i.i.d.\ pairs with failure probability $\delta/(2N)$ in each tail gives
$$
\bP\!\left(
\left|\sigma_q(\lambda)-\tilde\sigma_{q,s}(\lambda)\right|
>
R\sqrt{\frac{2\ln(2N/\delta)}{s-1}}
\right)
\le \frac{\delta}{N}.
$$
Union bounding over $\lambda\in\mathcal L_t$ yields, with probability at least $1-\delta$,
\begin{equation}
\label{eq:ppipp_grid_event_step1_unif}
\sup_{\lambda\in\mathcal L_t}
\left|\sigma_q(\lambda)-\tilde\sigma_{q,s}(\lambda)\right|
\le
R\sqrt{\frac{2\ln(2N/\delta)}{s-1}}.
\end{equation}
For any $\lambda,\lambda'\in[0,1]$,
$$
|\sigma_q(\lambda)-\sigma_q(\lambda')|
\le \sqrt{\Var((\lambda-\lambda')Y^{\mathrm{LLM}}_q)}
=
|\lambda-\lambda'|\sqrt{\Var(Y^{\mathrm{LLM}}_q)}
\le
|\lambda-\lambda'|\frac{R_{Y^{\mathrm{LLM}}}}{2},
$$
and similarly, for the sample standard deviation,
$$
\begin{aligned}
|\tilde\sigma_{q,s}(\lambda)-\tilde\sigma_{q,s}(\lambda')|
&=
\left|\sqrt{\widetilde{\Var}_s(Y_q-\lambda Y_q^{\mathrm{LLM}})}
-\sqrt{\widetilde{\Var}_s(Y_q-\lambda'Y_q^{\mathrm{LLM}})}\right|\\
&\le
\sqrt{\widetilde{\Var}_s((\lambda-\lambda')Y^{\mathrm{LLM}}_q)}
=
|\lambda-\lambda'|\sqrt{\widetilde{\Var}_s(Y^{\mathrm{LLM}}_q)}\\
&\le
|\lambda-\lambda'|\frac{R_{Y^{\mathrm{LLM}}}}{\sqrt{2}}.
\end{aligned}
$$
The last inequality uses $\widetilde{\Var}_s(Y_q^{\mathrm{LLM}})\le [s/(s-1)]R_{Y^{\mathrm{LLM}}}^2/4\le R_{Y^{\mathrm{LLM}}}^2/2$. Now for any $\lambda\in[0,1]$, choose $\lambda^\circ\in\mathcal L_t$ such that $|\lambda-\lambda^\circ|\le 1/t$. Then on the event Eq.~\eqref{eq:ppipp_grid_event_step1_unif},

\begin{equation}\nonumber
\begin{aligned}
|\sigma_q(\lambda)-\tilde\sigma_{q,s}(\lambda)|
&\le
|\sigma_q(\lambda)-\sigma_q(\lambda^\circ)|
+
|\sigma_q(\lambda^\circ)-\tilde\sigma_{q,s}(\lambda^\circ)|
+
|\tilde\sigma_{q,s}(\lambda^\circ)-\tilde\sigma_{q,s}(\lambda)|\\
&\le
R\sqrt{\frac{2\ln(2N/\delta)}{s-1}}
+
\left(\frac{R_{Y^{\mathrm{LLM}}}}{2}+\frac{R_{Y^{\mathrm{LLM}}}}{\sqrt{2}}\right)\frac{1}{t}.
\end{aligned}
\end{equation}
Since $\frac{1}{2}+\frac{1}{\sqrt{2}}\le \sqrt{2}$, we obtain Eq.~\eqref{eq:unif_var_bound_step1_abs} with probability at least $1-\delta$.  \Halmos
\endproof
\subsubsection{Proof of Lemma~\ref{lem:good_event_ppipp_abs}.}\label{subsubsec:proof_good_event_ppipp}

\par\begingroup

\proof{Proof.}
Let $v_{q,s}$, $a_{q,s}$, and $b_{q,s}$ denote the sample human variance, covariance, and predictor variance, all with denominator $s-1$. Then
\[
\tilde\sigma_{q,s}(\lambda)^2=v_{q,s}-2\lambda a_{q,s}+\lambda^2b_{q,s}.
\]
If $b_{q,s}>0$, this quadratic is minimized over $[0,1]$ by $\Pi_{[0,1]}(a_{q,s}/b_{q,s})$. If $b_{q,s}=0$, all sampled predictions are equal, so $a_{q,s}=0$ and the quadratic is constant; the convention $\tilde\lambda_{q,s}=0$ is a minimizer. The identical argument applies to the population moments, including zero predictor variance. Since the square root is increasing,
\[
\sqrt{\widetilde A_{q,s}(\tilde\lambda_{q,s})}=\min_{\lambda\in[0,1]}\tilde\sigma_{q,s}(\lambda),\qquad
\sqrt{A_q}=\min_{\lambda\in[0,1]}\sigma_q(\lambda).
\]
For any two continuous functions $f,g$ on $[0,1]$, the inequalities $g-\|f-g\|_\infty\le f\le g+\|f-g\|_\infty$ imply
\[
|\min f-\min g|\le\|f-g\|_\infty.
\]
Hence, deterministically,
\begin{equation}
\label{eq:min_difference_v5}
\left|\sqrt{\widetilde A_{q,s}(\tilde\lambda_{q,s})}-\sqrt{A_q}\right|
\le\sup_{\lambda\in[0,1]}|\tilde\sigma_{q,s}(\lambda)-\sigma_q(\lambda)|.
\end{equation}
For fixed $t,q,s$ with $2\le s\le t$, apply Lemma~\ref{lem:unif_var_step1_abs} with grid size $t$ and confidence $\eta_t$. Its radius equals $r(t,s)$. A union bound over $q$ and $s$ yields
\[
\bP(\mathcal G_t^c)\le Q(t-1)\eta_t
=\frac{\delta_0(t-1)}{(t+1)^4}\le\frac{\delta_0}{(t+1)^3}.
\]
Equation~\eqref{eq:min_difference_v5} proves the stated bound on $\mathcal G_t$. Finally,
\[
\bP\!\left(\bigcup_{t\ge m}\mathcal G_t^c\right)
\le\delta_0\sum_{t=m}^{\infty}(t+1)^{-3}
\le\delta_0\int_m^\infty x^{-3}\,dx=\frac{\delta_0}{2m^2}.
\]
All these events are defined on deterministic prefixes of the paired streams. Because $t=\sum_jn_{j,t}$ counts all acquired labels and $n_{q,t}\ge K\ge2$, every current count satisfies $2\le n_{q,t}\le t$. Thus, on $\mathcal G_t$, substituting $s=n_{q,t}$ gives the desired bounds for the fitted statistics at the algorithm's current counts. \Halmos
\endproof

\par\endgroup
\subsection{Step II: Bounding the Allocation Gap}

\par\begingroup
For the analysis, let
\[
\ell_B:=\ln\!\left(\frac{2Q(T_{\max}+1)^5}{\delta_0}\right),\qquad
C_\rho:=\sqrt{2}\bigl(R+R_{Y^{\mathrm{LLM}}}\bigr).
\]
For $2\le s\le t\le T_{\max}$, Eq.~\eqref{eq:rho_ppipp_step1_abs} gives
\begin{equation}
\label{eq:rho_order_bound}
r(t,s)\le C_\rho\sqrt{\frac{\ell_B}{s-1}}.
\end{equation}
Indeed, the logarithm in $r(t,s)$ is at most $\ell_B$, and $1/t\le\sqrt{\ell_B/(s-1)}$ since $s\le t$ and $\ell_B\ge1$. Every total count of the algorithm satisfies $t\le T_{\max}$, because each label costs at least $c_{\min}$.

\par\endgroup
We now show that the realized allocation tracks the oracle allocation. Let
\[
C_0:=K\sum_{j=1}^Q c_j,\qquad
D_B:=B-C_0-c_{\max}.
\]
For all sufficiently large $B$, $D_B>0$. Since the algorithm stops with leftover budget $\mathrm{Rem}_B\in[0,c_{\max})$,
\[
\sum_{j=1}^Q c_j(n_{j,B}-K)=B-C_0-\mathrm{Rem}_B\ge D_B.
\]
Because $\sum_j c_j\beta_j=1$, the pigeonhole principle implies that there exists a question $q_a$ such that
\[
n_{q_a,B}-K\ge \beta_{q_a}D_B.
\]

\par\begingroup

Let $t_a$ be the total count immediately before the last adaptive acquisition from $q_a$, so $n_{q_a,t_a}=n_{q_a,B}-1$. This acquisition exists because $D_B>0$. Writing $\beta_{\min}:=\min_q\beta_q>0$, we have, on every sample path,
\[
t_a\ge n_{q_a,t_a}\ge K-1+\beta_{q_a}D_B\ge\beta_{\min}D_B.
\]
Define $m_B:=\lfloor\beta_{\min}D_B\rfloor$ and
\[
\mathcal H_B:=\bigcap_{t\ge m_B}\mathcal G_t.
\]
For sufficiently large $B$, we have $B\ge2(C_0+c_{\max})$ and $\beta_{\min}D_B\ge2$, hence $m_B\ge\beta_{\min}B/4$. Equation~\eqref{eq:tail_good_v5} implies
\begin{equation}
\label{eq:late_failure_v5}
\bP(\mathcal H_B^c)\le\frac{8\delta_0}{\beta_{\min}^2B^2}.
\end{equation}
The pathwise bound $t_a\ge m_B$ therefore ensures $\mathcal G_{t_a}$ on $\mathcal H_B$. By Lemma~\ref{lem:good_event_ppipp_abs}, on this event every question satisfies
\[
\sqrt{A_q}\le\sqrt{\widehat A_{q,t_a}(\hat\lambda_{q,t_a})}+\rho^{++}_{q,t_a},\qquad
\sqrt{A_{q,t_a}^{\mathrm{UCB}}}\le\sqrt{A_q}+2\rho^{++}_{q,t_a}.
\]
Here and below $A_{q,t_a}^{\mathrm{UCB}}$ abbreviates $A_{q,t_a}^{\mathrm{UCB}}(\hat\lambda_{q,t_a})$. The remaining comparisons are on $\mathcal H_B$.
\par\endgroup
 At time ${t_a}$ the remaining budget is at least $c_{\max}$, so every question is affordable. For any question $q$, since $n_{q,{t_a}}\le n_{q,B}$ and $q_a$ is selected at time ${t_a}$,
\[
\frac{w_qA_q}{c_qn_{q,B}^2}
\le
\mathcal I^*_{q,{t_a}}
\le
\mathcal I^{\mathrm{UCB}}_{q,{t_a}}
\le
\mathcal I^{\mathrm{UCB}}_{q_a,{t_a}}.
\]
Using the upper control of $\sqrt{A^{\mathrm{UCB}}}$ and Eq.~\eqref{eq:rho_order_bound},
\[
\frac{w_qA_q}{c_qn_{q,B}^2}
\le
\frac{w_{q_a}}{c_{q_a}(n_{q_a,B}-1)^2}
\left(
\sqrt{A_{q_a}}
+2C_\rho\sqrt{\frac{\ell_B}{n_{q_a,B}-2}}
\right)^2.
\]
Taking square roots and using $n_{q_a,B}-1\ge \beta_{q_a}D_B$ and $n_{q_a,B}-2\ge \beta_{q_a}D_B$ {(because $K\ge2$) yields}
\[
\frac{\sqrt{w_qA_q/c_q}}{n_{q,B}}
\le
\frac{\Lambda}{D_B}
+
\frac{2C_\rho\Lambda\sqrt{\ell_B}}{\sqrt{A_{q_a}\beta_{q_a}}\,D_B^{3/2}}
\le
\frac{\Lambda}{D_B}
+
\frac{2C_\rho\Lambda\sqrt{\ell_B}}{\sqrt{A_{\min}\beta_{\min}}\,D_B^{3/2}},
\]
{Rearranging gives}
\[
n_{q,B}
\ge
\frac{D_B\beta_q}{
1+\frac{2C_\rho\sqrt{\ell_B}}{\sqrt{A_{\min}\beta_{\min}D_B}}
}.
\]
Using $1/(1+x)\ge1-x$ for $x\ge0$ and $D_B=B-C_0-c_{\max}$,
\[
n_{q,B}
\ge
n_q^*
-
\Delta_{q,B},
\]
where
\begin{equation}
\label{eq:Delta_ppipp_final}
\Delta_{q,B}
:=
\beta_q\left(
C_0+c_{\max}
+
\frac{2C_\rho\sqrt{B\ell_B}}{\sqrt{A_{\min}\beta_{\min}}}
\right).
\end{equation}
The upper deviation follows from the budget constraint:
\[
c_qn_{q,B}
\le
B-\sum_{j\ne q}c_jn_{j,B}
\le
c_qn_q^*+\sum_{j\ne q}c_j\Delta_{j,B}.
\]
Therefore,
\begin{equation}
\label{eq:allocation_gap_bound_app}
|n_{q,B}-n_q^*|
\le
\Gamma_{q,B}
:=
\max\left\{\Delta_{q,B},\sum_{j\ne q}\frac{c_j}{c_q}\Delta_{j,B}\right\}.
\end{equation}
With fixed $K$, Eq.~\eqref{eq:allocation_gap_bound_app} implies $\Gamma_{q,B}=O(\sqrt{B\ell_B})$ for each fixed problem instance.

\subsection{Step III: Converting Allocation Gaps to Regret}

{On $\mathcal H_B$, apply Taylor}'s theorem to $f_q(n)=w_qA_q/n$ around $n_q^*=B\beta_q$:
\[
f_q(n_{q,B})-f_q(n_q^*)
=
-\frac{w_qA_q}{(n_q^*)^2}(n_{q,B}-n_q^*)
+
\frac{w_qA_q}{\xi_q^3}(n_{q,B}-n_q^*)^2,
\]
where $\xi_q$ lies between $n_{q,B}$ and $n_q^*$. The first-order term telescopes by the oracle first-order condition,
\[
\frac{w_qA_q}{(n_q^*)^2}
=
\frac{\Lambda^2}{B^2}c_q.
\]
Hence
\[
-\sum_q\frac{w_qA_q}{(n_q^*)^2}(n_{q,B}-n_q^*)
=
\frac{\Lambda^2}{B^2}\left(B-\sum_qc_qn_{q,B}\right)
\le
\frac{\Lambda^2c_{\max}}{B^2}.
\]
With fixed $\delta_0$, $\ell_B=O(\ln B)=o(B)$. Because $\Gamma_{q,B}=O(\sqrt{B\ell_B})=o(B)$, for all sufficiently large $B$ we have $\xi_q\ge n_q^*/2=B\beta_q/2$. Therefore the second-order term satisfies
\[
\sum_q\frac{w_qA_q}{\xi_q^3}(n_{q,B}-n_q^*)^2
\le
\sum_q \frac{8w_qA_q}{B^3\beta_q^3}\Gamma_{q,B}^2
=
O\!\left(\frac{\ell_B}{B^2}\right).
\]
Combining the two terms, {on $\mathcal H_B$ the terminal objective gap} is $O(\ell_B/B^2)$.

\par\begingroup

On $\mathcal H_B^c$, every question has at least $K$ samples, so the objective gap is at most $V_{\max}:=\sum_qw_qA_q/K$. Combining the two event bounds gives, for all sufficiently large $B$,
\begin{equation}
\label{eq:regret_explicit_v5}
\mathcal R(B)\le\frac{\Lambda^2c_{\max}}{B^2}
+\frac{8}{B^3}\sum_q\frac{w_qA_q}{\beta_q^3}\Gamma_{q,B}^2
+\frac{8\delta_0V_{\max}}{\beta_{\min}^2B^2}.
\end{equation}
Since $\Gamma_{q,B}=O(\sqrt{B\ell_B})$ and $\ell_B=O(\ln B)$, this proves Theorem~\ref{thm:regret}. Write $\Delta_B:=C_0+c_{\max}+2C_\rho\sqrt{B\ell_B/(A_{\min}\beta_{\min})}$, so that $\Delta_{q,B}=\beta_q\Delta_B$. All conditions above hold once $B\ge2(C_0+c_{\max})$, $\beta_{\min}B\ge4$, and $\Delta_B\le B\min_qc_q\beta_q/2$; the last gives $\Gamma_{q,B}\le\beta_qB/2$ for every $q$, since $\sum_{j\ne q}(c_j/c_q)\Delta_{j,B}\le\Delta_B/c_q$. Thus $B_0$ may be taken as any budget beyond which these three inequalities hold. \Halmos

\par\endgroup
\subsection{Finite Coefficient Intervals}
\label{subsec:finite_upper_bounds}

\par\begingroup\leavevmode

Fix a finite interval $[a,b]$ with $a<b$, and define
\[
A_q^{[a,b]}:=\min_{\lambda\in[a,b]}\Var(Y_q-\lambda Y_q^{\mathrm{LLM}}).
\]
Project the empirical covariance-to-predictor-variance ratio onto $[a,b]$; if the denominator is zero, use the fixed coefficient $a$. With $Y_q':=Y_q-aY_q^{\mathrm{LLM}}$ and $S_q':=(b-a)Y_q^{\mathrm{LLM}}$, the coefficient $u=(\lambda-a)/(b-a)$ lies in $[0,1]$ and
\[
Y_q'-uS_q'=Y_q-\lambda Y_q^{\mathrm{LLM}}.
\]
When the sample predictor variance is positive,
\[
\Pi_{[0,1]}\!\left(\frac{\widehat{\Cov}(Y_q',S_q')}{\widehat{\Var}(S_q')}\right)
=\frac{\Pi_{[a,b]}\!\left(\widehat{\Cov}(Y_q,Y_q^{\mathrm{LLM}})/\widehat{\Var}(Y_q^{\mathrm{LLM}})\right)-a}{b-a}.
\]
The zero-variance conventions give the same residuals as well. The transformed variables have range bounds $R_Y'=R_Y+|a|R_{Y^{\mathrm{LLM}}}$ and $R_S'=(b-a)R_{Y^{\mathrm{LLM}}}$. Use these bounds in Eq.~\eqref{eq:rho_ppipp_step1_abs}. The preceding proof then gives $O(\ln B/B^2)$ allocation regret relative to the oracle formed from $A_q^{[a,b]}$, provided $\min_qA_q^{[a,b]}>0$. In particular, $[0,\bar\lambda]$ uses $R_Y'=R_Y$ and $R_S'=\bar\lambda R_{Y^{\mathrm{LLM}}}$. The interval is fixed as $B$ increases; this result concerns scalar-mean allocation. A degenerate interval $a=b$ fixes the coefficient ($a=1$ gives PPI and $a=0$ human-only sampling); the same proof applies with the grid reduced to $\{a\}$, the term $\sqrt2R_{Y^{\mathrm{LLM}}}/t$ dropped, $R$ replaced by $R_Y+|a|R_{Y^{\mathrm{LLM}}}$, and $A_q^{[a,a]}=\Var(Y_q-aY_q^{\mathrm{LLM}})$, provided $\min_qA_q^{[a,a]}>0$. Appendix~\ref{app:twin-lambda-range} compares the intervals $[0,1]$ and $[0,2]$ with an unrestricted fit on the Twin pool.

\par\endgroup

\clearpage
\section{Benchmark Details}\label{app:benchmark-details}

\begingroup
\subsection{TS-Neyman adaptation to learned residuals}\label{app:ts-unified}
\label{app:ts-adaptation}

\citet{morikawa2026tsneyman} develop TS-Neyman for stratified mean estimation. Their posterior-sampling construction uses a working model for the raw within-stratum variance. With fixed prior hyperparameters $(\nu_0,s_0^2)$ and $n_h$ observations in stratum $h$, the sampled variance can be written as
\[
\widetilde\sigma_h^2
=\frac{\nu_0s_0^2+(n_h-1)s_h^2}{X_h},
\qquad X_h\sim\chi^2_{\nu_0+n_h-1}.
\]
Their one-step allocation index is proportional to $N_h^2\widetilde\sigma_h^2/\{n_h(n_h+1)\}$, where $N_h$ is the stratum size. Their unequal-cost extension divides this index by the stratum cost; cost adjustment is therefore already supported by the original method.

The benchmark used with the main comparison applies this sampling construction to the learned rectification difficulty and the marginal-efficiency objective of Section~\ref{subsec:oracle_marginal}. At each decision, let $\widehat A_q$ be the current sample residual variance after fitting $\widehat\lambda_q\in[0,1]$, floored at $10^{-12}$, and let $\overline A=Q^{-1}\sum_q\widehat A_q$. For each affordable question, draw
\begin{equation}
\widetilde A_q
=\frac{\nu_0\overline A+(n_q-1)\widehat A_q}{\max\{X_q,10^{-12}\}},
\qquad X_q\sim\chi^2_{\nu_0+n_q-1},
\label{eq:ts-residual-draw}
\end{equation}
and select the largest index
\begin{equation}
I_q^{\rm TS}=\frac{w_q\widetilde A_q}{c_qn_q^2}.
\label{eq:ts-residual-index}
\end{equation}
The cross-question scale $\overline A$ is recomputed at each step; the reported setting is $\nu_0=2$. In the main text this residual adaptation is labeled TRD-TS; TS-Neyman (human-only) applies the same draw to the sample variance of $Y$ in place of $\widehat A_q$ and reports the sample mean. As in the other implementations, ties select the first affordable question in the fixed order, and sampling stops when no further label is affordable. All policies pay for the same $K=3$ initial labels per question. AI-augmented policies refit the final PPI++ estimate $\overline Y_q+\widehat\lambda_q(\mu_{S,q}-\overline S_q)$ using all acquired labels, while human-only policies report the sample mean. Twin replays sample with replacement from the fixed empirical question pool.

The changes from the original rule are substantive: the raw variance becomes an online fitted residual variance, the fixed prior scale becomes a current cross-question estimate, and the exact one-step denominator $n_q(n_q+1)$ becomes the $n_q^2$ marginal-efficiency index used by our other adaptive policies. Thus TRD-TS is an adaptation of the posterior-sampling method, and the draw in Eq.~\eqref{eq:ts-residual-draw} is an algorithmic working distribution.

\endgroup

\begingroup
\subsection{StratPPI adaptation}
\label{sec:pilot-history}
Stratified PPI \citep{fisch2024stratified} first estimates how variable the rectified outcome is in each stratum and then allocates labels across strata by a Neyman rule. Its variance estimates come from autorater confidence available before labeling, which our data do not contain. Our StratPPI benchmark keeps this estimate-then-allocate structure and obtains the estimates from a uniform paid pilot instead.

After the common initialization of $K=3$ labels per question, the policy spends a fraction $\alpha$ of the remaining budget on the pilot, each time labeling the affordable question with the fewest labels. When the pilot ends, it fits $\widehat\lambda_q\in[0,1]$ and the residual variance $\widehat A_q$ of every question from the labels acquired so far and freezes the shares
\[
p_q=\frac{\sqrt{w_q\widehat A_q/c_q}}{\sum_{j=1}^Q\sqrt{w_j\widehat A_j/c_j}}.
\]
Each later label goes to the affordable question with the largest deficit $Np_q-n_q$, where $N$ is the current total number of labels. This is the deficit rule of the feasible Oracle with the pilot estimates in place of the true difficulties. The final estimator refits each coefficient using all acquired pairs.

StratPPI therefore adapts to the data once, when the pilot ends, whereas TRD-UCB, TRD-TS, and TRD-$\varepsilon$-Greedy update the allocation after every label. A short pilot freezes the shares on noisy estimates; a long pilot leaves a smaller part of the budget for the adapted allocation. Because the pilot length depends on the terminal budget, StratPPI is run separately at each reported budget. The main comparison uses $\alpha=0.5$, half of the remaining budget on the pilot; Appendix~\ref{app:ts-selection} reports the full grid of pilot fractions.
\endgroup

\clearpage
\begingroup
\subsection{Policy constants}
\label{app:ts-selection}

All main comparisons use the same policy constants: $\gamma=0.08$ for TRD-UCB, $\nu_0=2$ for TRD-TS and TS-Neyman (human-only), $c=5$ for TRD-$\varepsilon$-Greedy, and $\alpha=0.5$ for StratPPI (Appendix~\ref{sec:pilot-history}). These constants were selected once from the candidate grids in Table~\ref{tab:tuning-grid}, based on comparisons of the mean allocation-objective gap to the Oracle over separate Synthetic tuning replications and Twin tuning replays, and are held fixed across the main experimental conditions. TRD-UCB's $\gamma=0.08$ is also selected by either environment alone. Because the Twin tuning replays use the same 68-question pool, the task-split check in Appendix~\ref{app:supp-task-split} is a within-pool check rather than a held-out-task test.

\begin{table}[!htbp]
\centering\footnotesize
\setlength{\tabcolsep}{4pt}
\caption{Tuning grid for TRD-UCB, TRD-TS, and TRD-$\varepsilon$-Greedy, with the StratPPI pilot fraction: relative gap (\%) of each policy setting to the same-stream Oracle, averaged over $B\in\{500,1{,}000,2{,}000\}$ on the Synthetic tuning design and over $B\in\{1{,}000,2{,}000,3{,}000,4{,}000\}$ on the Twin tuning replays.}
\label{tab:tuning-grid}
\begin{tabular}{lrrrrrr}
\toprule
Setting & \multicolumn{3}{c}{Allocation objective} & \multicolumn{3}{c}{Realized MSE} \\
\cmidrule(lr){2-4}\cmidrule(lr){5-7}
 & Synthetic & Twin & Mean & Synthetic & Twin & Mean \\
\midrule
TRD-UCB $\gamma=0.08$ & 9.5 & 2.1 & \textbf{5.8} & 9.5 & 6.9 & 8.2 \\
TRD-UCB $\gamma=0.16$ & 9.9 & 2.5 & 6.2 & 9.4 & 6.2 & 7.8 \\
TRD-UCB $\gamma=0.04$ & 10.4 & 2.7 & 6.5 & 11.9 & 10.4 & 11.2 \\
TRD-UCB $\gamma=0.32$ & 11.5 & 3.5 & 7.5 & 10.4 & 6.0 & 8.2 \\
TRD-UCB $\gamma=0.02$ & 12.2 & 4.7 & 8.4 & 14.1 & 16.5 & 15.3 \\
TRD-UCB $\gamma=0.64$ & 15.2 & 4.7 & 9.9 & 14.8 & 6.6 & 10.7 \\
\midrule
TRD-TS $\nu_0=2$ & 13.3 & 2.8 & \textbf{8.1} & 12.4 & 10.1 & 11.2 \\
TRD-TS $\nu_0=1$ & 12.3 & 3.9 & 8.1 & 12.4 & 13.6 & 13.0 \\
TRD-TS $\nu_0=0.5$ & 11.5 & 5.9 & 8.7 & 10.8 & 19.5 & 15.2 \\
TRD-TS $\nu_0=5$ & 15.2 & 2.3 & 8.7 & 14.8 & 7.1 & 11.0 \\
TRD-TS $\nu_0=10$ & 17.6 & 2.4 & 10.0 & 17.1 & 6.1 & 11.6 \\
\midrule
TRD-$\varepsilon$-Greedy $c=5$ & 29.3 & 4.4 & \textbf{16.9} & 29.9 & 13.3 & 21.6 \\
TRD-$\varepsilon$-Greedy $c=2$ & 21.4 & 17.0 & 19.2 & 22.0 & 42.2 & 32.1 \\
TRD-$\varepsilon$-Greedy $c=10$ & 37.8 & 5.1 & 21.4 & 39.6 & 8.2 & 23.9 \\
TRD-$\varepsilon$-Greedy $c=20$ & 44.2 & 7.7 & 26.0 & 46.7 & 9.2 & 28.0 \\
TRD-$\varepsilon$-Greedy $c=1$ & 17.0 & 48.5 & 32.8 & 18.2 & 97.5 & 57.9 \\
TRD-$\varepsilon$-Greedy $c=0.5$ & 17.7 & 103.3 & 60.5 & 20.3 & 176.9 & 98.6 \\
\midrule
StratPPI $\alpha=0.5$ & 15.5 & 3.6 & \textbf{9.6} & 17.0 & 10.2 & 13.6 \\
StratPPI $\alpha=0.6$ & 16.3 & 3.0 & 9.7 & 16.4 & 7.8 & 12.1 \\
StratPPI $\alpha=0.7$ & 18.2 & 2.2 & 10.2 & 18.4 & 6.7 & 12.6 \\
StratPPI $\alpha=0.4$ & 16.4 & 5.9 & 11.2 & 18.4 & 15.8 & 17.1 \\
StratPPI $\alpha=0.8$ & 21.3 & 2.5 & 11.9 & 20.6 & 6.4 & 13.5 \\
StratPPI $\alpha=0.3$ & 18.6 & 10.1 & 14.4 & 21.0 & 25.0 & 23.0 \\
StratPPI $\alpha=0.9$ & 25.9 & 4.4 & 15.2 & 25.0 & 6.9 & 15.9 \\
StratPPI $\alpha=0.2$ & 23.1 & 21.1 & 22.1 & 26.3 & 47.8 & 37.0 \\
StratPPI $\alpha=0.1$ & 32.4 & 59.0 & 45.7 & 38.1 & 111.3 & 74.7 \\
\midrule
Uniform (AI-augmented) & 35.3 & 10.2 & 22.7 & 35.2 & 10.1 & 22.6 \\
 
\bottomrule\end{tabular}
\par\smallskip{\footnotesize\raggedright Notes: Each environment uses 200 tuning replications or replays. Within each family, rows are ordered by the Mean allocation-objective gap; bold marks the setting used in all experiments. StratPPI is run separately at each terminal budget. Uniform (AI-augmented) is shown for reference.\par}
\end{table}
\endgroup

\clearpage
\section{Additional Numerical Results}\label{app:additional-results}

\begingroup
\subsection{Complete Synthetic comparisons}
\label{app:ts-synthetic}

Section~\ref{subsec:synthetic_results} reports realized error for the seven implementable methods. This appendix adds two elements on the same 500 replications: the expected-MSE counterpart of every comparison and the labels each method saves relative to Traditional (human-only) and to Uniform (AI-augmented). Expected MSE uses $\sum_qw_qA_q/n_q$ for AI-augmented outputs and $\sum_qw_q\operatorname{Var}(Y_q)/n_q$ for human-only outputs, with the analytic pre-clipping variances of Section~\ref{subsec:synthetic_setup}.

Table~\ref{tab:syn-endpoint} reports both gaps at $B=1{,}000$ and $2{,}000$ and the labels saved at equal realized error. TRD-UCB has the smallest expected and realized gaps among the implementable methods at both budgets; at $B=2{,}000$ they are 3.4\% and 3.5\%, against 6.9\% and 6.6\% for TRD-TS and 7.6\% and 8.8\% for StratPPI. To reach the realized error of Traditional (human-only) at 1{,}000, 1{,}500, and 2{,}000 labels, TRD-UCB needs 41.9\%, 45.9\%, and 48.4\% fewer labels, against 47.3\% to 51.6\% for the Oracle and 31.6\% to 35.6\% for Uniform (AI-augmented). To reach the realized error of Uniform (AI-augmented) at the same budgets, TRD-UCB needs 17.9\%, 21.2\%, and 22.7\% fewer labels, against 24.6\% to 25.7\% for the Oracle, 15.5\% to 20.1\% for TRD-TS, and 12.8\% to 18.7\% for StratPPI. With $c=5$ and $Q=100$, the exploration probability $\min\{1,500/\tau\}$ of TRD-$\varepsilon$-Greedy equals one for the first 500 labels after initialization, so the policy allocates at random, close to Uniform (AI-augmented), up to about $B=800$ and then approaches TRD-TS (Figure~\ref{fig:mse_curves}); its realized gap is 34.9\% at $B=1{,}000$ and 9.7\% at $B=2{,}000$.

\begin{table}[ht]\centering\footnotesize
\caption{Synthetic baseline: expected-MSE and realized-MSE gaps to the Oracle and labels saved relative to Traditional (human-only) and to Uniform (AI-augmented) at equal realized MSE (500 replications).}
\label{tab:syn-endpoint}
\setlength{\tabcolsep}{3pt}
\fitwidth{\begin{tabular}{lcccccccccc}\toprule
 & \multicolumn{2}{c}{Expected-MSE gap} & \multicolumn{2}{c}{Realized-MSE gap} & \multicolumn{3}{c}{Saved vs.\ Traditional (\%)} & \multicolumn{3}{c}{Saved vs.\ Uniform (\%)} \\ \cmidrule(lr){2-3}\cmidrule(lr){4-5}\cmidrule(lr){6-8}\cmidrule(lr){9-11}
Method & 1{,}000 & 2{,}000 & 1{,}000 & 2{,}000 & 1{,}000 & 1{,}500 & 2{,}000 & 1{,}000 & 1{,}500 & 2{,}000 \\ \midrule
Oracle & 5.48 & 2.74 & 5.96 & 2.87 & 47.3 & 50.3 & 51.6 & 24.6 & 25.1 & 25.7 \\\midrule
TRD-UCB (Ours) & 8.3\% & 3.4\% & 7.5\% & 3.5\% & 41.9 & 45.9 & 48.4 & 17.9 & 21.2 & 22.7 \\
Traditional (human-only) & 127.3\% & 127.5\% & 105.3\% & 115.5\% & --- & --- & --- & $-$52.4 & n.a. & n.a. \\
TS-Neyman (human-only) & 93.7\% & 84.6\% & 75.2\% & 75.5\% & 13.6 & 16.4 & 17.8 & $-$27.0 & $-$27.9 & n.a. \\
Uniform (AI-augmented) & 36.2\% & 36.4\% & 35.5\% & 36.7\% & 31.6 & 34.0 & 35.6 & --- & --- & --- \\
StratPPI & 15.3\% & 7.6\% & 15.2\% & 8.8\% & 37.8 & 42.5 & 45.5 & 12.8 & 17.0 & 18.7 \\
TRD-TS & 12.9\% & 6.9\% & 11.0\% & 6.6\% & 39.9 & 44.3 & 47.0 & 15.5 & 18.5 & 20.1 \\
TRD-$\varepsilon$-Greedy & 35.2\% & 9.0\% & 34.9\% & 9.7\% & 25.8 & 33.9 & 40.5 & 0.2 & 10.6 & 16.6 \\
 
\bottomrule\end{tabular}}
\par\vspace{2pt}{\footnotesize\raggedright Notes: $Q=100$, $\rho_q\sim\mathrm{Beta}(2,1)$, $h=2$, $K=3$. The Oracle row gives its expected and realized MSE ($\times10^{-2}$) and its labels saved. Labels saved use the earliest budget from which a method's mean realized MSE stays at or below that of the reference method (Traditional (human-only) or Uniform (AI-augmented)) at the target budget, interpolated linearly; negative values mean that more labels are needed, and n.a.\ that the target is not reached by $B=2{,}000$. StratPPI uses $\alpha=0.5$, TRD-TS $\nu_0=2$, TRD-$\varepsilon$-Greedy $c=5$.\par}
\end{table}
Figures~\ref{fig:ts-full-synthetic-quality} and~\ref{fig:ts-full-synthetic-h} give the expected-MSE counterparts of Figures~\ref{fig:mse_vs_rho} and~\ref{fig:mse_vs_h}. Across the five AI-quality settings, TRD-UCB's expected gap stays between 3.1\% and 4.1\%, against 6.2\% to 7.1\% for TRD-TS, 7.0\% to 7.8\% for StratPPI, 8.2\% to 9.2\% for TRD-$\varepsilon$-Greedy, and 32\% to 39\% for Uniform (AI-augmented). The difficulty-heterogeneity comparison has the same exception as the realized one: at $h=0.5$, TRD-TS has the smaller expected gap (2.4\% against 3.1\%); from $h=1$ on, TRD-UCB has the smallest.

Figures~\ref{fig:syn-rho-real-b1000} and~\ref{fig:syn-h-real-b1000} repeat the realized-MSE comparisons of Figures~\ref{fig:mse_vs_rho} and~\ref{fig:mse_vs_h} at $B=1{,}000$, with the same settings and policy constants.

\begin{figure}[ht]
\centering
\includegraphics[width=\textwidth]{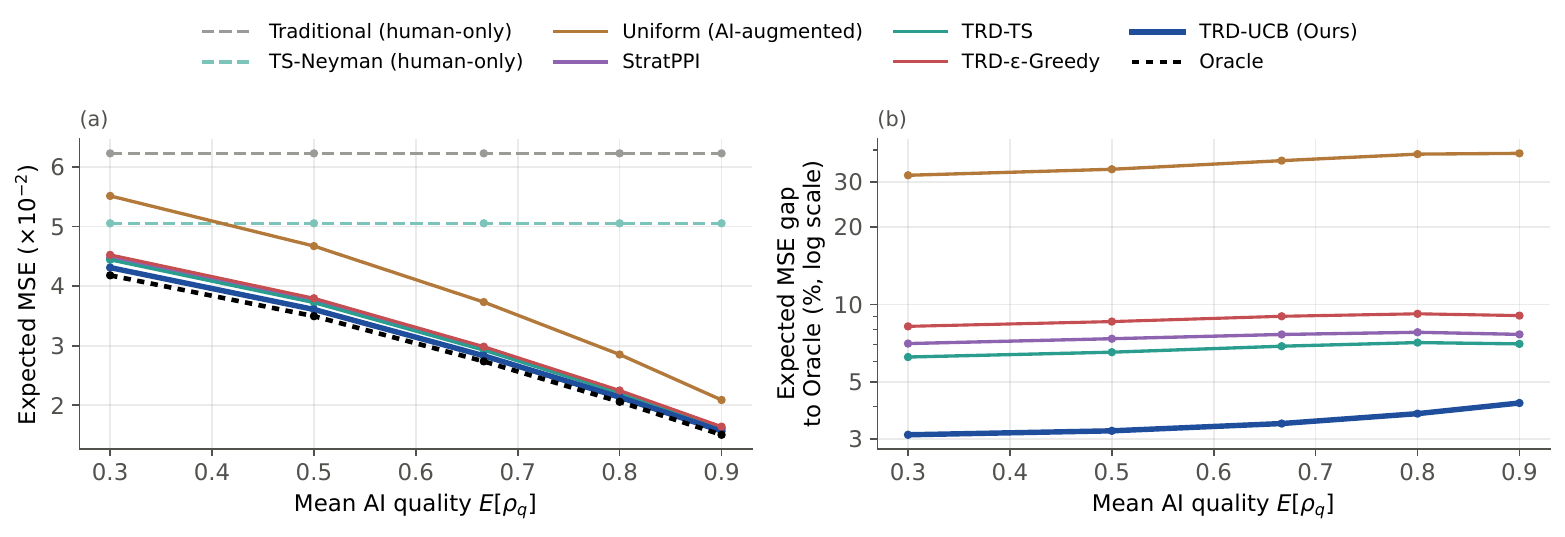}
\caption{Synthetic AI-quality sensitivity, expected MSE, at $B=2{,}000$, $h=2$, $K=3$, and 500 replications per setting. (a) Expected MSE ($\times10^{-2}$) of all methods, with $\operatorname{Var}(Y_q)$ for human-only outputs. (b) Expected-MSE gap to the same-path Oracle for the AI-augmented methods.}
\label{fig:ts-full-synthetic-quality}
\end{figure}

\begin{figure}[ht]
\centering
\includegraphics[width=\textwidth]{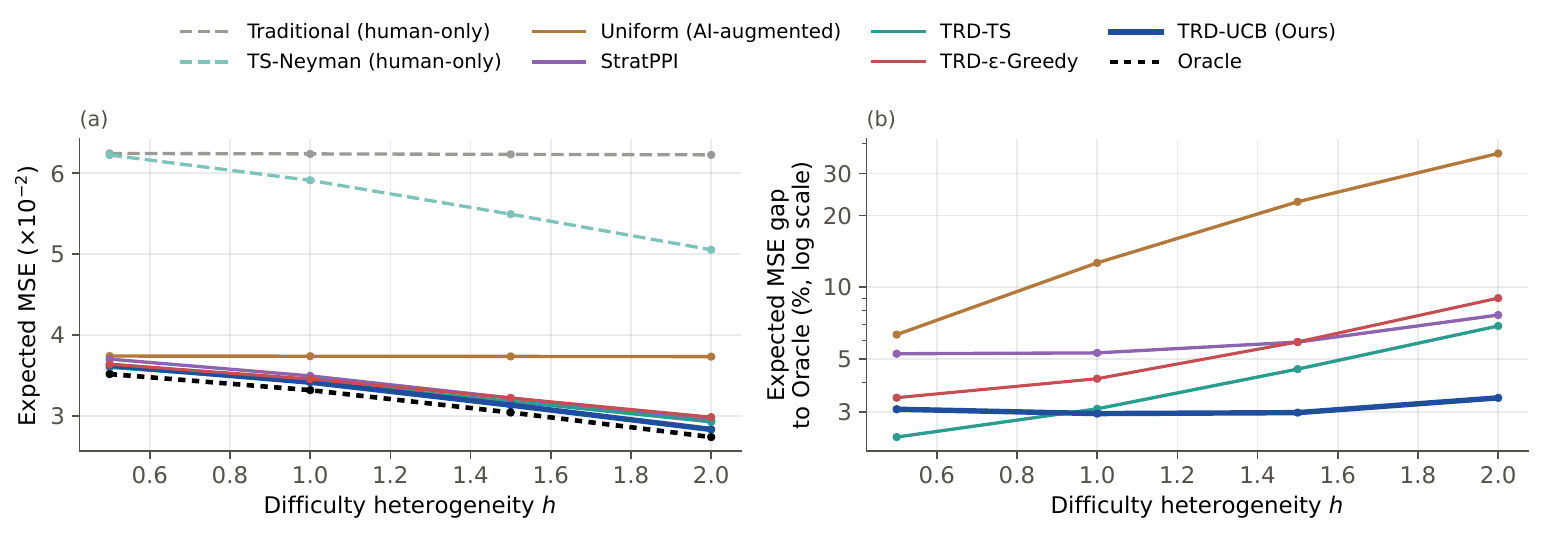}
\caption{Synthetic difficulty-heterogeneity sensitivity, expected MSE, at $B=2{,}000$, $\rho_q\sim\mathrm{Beta}(2,1)$, $K=3$, and 500 replications per $h$. Panels as in Figure~\ref{fig:ts-full-synthetic-quality}.}
\label{fig:ts-full-synthetic-h}
\end{figure}

\begin{figure}[ht]
\centering
\includegraphics[width=\textwidth]{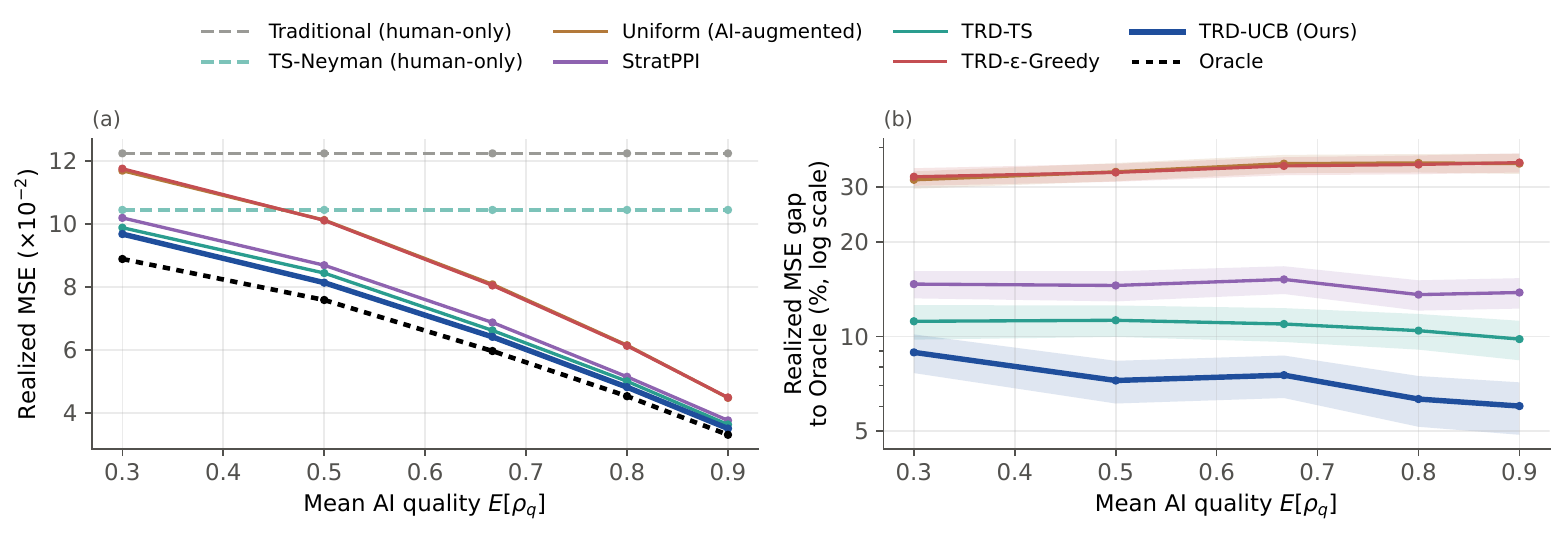}
\caption{Synthetic AI-quality sensitivity, realized MSE, at $B=1{,}000$, $h=2$, $K=3$, and 500 replications per setting. (a) Realized MSE ($\times10^{-2}$) of all methods. (b) Realized-MSE gap to the same-path Oracle for the AI-augmented methods, with paired 95\% Monte Carlo bands.}
\label{fig:syn-rho-real-b1000}
\end{figure}

\begin{figure}[ht]
\centering
\includegraphics[width=\textwidth]{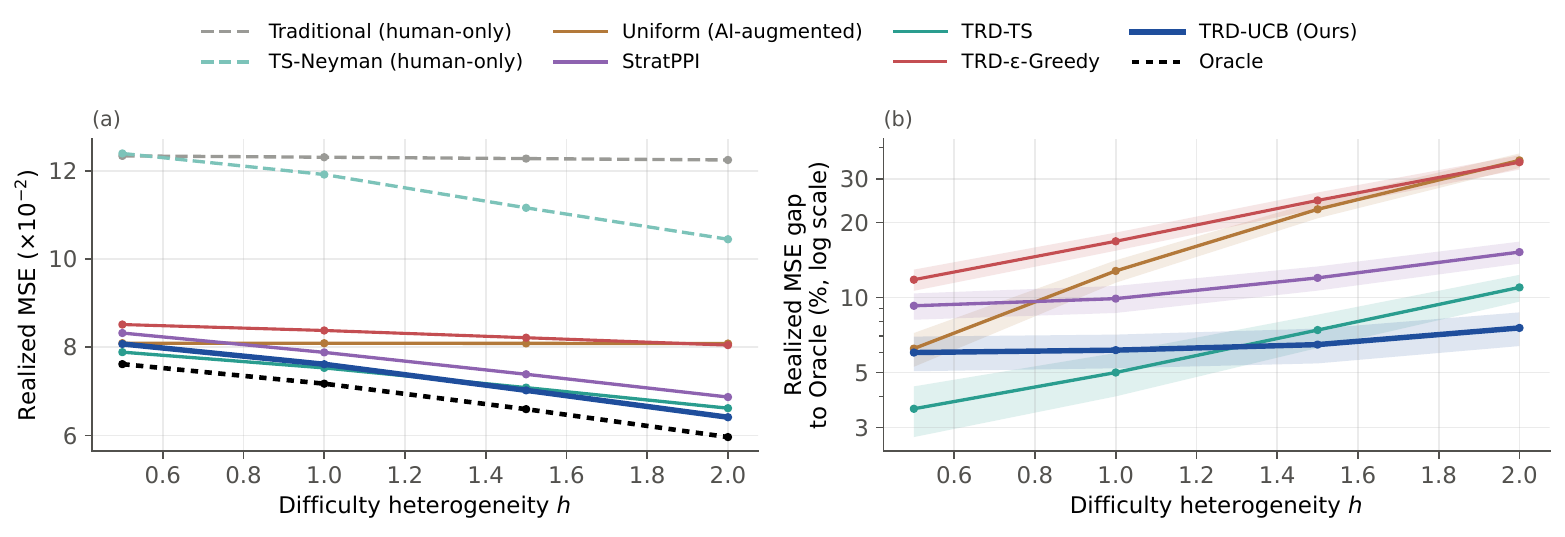}
\caption{Synthetic difficulty-heterogeneity sensitivity, realized MSE, at $B=1{,}000$, $\rho_q\sim\mathrm{Beta}(2,1)$, $K=3$, and 500 replications per $h$. Panels as in Figure~\ref{fig:syn-rho-real-b1000}.}
\label{fig:syn-h-real-b1000}
\end{figure}
\endgroup

\clearpage
\begingroup
\subsection{Complete Twin comparisons}\label{app:ts-twin}

Section~\ref{sec:digital_twin} reports realized error for the seven implementable methods. This appendix adds, on the same 500 paired streams, the expected-MSE counterpart of each Twin table, label efficiency relative to Uniform (AI-augmented), and the heterogeneity setting $h=0.25$. Expected MSE uses $\sum_qw_qA_q/n_q$ for AI-augmented outputs and the empirical-pool $\sum_qw_q\operatorname{Var}(Y_q)/n_q$ for human-only outputs.

\subsubsection{Budgets and label efficiency}
\label{app:ts-twin-budget}

Table~\ref{tab:ec-twin-budget-obj} reports the budget comparison on expected MSE; Figure~\ref{fig:dt-curves} shows both gap paths. On expected MSE, TRD-UCB has the smallest gap at every budget (3.8\%, 2.1\%, 1.5\%, and 1.2\% at $B=1{,}000$ to $4{,}000$), followed by TRD-TS and StratPPI; the realized ranking is the one discussed in Section~\ref{subsec:dt_comparison}. TRD-$\varepsilon$-Greedy reaches an expected gap of 3.6\% at $B=2{,}000$, but its realized gap stays near 12\% at every budget, so its allocation-objective gains do not carry over to realized error on this pool.

\begin{table}[ht]\centering\footnotesize
\caption{Twin expected MSE ($\times10^{-2}$, 68-question sum) across budgets, 500 paired streams. Gap to Oracle (\%) in parentheses.}\label{tab:ec-twin-budget-obj}
\setlength{\tabcolsep}{3pt}
\fitwidth{\begin{tabular}{rcccccccc}\toprule
$B$ & Oracle & \makecell{TRD-UCB\\(Ours)} & \makecell{Traditional\\(human-only)} & \makecell{TS-Neyman\\(human-only)} & \makecell{Uniform\\(AI-aug.)} & StratPPI & TRD-TS & \makecell{TRD-$\varepsilon$-\\Greedy}\\\midrule
1{,}000 & 44.5 & \textbf{46.2 (3.8\%)} & 51.0 (14.5\%) & 48.7 (9.3\%) & 49.2 (10.5\%) & 47.4 (6.4\%) & 46.7 (4.8\%) & 47.3 (6.1\%) \\
2{,}000 & 22.3 & \textbf{22.7 (2.1\%)} & 25.5 (14.4\%) & 23.9 (7.3\%) & 24.6 (10.3\%) & 23.1 (3.8\%) & 22.9 (2.9\%) & 23.1 (3.6\%) \\
3{,}000 & 14.8 & \textbf{15.1 (1.5\%)} & 16.9 (14.1\%) & 15.8 (6.5\%) & 16.3 (10.0\%) & 15.2 (2.5\%) & 15.1 (2.0\%) & 15.4 (3.6\%) \\
4{,}000 & 11.1 & \textbf{11.3 (1.2\%)} & 12.7 (14.1\%) & 11.8 (6.0\%) & 12.2 (10.0\%) & 11.3 (1.9\%) & 11.3 (1.5\%) & 11.5 (3.5\%) \\
 
\bottomrule\end{tabular}}
\par\vspace{2pt}{\footnotesize\raggedright Notes: Expected MSE is $\sum_qA_q/n_q(B)$ for AI-augmented outputs and $\sum_q\operatorname{Var}(Y_q)/n_q(B)$ for human-only outputs, at each method's terminal counts. Bold marks the lowest non-Oracle value in each row. StratPPI uses $\alpha=0.5$, TRD-TS $\nu_0=2$, TRD-$\varepsilon$-Greedy $c=5$.\par}
\end{table}
Figure~\ref{fig:ts-full-twin-saving} repeats the label-efficiency calculation of Figure~\ref{fig:dt-budget-saved} with Uniform (AI-augmented) as the target, so that every method is compared with uniform allocation under the same corrected estimator. TRD-UCB needs 1.1\% more labels at the 1{,}000-label target and 2.5\%, 4.0\%, and 4.5\% fewer at the 2{,}000-, 3{,}000-, and 4{,}000-label targets; the Oracle saves 7.7\% to 9.3\%. StratPPI and TRD-TS save at most 3.0\% and 3.1\%.

\begin{figure}[ht]
\centering
\includegraphics[width=.8\textwidth]{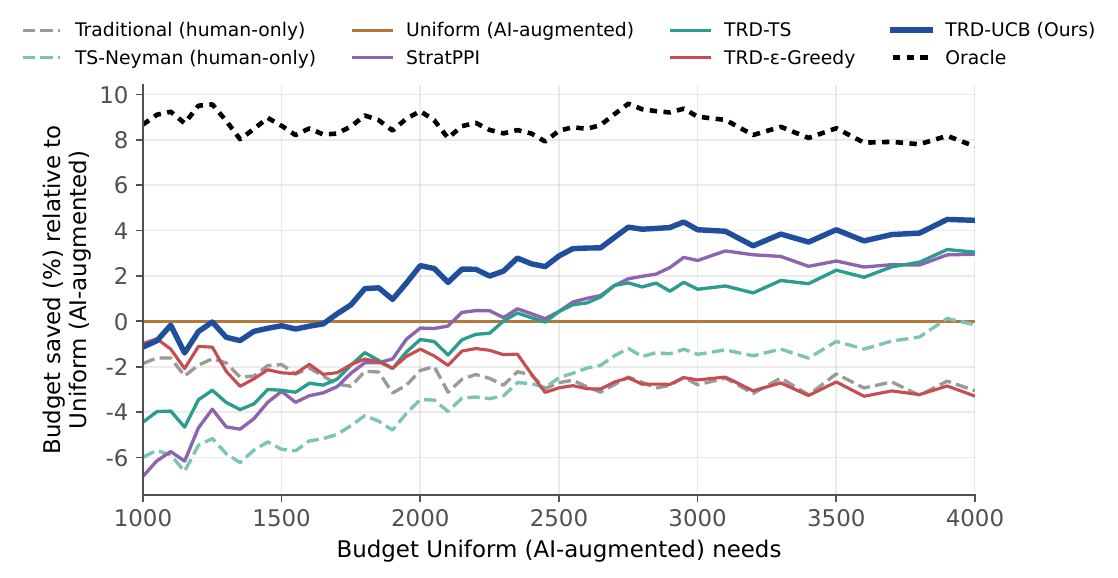}
\caption{Twin labels saved at equal mean realized MSE relative to Uniform (AI-augmented). Required budget is the earliest budget from which a method's mean error remains at or below the target through $B=5{,}000$, with linear interpolation. Targets below $B=1{,}000$ are omitted.}
\label{fig:ts-full-twin-saving}
\end{figure}

\clearpage
\subsubsection{Weights, costs, and difficulty}
\label{app:ts-twin-sensitivity}

Tables~\ref{tab:ec-twin-weights-obj} and~\ref{tab:ec-twin-het-obj} give the expected-MSE counterparts of Tables~\ref{tab:dt-sensitivity} and~\ref{tab:dt-heterogeneity}, and Table~\ref{tab:ec-twin-het-real} adds $h=0.25$ to the realized comparison. In every weight and cost row, TRD-UCB has the lowest expected MSE among the implementable methods, with gaps between 2.1\% and 2.3\%. The heterogeneity comparison shows where adaptive learning does not pay. At $h=0$ all difficulties are equal, uniform allocation is optimal, and Uniform (AI-augmented) is within 0.2\% of the Oracle on both measures; at $h=0.25$ it remains the best implementable method (realized gap 0.7\%, against 5.0\% for TRD-UCB). From $h=0.5$ to $1.5$, TRD-UCB has the lowest expected MSE; on realized MSE, Uniform (AI-augmented) remains lowest at $h=0.5$ (2.9\% against 5.5\%), and TRD-UCB is lowest from $h=1$ on. At $h=2$, StratPPI (4.8\%), TRD-$\varepsilon$-Greedy (5.3\%), and TRD-TS (5.6\%) have smaller expected gaps than TRD-UCB (6.3\%), whose additive bonus term over-samples the easiest questions (Section~\ref{subsec:dt_heterogeneity}); on realized error TRD-UCB (13.6\%) and StratPPI (13.7\%) are close.

\begin{table}[ht]\centering\footnotesize
\caption{Twin expected MSE ($\times10^{-2}$) under weight and cost heterogeneity, $B=2{,}000$, 500 paired streams. Gap to Oracle (\%) in parentheses.}\label{tab:ec-twin-weights-obj}
\setlength{\tabcolsep}{3pt}
\textit{(a) Varying weights, $c_q=1$}\par\smallskip
\fitwidth{\begin{tabular}{rcccccccc}\toprule
$a$ & Oracle & \makecell{TRD-UCB\\(Ours)} & \makecell{Traditional\\(human-only)} & \makecell{TS-Neyman\\(human-only)} & \makecell{Uniform\\(AI-aug.)} & StratPPI & TRD-TS & \makecell{TRD-$\varepsilon$-\\Greedy}\\\midrule
0.00 & 22.3 & \textbf{22.7 (2.1\%)} & 25.5 (14.4\%) & 23.9 (7.3\%) & 24.6 (10.3\%) & 23.1 (3.8\%) & 22.9 (2.9\%) & 23.1 (3.6\%) \\
0.25 & 21.9 & \textbf{22.3 (2.2\%)} & 25.2 (15.4\%) & 23.5 (7.4\%) & 24.3 (11.3\%) & 22.7 (3.8\%) & 22.5 (2.9\%) & 22.6 (3.6\%) \\
0.50 & 21.2 & \textbf{21.6 (2.2\%)} & 25.0 (17.9\%) & 22.7 (7.4\%) & 24.1 (13.8\%) & 22.0 (3.8\%) & 21.8 (3.0\%) & 21.9 (3.6\%) \\
0.75 & 20.2 & \textbf{20.6 (2.2\%)} & 24.7 (22.6\%) & 21.7 (7.4\%) & 23.9 (18.4\%) & 21.0 (3.9\%) & 20.8 (3.1\%) & 20.9 (3.8\%) \\
1.00 & 18.5 & \textbf{19.0 (2.3\%)} & 24.5 (32.0\%) & 20.0 (7.6\%) & 23.6 (27.4\%) & 19.6 (5.6\%) & 19.2 (3.4\%) & 19.5 (5.4\%) \\
 
\bottomrule\end{tabular}}
\par\medskip
\textit{(b) Varying costs, $w_q=1$}\par\smallskip
\fitwidth{\begin{tabular}{rcccccccc}\toprule
$a$ & Oracle & \makecell{TRD-UCB\\(Ours)} & \makecell{Traditional\\(human-only)} & \makecell{TS-Neyman\\(human-only)} & \makecell{Uniform\\(AI-aug.)} & StratPPI & TRD-TS & \makecell{TRD-$\varepsilon$-\\Greedy}\\\midrule
0.00 & 22.3 & \textbf{22.7 (2.1\%)} & 25.5 (14.4\%) & 23.9 (7.3\%) & 24.6 (10.3\%) & 23.1 (3.8\%) & 22.9 (2.9\%) & 23.1 (3.6\%) \\
0.25 & 22.4 & \textbf{22.9 (2.2\%)} & 25.6 (14.1\%) & 24.1 (7.3\%) & 24.7 (10.0\%) & 23.3 (3.8\%) & 23.1 (2.9\%) & 23.2 (3.6\%) \\
0.50 & 22.3 & \textbf{22.8 (2.2\%)} & 25.6 (14.8\%) & 24.0 (7.4\%) & 24.7 (10.7\%) & 23.2 (3.8\%) & 23.0 (3.0\%) & 23.1 (3.5\%) \\
0.75 & 22.0 & \textbf{22.5 (2.3\%)} & 25.8 (17.2\%) & 23.6 (7.5\%) & 24.8 (13.0\%) & 22.8 (3.8\%) & 22.7 (3.1\%) & 22.7 (3.4\%) \\
1.00 & 21.1 & \textbf{21.6 (2.3\%)} & 25.9 (22.5\%) & 22.7 (7.6\%) & 25.0 (18.1\%) & 21.9 (3.8\%) & 21.8 (3.2\%) & 21.8 (3.3\%) \\
 
\bottomrule\end{tabular}}
\par\vspace{2pt}{\footnotesize\raggedright Notes: Scenarios as in Table~\ref{tab:dt-sensitivity}. Expected MSE as in Table~\ref{tab:ec-twin-budget-obj}. Bold marks the lowest non-Oracle value in each row.\par}
\end{table}

\begin{table}[ht]\centering\footnotesize
\caption{Twin expected-MSE gaps to Oracle ($\times10^{-2}$) under varying $A_q$ heterogeneity, $B=2{,}000$, 500 paired streams. Mean $A_q$ is held constant.}\label{tab:ec-twin-het-obj}
\setlength{\tabcolsep}{3pt}
\fitwidth{\begin{tabular}{rcccccccc}\toprule
$h$ & Oracle & \makecell{TRD-UCB\\(Ours)} & \makecell{Traditional\\(human-only)} & \makecell{TS-Neyman\\(human-only)} & \makecell{Uniform\\(AI-aug.)} & StratPPI & TRD-TS & \makecell{TRD-$\varepsilon$-\\Greedy}\\\midrule
0.00 & 24.5 & 0.44 (1.8\%) & 0.90 (3.7\%) & 1.59 (6.5\%) & \textbf{0.00 (0.0\%)} & 1.23 (5.0\%) & 0.73 (3.0\%) & 1.01 (4.1\%) \\
0.25 & 24.3 & 0.40 (1.6\%) & 1.07 (4.4\%) & 1.56 (6.4\%) & \textbf{0.17 (0.7\%)} & 1.14 (4.7\%) & 0.68 (2.8\%) & 0.96 (4.0\%) \\
0.50 & 23.9 & \textbf{0.39 (1.6\%)} & 1.53 (6.4\%) & 1.55 (6.5\%) & 0.63 (2.6\%) & 1.04 (4.4\%) & 0.64 (2.7\%) & 0.90 (3.8\%) \\
1.00 & 22.3 & \textbf{0.47 (2.1\%)} & 3.20 (14.4\%) & 1.64 (7.3\%) & 2.29 (10.3\%) & 0.85 (3.8\%) & 0.65 (2.9\%) & 0.80 (3.6\%) \\
1.50 & 20.0 & \textbf{0.72 (3.6\%)} & 5.48 (27.4\%) & 1.86 (9.3\%) & 4.58 (22.9\%) & 0.76 (3.8\%) & 0.76 (3.8\%) & 0.80 (4.0\%) \\
2.00 & 17.6 & 1.11 (6.3\%) & 7.91 (45.0\%) & 2.22 (12.6\%) & 7.01 (39.8\%) & \textbf{0.84 (4.8\%)} & 0.99 (5.6\%) & 0.94 (5.3\%) \\
 
\bottomrule\end{tabular}}
\par\vspace{2pt}{\footnotesize\raggedright Notes: The Oracle column gives its expected MSE; the other columns give each method's expected MSE minus the Oracle's, with the relative gap in parentheses. $h=0$ makes all $A_q$ equal; $h=1$ is the empirical difficulty. Bold marks the lowest non-Oracle gap in each row.\par}
\end{table}

\begin{table}[ht]\centering\footnotesize
\caption{Twin realized-MSE gaps to Oracle ($\times10^{-2}$) under varying $A_q$ heterogeneity, including $h=0.25$, $B=2{,}000$, 500 paired streams.}\label{tab:ec-twin-het-real}
\setlength{\tabcolsep}{3pt}
\fitwidth{\begin{tabular}{rcccccccc}\toprule
$h$ & Oracle & \makecell{TRD-UCB\\(Ours)} & \makecell{Traditional\\(human-only)} & \makecell{TS-Neyman\\(human-only)} & \makecell{Uniform\\(AI-aug.)} & StratPPI & TRD-TS & \makecell{TRD-$\varepsilon$-\\Greedy}\\\midrule
0.00 & 25.1 & 1.22 (4.9\%) & 0.66 (2.6\%) & 2.50 (10.0\%) & \textbf{0.04 (0.2\%)} & 2.74 (10.9\%) & 1.93 (7.7\%) & 2.35 (9.4\%) \\
0.25 & 24.9 & 1.25 (5.0\%) & 0.81 (3.2\%) & 2.56 (10.3\%) & \textbf{0.19 (0.7\%)} & 2.61 (10.5\%) & 2.07 (8.3\%) & 2.50 (10.0\%) \\
0.50 & 24.4 & 1.35 (5.5\%) & 1.33 (5.4\%) & 2.89 (11.8\%) & \textbf{0.70 (2.9\%)} & 2.49 (10.2\%) & 2.23 (9.1\%) & 2.61 (10.7\%) \\
1.00 & 22.9 & \textbf{1.66 (7.3\%)} & 2.97 (13.0\%) & 3.32 (14.5\%) & 2.33 (10.2\%) & 2.41 (10.5\%) & 2.55 (11.2\%) & 2.64 (11.6\%) \\
1.50 & 20.6 & \textbf{2.12 (10.3\%)} & 5.30 (25.8\%) & 3.63 (17.7\%) & 4.63 (22.5\%) & 2.27 (11.0\%) & 2.73 (13.3\%) & 2.67 (13.0\%) \\
2.00 & 18.1 & \textbf{2.47 (13.6\%)} & 7.80 (43.1\%) & 3.98 (22.0\%) & 7.09 (39.2\%) & 2.49 (13.7\%) & 2.84 (15.7\%) & 2.84 (15.7\%) \\
 
\bottomrule\end{tabular}}
\par\vspace{2pt}{\footnotesize\raggedright Notes: As in Table~\ref{tab:dt-heterogeneity}, with the setting $h=0.25$ added.\par}
\end{table}

Appendix~\ref{app:supplemental-validation} reports the robustness checks on the same Twin streams.
\endgroup

\clearpage
\begingroup
\subsection{Gain attribution and mechanisms}\label{app:gain-attribution}

Relative to Traditional (human-only), TRD-UCB changes three things: the final estimate is corrected with the AI predictions, labels are allocated adaptively, and the allocation index uses the learned rectification difficulty rather than the raw human variance. This appendix adds these components one at a time and then compares TRD-UCB with each primary method; Appendix~\ref{app:realized-reuse} examines how reusing acquired labels affects realized error.

\subsubsection{Layered decomposition}\label{app:attr-decomposition}

Table~\ref{tab:ec-attribution} moves from Traditional (human-only) to TRD-UCB in steps that each change one component, on the replications of Sections~\ref{sec:synthetic} and~\ref{sec:digital_twin}. In panel (a), AI correction is added first: Uniform (AI-augmented) keeps uniform allocation and replaces the sample mean with the PPI++ estimate. UCB(Raw) then allocates adaptively by online estimates of the raw human variance $\operatorname{Var}(Y)$ and reports the same PPI++ estimate. TRD-UCB finally replaces the raw variance in the index with the learned rectification difficulty. Panel (b) reverses the first two steps: UCB(Raw, human-only) allocates by the raw-variance index and reports the sample mean, and UCB(Raw) applies the PPI++ estimate on the same acquisition path. All policies use the same streams, so each step is a paired difference; steps are expressed in percentage points of Traditional's realized MSE and add up to TRD-UCB's total reduction.

In Synthetic, every step lowers realized error, with intervals excluding zero. At $B=2{,}000$, AI correction contributes 36.5 percentage points, adaptive allocation 13.3, and the residual index 2.1, for a total of 52.0, and 1.6 points remain to the Oracle. The two orders give different shares: adaptive allocation contributes 20.8 points when it comes first and AI correction 29.0 after it, because once the estimate is corrected, less variance remains to be allocated. On the Twin pool, AI correction contributes 1.7 to 2.8 points. Adaptive allocation contributes $-1.3$ points at $B=1{,}000$, with an interval just below zero, and 2.2 to 4.3 points from $B=2{,}000$ to $4{,}000$, where it is the largest component. The residual index contributes at most 0.6 points on this pool; Appendix~\ref{app:attr-index} examines how this increment changes with differences in residual ratios. AI correction contributes about as much on the adaptive path as on the uniform one (1.9 to 2.9 against 1.7 to 2.8 points).

\begin{table}[ht]\centering\footnotesize
\caption{Layered realized-MSE reduction from Traditional (human-only) to TRD-UCB, in percentage points of Traditional's realized MSE, with paired 95\% Monte Carlo intervals, 500 paired replications.}
\label{tab:ec-attribution}
\setlength{\tabcolsep}{3pt}
\fitwidth{\begin{tabular}{lcccccc}\toprule
 & \multicolumn{2}{c}{Synthetic} & \multicolumn{4}{c}{Twin} \\ \cmidrule(lr){2-3}\cmidrule(lr){4-7}
Step & 1{,}000 & 2{,}000 & 1{,}000 & 2{,}000 & 3{,}000 & 4{,}000 \\ \midrule
\multicolumn{7}{l}{\emph{(a) Estimator first}} \\
\quad +\,AI correction (Uniform, AI-augmented) & 34.0 {\scriptsize[32.6, 35.4]} & 36.5 {\scriptsize[35.1, 38.0]} & 1.7 {\scriptsize[1.3, 2.1]} & 2.5 {\scriptsize[2.1, 2.9]} & 2.7 {\scriptsize[2.3, 3.0]} & 2.8 {\scriptsize[2.4, 3.2]} \\
\quad +\,adaptive allocation (UCB(Raw)) & 11.5 {\scriptsize[10.4, 12.7]} & 13.3 {\scriptsize[12.2, 14.4]} & $-$1.3 {\scriptsize[$-$2.5, $-$0.04]} & 2.2 {\scriptsize[1.0, 3.3]} & 3.4 {\scriptsize[2.2, 4.6]} & 4.3 {\scriptsize[3.1, 5.5]} \\
\quad +\,residual index (TRD-UCB) & 2.1 {\scriptsize[1.5, 2.8]} & 2.1 {\scriptsize[1.7, 2.6]} & 0.2 {\scriptsize[$-$0.1, 0.6]} & 0.4 {\scriptsize[0.03, 0.8]} & 0.6 {\scriptsize[0.2, 1.0]} & 0.2 {\scriptsize[$-$0.1, 0.6]} \\
\midrule\multicolumn{7}{l}{\emph{(b) Allocation first}} \\
\quad +\,adaptive allocation (UCB(Raw), human-only) & 17.8 {\scriptsize[16.4, 19.3]} & 20.8 {\scriptsize[19.3, 22.3]} & $-$1.5 {\scriptsize[$-$2.7, $-$0.3]} & 2.0 {\scriptsize[0.9, 3.2]} & 3.2 {\scriptsize[2.0, 4.4]} & 4.2 {\scriptsize[3.0, 5.4]} \\
\quad +\,AI correction (UCB(Raw)) & 27.7 {\scriptsize[26.7, 28.6]} & 29.0 {\scriptsize[28.1, 30.0]} & 1.9 {\scriptsize[1.5, 2.4]} & 2.6 {\scriptsize[2.2, 3.0]} & 2.8 {\scriptsize[2.4, 3.2]} & 2.9 {\scriptsize[2.6, 3.3]} \\
\quad +\,residual index (TRD-UCB) & 2.1 {\scriptsize[1.5, 2.8]} & 2.1 {\scriptsize[1.7, 2.6]} & 0.2 {\scriptsize[$-$0.1, 0.6]} & 0.4 {\scriptsize[0.03, 0.8]} & 0.6 {\scriptsize[0.2, 1.0]} & 0.2 {\scriptsize[$-$0.1, 0.6]} \\
\midrule TRD-UCB total & 47.6 {\scriptsize[45.9, 49.3]} & 52.0 {\scriptsize[50.2, 53.7]} & 0.7 {\scriptsize[$-$0.6, 2.0]} & 5.1 {\scriptsize[3.9, 6.3]} & 6.6 {\scriptsize[5.3, 7.9]} & 7.3 {\scriptsize[6.0, 8.6]} \\
Remaining gap to the Oracle & 3.7 {\scriptsize[3.1, 4.2]} & 1.6 {\scriptsize[1.3, 2.0]} & 10.5 {\scriptsize[9.4, 11.5]} & 6.4 {\scriptsize[5.5, 7.4]} & 4.6 {\scriptsize[3.7, 5.4]} & 3.4 {\scriptsize[2.7, 4.1]} \\
 
\bottomrule\end{tabular}}
\par\vspace{2pt}{\footnotesize\raggedright Notes: Each row adds one component to the policy of the row above, starting from Traditional (human-only); the policy reached by each step is given in parentheses. Positive values are reductions. All AI-augmented policies refit the learned-coefficient PPI++ estimate, and UCB(Raw) and UCB(Raw, human-only) share the same acquisition path. The Oracle row is TRD-UCB's remaining gap to the Oracle.\par}
\end{table}

\begingroup
\subsubsection{Why the residual ratio matters for allocation}\label{app:attr-index}

UCB(Raw) and TRD-UCB use the same adaptive rule of \citet{carpentier2015adaptive}, extended with label costs, and the same final PPI++ estimator; only the variance estimate entering the index differs. Write $V_q=\operatorname{Var}(Y_q)$ and $a_q=A_q/V_q$, the fraction of human-response variance remaining after AI correction. Their known-variance allocations satisfy
\begin{equation*}
n_q^{\mathrm{Raw}}\propto\sqrt{\frac{w_qV_q}{c_q}},
\qquad
n_q^{\mathrm{Residual}}\propto\sqrt{\frac{w_qV_qa_q}{c_q}}.
\end{equation*}
If $a_q=a>0$ for every question, the common factor cancels and the two allocation shares coincide. When the ratios differ, residual-based allocation shifts labels toward questions where AI leaves more variation unexplained, relative to raw-variance allocation. For example, two questions with equal $V_q$, weights, and costs but residual ratios of $0.24$ and $0.96$ receive labels in a $1{:}2$ ratio under residual allocation, compared with $1{:}1$ under raw allocation. Thus the additional allocation value depends on differences in residual ratios across questions. On the original Twin pool, using known $V_q$ instead of $A_q$ raises the objective by only $0.24\%$, consistent with the small residual-index increment in Table~\ref{tab:ec-attribution}. In the main Synthetic design, the residual index lowers realized MSE by $4.2\%$ relative to UCB(Raw) at $B=2{,}000$.

We examine this mechanism by varying the dispersion of $a_q$ while keeping $\sum_q A_q/\sum_q V_q=0.6$. In Synthetic, we set $v_q=1$, retain $\sigma_\eta=0.5$, and assign equal numbers of questions to $\rho_q^2=0.5\pm d$, where $d\in\{0,0.2,0.4,0.45\}$. This gives $V_q=1.25$ and $a_q=0.6\mp0.8d$ before the model's original clipping at $[-10,10]$.

The Twin-based semi-synthetic design retains all human responses and their sampling probabilities, replacing only the AI signals by
\begin{equation*}
\widetilde S_q=\mu_q+\sqrt{V_q}\left(r_q Z_q+\sqrt{1-r_q^2}E_q\right),
\qquad Z_q=(Y_q-\mu_q)/\sqrt{V_q}.
\end{equation*}
Independent uniform noise is centered, orthogonalized against $Z_q$, and standardized within each question pool to form $E_q$. We set $r_q^2=0.4+0.8dz_q$, so $a_q=0.6-0.8dz_q$. Random equal-size group indicators are centered with weights $V_q$ and rescaled to $\max_q|z_q|=1$; hence $\sum_q V_qz_q=0$ and the average residual fraction stays fixed. Both policies use $K=3$, $\gamma=0.08$, shared potential observation streams, and the same final PPI++ coefficient fit over $[0,1]$. The common range bound is $\bar R=20$ in Synthetic and $\bar R\approx4.17$ in the constructed Twin pools.

\begin{table}[ht]\centering\footnotesize
\caption{Residual ratios and the value of residual-based allocation, $B=2{,}000$.}
\label{tab:ec-residual-structure}
\setlength{\tabcolsep}{5pt}
\fitwidth{\begin{tabular}{lcccc}\toprule
 & & Residual ratios & \multicolumn{2}{c}{Reduction relative to UCB(Raw) (\%)} \\ \cmidrule(lr){4-5}
Design & $d$ & Low / high & Allocation objective & Realized MSE \\ \midrule
Synthetic & 0 & 0.60 / 0.60 & 0.0 {\scriptsize[$-$0.1, 0.0]} & 1.4 {\scriptsize[0.7, 2.1]} \\
 & 0.2 & 0.44 / 0.76 & 1.9 {\scriptsize[1.8, 1.9]} & 3.5 {\scriptsize[2.7, 4.2]} \\
 & 0.4 & 0.28 / 0.92 & 8.0 {\scriptsize[7.9, 8.1]} & 8.9 {\scriptsize[8.0, 9.8]} \\
 & 0.45 & 0.24 / 0.96 & 10.4 {\scriptsize[10.3, 10.5]} & 10.8 {\scriptsize[9.9, 11.7]} \\
\midrule
Twin semi-synthetic & 0 & 0.60 / 0.60 & $-$0.2 {\scriptsize[$-$0.3, $-$0.2]} & 3.4 {\scriptsize[2.2, 4.5]} \\
 & 0.2 & 0.45 / 0.75 & 1.1 {\scriptsize[1.0, 1.2]} & 3.5 {\scriptsize[2.4, 4.6]} \\
 & 0.4 & 0.30 / 0.91 & 5.2 {\scriptsize[4.9, 5.5]} & 6.7 {\scriptsize[5.2, 8.2]} \\
 & 0.45 & 0.26 / 0.94 & 6.7 {\scriptsize[6.3, 7.1]} & 7.9 {\scriptsize[6.4, 9.4]} \\
 
\bottomrule\end{tabular}}
\par\smallskip{\footnotesize\raggedright Notes: The aggregate residual fraction $\sum_qA_q/\sum_qV_q$ is $0.6$ throughout. Low / high gives the two residual ratios in Synthetic and the mean of the within-pool minimum / maximum across 20 constructed Twin pools. Brackets give paired pointwise 95\% Monte Carlo intervals from 500 independent Synthetic replications or 20 Twin pools with 25 paired replays each; Twin intervals are calculated across pools. Positive reductions favor TRD-UCB.\par}
\end{table}

Table~\ref{tab:ec-residual-structure} connects the residual ratios to allocation performance. With a common ratio of $0.60$, the residual index provides no allocation-objective improvement. As the two ratios separate to $0.24/0.96$ in Synthetic and approximately $0.26/0.94$ in the Twin-based design, the objective reductions reach $10.4\%$ and $6.7\%$; realized MSE falls by $10.8\%$ and $7.9\%$. The intermediate settings show the same increasing pattern in the allocation objective, although the average amount of unexplained variation stays fixed.

To check the role of label reuse (Appendix~\ref{app:realized-reuse}), we also retain each policy's acquired counts and refit PPI++ on 20 independent fresh label sets per replay. At $d=0$, the MSE reductions are $0.0\%$ in Synthetic and $-0.3\%$ in the Twin-based design; at $d=0.45$, they are $10.6\%$ and $5.8\%$. The gains under heterogeneous residual ratios therefore persist when estimation uses labels independent of those that determined the allocation.

The preceding designs hold the average residual fraction fixed. In practice, the human--AI correlation usually differs across questions, and how strong it is depends on the quality of the predictor. Table~\ref{tab:ec-corr-level} therefore varies both whether the correlation is weak or strong and whether it is equal across questions or differs between them. Questions are split into two equal random groups. The correlation is 0.2 or 0.8 for all questions, or differs by 0.3 between the groups around the same levels (0.05 / 0.35 and 0.65 / 0.95). In Synthetic, $Q=100$, $v_q$ is drawn as in Section~\ref{subsec:synthetic_setup} with $h=2$, and $Y_q=\sqrt{v_q}\,U$ and $S_q=\sqrt{v_q}\,(\rho_qU+\sqrt{1-\rho_q^2}\,\varepsilon)$ without additional response noise, so that $A_q=v_q(1-\rho_q^2)$. The Twin-based design keeps all human responses and uses the signal construction above with $r_q$ set to these correlations. Uniform (AI-augmented), UCB(Raw), and TRD-UCB share the observation streams and the final PPI++ estimate, with 500 Synthetic replications and 20 constructed Twin pools with 25 replays each.

\begin{table}[ht]\centering\footnotesize
\caption{Realized MSE and allocation objective ($\times10^{-2}$) by level and spread of the human--AI correlation, $B=2{,}000$. Gap to Oracle (\%) in parentheses.}
\label{tab:ec-corr-level}
\setlength{\tabcolsep}{5pt}
\fitwidth{\begin{tabular}{llcccc}\toprule
Correlation & Groups & Oracle & TRD-UCB & UCB(Raw) & Uniform (AI-aug.) \\ \midrule
\multicolumn{6}{l}{\textbf{A. Realized MSE}} \\
\multicolumn{6}{l}{\emph{Synthetic}} \\
\quad Weak, equal & 0.20 / 0.20 & 3.79 & 3.90 (2.9) & \textbf{3.86 (1.9)} & 4.94 (30.4) \\
\quad Weak, different & 0.05 / 0.35 & 3.70 & 3.81 (2.8) & \textbf{3.78 (2.1)} & 4.82 (30.3) \\
\quad Strong, equal & 0.80 / 0.80 & 1.43 & \textbf{1.48 (3.3)} & 1.50 (4.5) & 1.87 (30.1) \\
\quad Strong, different & 0.65 / 0.95 & 1.14 & \textbf{1.20 (5.4)} & 1.38 (20.8) & 1.72 (50.8) \\
\multicolumn{6}{l}{\emph{Twin semi-synthetic}} \\
\quad Weak, equal & 0.20 / 0.20 & 22.84 & \textbf{24.67 (8.0)} & 24.72 (8.2) & 25.21 (10.4) \\
\quad Weak, different & 0.05 / 0.35 & 22.29 & \textbf{24.14 (8.3)} & 24.22 (8.6) & 24.68 (10.7) \\
\quad Strong, equal & 0.80 / 0.80 & 9.20 & \textbf{10.04 (9.0)} & 10.88 (18.2) & 10.18 (10.6) \\
\quad Strong, different & 0.65 / 0.95 & 7.49 & \textbf{8.41 (12.3)} & 10.14 (35.5) & 9.38 (25.2) \\
\midrule
\multicolumn{6}{l}{\textbf{B. Allocation objective}} \\
\multicolumn{6}{l}{\emph{Synthetic}} \\
\quad Weak, equal & 0.20 / 0.20 & 3.65 & 3.76 (3.1) & \textbf{3.75 (2.9)} & 4.77 (30.8) \\
\quad Weak, different & 0.05 / 0.35 & 3.56 & 3.67 (3.1) & \textbf{3.67 (3.0)} & 4.66 (31.0) \\
\quad Strong, equal & 0.80 / 0.80 & 1.37 & 1.42 (3.9) & \textbf{1.41 (2.9)} & 1.79 (30.8) \\
\quad Strong, different & 0.65 / 0.95 & 1.09 & \textbf{1.16 (6.1)} & 1.32 (20.8) & 1.68 (53.7) \\
\multicolumn{6}{l}{\emph{Twin semi-synthetic}} \\
\quad Weak, equal & 0.20 / 0.20 & 22.31 & 22.79 (2.1) & \textbf{22.77 (2.1)} & 24.44 (9.5) \\
\quad Weak, different & 0.05 / 0.35 & 21.76 & \textbf{22.22 (2.1)} & 22.23 (2.2) & 23.85 (9.6) \\
\quad Strong, equal & 0.80 / 0.80 & 8.37 & 8.60 (2.8) & \textbf{8.54 (2.1)} & 9.16 (9.5) \\
\quad Strong, different & 0.65 / 0.95 & 6.64 & \textbf{7.22 (8.7)} & 7.96 (19.8) & 8.51 (28.1) \\
 
\bottomrule\end{tabular}}
\par\smallskip{\footnotesize\raggedright Notes: Groups give the human--AI correlations of the two halves of the questions. Synthetic sums over 100 questions and Twin over 68; all methods report the same PPI++ estimate. Bold marks the lowest value among the implementable methods. Each Synthetic setting uses 500 replications; each Twin setting uses 20 constructed pools with 25 paired replays each.\par}
\end{table}

TRD-UCB has a lower realized MSE and allocation objective than Uniform (AI-augmented) in every setting of the table. Its realized MSE is 2.8\% to 5.4\% above the Oracle's in Synthetic and 8.0\% to 12.3\% in the Twin-based design, against 30.1\% to 50.8\% and 10.4\% to 25.2\% for uniform allocation; the gain over uniform allocation is largest when the correlations are strong and differ across questions.

Relative to UCB(Raw), the residual index matters only when the correlations both differ across questions and are strong. In the other three settings, TRD-UCB and UCB(Raw) perform similarly: their gaps to the Oracle differ by at most 1.2 percentage points, except for realized MSE in the Twin-based design with strong correlations, which is affected by label reuse (see below). With weak correlations, as on the original Twin pool (correlations at most 0.58), the AI removes little of each question's human-response variance, so the residual and raw variances nearly coincide. With equal correlations, the residual fraction is the same for every question, so the two indexes target the same allocation. The small remaining differences come from learning the indexes online: the residual index uses a fitted coefficient, and its smaller variance gives the exploration bonus more weight. With strong and different correlations, the two indexes diverge: UCB(Raw)'s realized MSE is 20.8\% above the Oracle's in Synthetic and 35.5\% in the Twin-based design, against 5.4\% and 12.3\% for TRD-UCB, and the allocation objective shows the same ordering (20.8\% and 19.8\% against 6.1\% and 8.7\%). The residual index thus pays off when the AI explains human responses much better on some questions than on others, and the payoff grows with the strength of the correlation because $1-\rho_q^2$ falls faster as $\rho_q$ grows. An intermediate level of 0.5 gives intermediate results. Because the same policy is used throughout, a firm does not need to know in advance how strong or how uneven its predictor's accuracy is: TRD-UCB improves on uniform allocation in each case and adds to the gains of raw-variance allocation when the predictor is accurate on some questions but not on others.

In the two strong-correlation rows of the Twin-based design, UCB(Raw)'s realized MSE exceeds Uniform's although its allocation objective is lower, and with equal correlations TRD-UCB's realized MSE is well below UCB(Raw)'s although their allocation objectives are close. These realized differences arise in estimation from the reuse of acquired labels (Appendix~\ref{app:realized-reuse}), not from the allocation itself.

\endgroup

\subsubsection{Comparison with each primary method}\label{app:attr-paired}

Table~\ref{tab:ec-primary-paired} reports TRD-UCB's paired realized-MSE reduction relative to each implementable method of the main comparison. In Synthetic all ten contrasts favor TRD-UCB, with intervals excluding zero; the smallest is 2.8\% against TRD-TS at $B=2{,}000$. On the Twin pool TRD-UCB is better than TS-Neyman (human-only), StratPPI, and TRD-TS at every displayed budget, and better than Uniform (AI-augmented) at $B=2{,}000$, $3{,}000$, and $4{,}000$ (2.7\%, 4.1\%, and 4.7\%). At $B=1{,}000$ its contrasts with Uniform (AI-augmented) and Traditional (human-only) have intervals that include zero.

\begin{table}[ht]\centering\footnotesize
\caption{TRD-UCB's realized-MSE reduction relative to each primary method (\% of the comparator's mean), with paired 95\% Monte Carlo intervals, 500 paired replications.}
\label{tab:ec-primary-paired}
\setlength{\tabcolsep}{3pt}
\begin{tabular}{lcccccc}\toprule
 & \multicolumn{2}{c}{Synthetic} & \multicolumn{4}{c}{Twin} \\ \cmidrule(lr){2-3}\cmidrule(lr){4-7}
Comparator & 1{,}000 & 2{,}000 & 1{,}000 & 2{,}000 & 3{,}000 & 4{,}000 \\ \midrule
Traditional (human-only) & 47.6 {\scriptsize[45.9, 49.3]} & 52.0 {\scriptsize[50.2, 53.7]} & 0.7 {\scriptsize[$-$0.6, 2.0]} & 5.1 {\scriptsize[3.9, 6.3]} & 6.6 {\scriptsize[5.3, 7.9]} & 7.3 {\scriptsize[6.0, 8.6]} \\
TS-Neyman (human-only) & 38.6 {\scriptsize[37.1, 40.1]} & 41.0 {\scriptsize[39.6, 42.4]} & 4.8 {\scriptsize[4.0, 5.5]} & 6.3 {\scriptsize[5.4, 7.3]} & 5.5 {\scriptsize[4.6, 6.4]} & 4.8 {\scriptsize[4.1, 5.5]} \\
Uniform (AI-augmented) & 20.7 {\scriptsize[18.9, 22.4]} & 24.3 {\scriptsize[22.7, 25.9]} & $-$1.0 {\scriptsize[$-$2.3, 0.2]} & 2.7 {\scriptsize[1.5, 3.8]} & 4.1 {\scriptsize[2.8, 5.3]} & 4.7 {\scriptsize[3.4, 5.9]} \\
StratPPI & 6.7 {\scriptsize[5.6, 7.8]} & 4.8 {\scriptsize[3.9, 5.7]} & 5.3 {\scriptsize[4.3, 6.3]} & 3.0 {\scriptsize[2.1, 3.8]} & 1.0 {\scriptsize[0.2, 1.7]} & 1.4 {\scriptsize[0.7, 2.1]} \\
TRD-TS & 3.1 {\scriptsize[2.0, 4.1]} & 2.8 {\scriptsize[2.1, 3.6]} & 3.1 {\scriptsize[2.4, 3.9]} & 3.5 {\scriptsize[2.7, 4.3]} & 2.5 {\scriptsize[1.8, 3.3]} & 1.2 {\scriptsize[0.7, 1.7]} \\
 
\bottomrule\end{tabular}
\par\vspace{2pt}{\footnotesize\raggedright Notes: Positive values favor TRD-UCB. StratPPI uses $\alpha=0.5$ and is run separately at each budget on the same streams.\par}
\end{table}

\endgroup

\begingroup
\subsection{Realized error and the reuse of acquired labels}\label{app:realized-reuse}

On the Twin pool, TRD-UCB's realized error exceeds its allocation objective by more than that of the nonadaptive methods (Section~\ref{subsec:dt_comparison}). This section traces the excess to the reuse of acquired labels and shows how the exploration multiplier and variance shrinkage reduce it. This section and Appendix~\ref{app:supp-coverage} use the 500 main Twin streams.

\subsubsection{Reuse of acquired labels}\label{app:attr-reuse}

Table~\ref{tab:fixed-count-check} separates the allocation from the reuse of acquired labels on the Twin pool. It keeps each policy's final counts on the 500 streams and redraws independent labels at those counts. With redrawn labels, the realized error of each policy exceeds its allocation objective by at most 2.4\%, and TRD-UCB's realized error is 5.6\%, 7.2\%, and 7.9\% below that of Uniform (AI-augmented) at $B=1{,}000$, $2{,}000$, and $4{,}000$, close to the allocation-objective differences of 6.1\%, 7.4\%, and 8.0\%. In the actual runs, TRD-UCB's realized error exceeds its objective by 9.9\% at $B=1{,}000$, against 2.2\% for Uniform (AI-augmented). This excess therefore comes from the dependence between the labels that set the counts and the labels in the estimates, not from the allocation itself.

Twenty-five of the 68 questions are binary, and for a binary question the sample variance is a function of the sample mean. A running mean near an extreme can lower a variance-based index and delay additional acquisitions, although the confidence bonus and sample counts also enter the index. On the binary questions, TRD-UCB's realized error exceeds $A_q/n_q$ by 13.0\% at $B=1{,}000$, against 0.7\% under Uniform, and these questions carry 85\% of TRD-UCB's additional excess over Uniform. Keeping each policy's final counts and redrawing the labels removes this excess (Table~\ref{tab:fixed-count-check}).

\begin{table}[ht]
\centering\footnotesize
\caption{Twin realized MSE with each policy's final counts held fixed and the labels redrawn ($\times10^{-3}$ per question; 500 paired streams).}
\label{tab:fixed-count-check}
\begin{tabular}{rlrrr}\toprule
$B$ & Policy & Allocation objective & Realized MSE & Realized MSE, redrawn labels \\ \midrule
1,000 & Oracle & 6.549 & 6.686 & 6.703 \\
 & TRD-UCB & 6.799 & 7.472 & 6.954 \\
 & Uniform (AI-augmented) & 7.237 & 7.395 & 7.369 \\
\addlinespace
2,000 & Oracle & 3.274 & 3.361 & 3.319 \\
 & TRD-UCB & 3.343 & 3.605 & 3.388 \\
 & Uniform (AI-augmented) & 3.611 & 3.703 & 3.651 \\
\addlinespace
4,000 & Oracle & 1.637 & 1.680 & 1.652 \\
 & TRD-UCB & 1.656 & 1.744 & 1.665 \\
 & Uniform (AI-augmented) & 1.800 & 1.829 & 1.809 \\
\bottomrule\end{tabular}
\par\smallskip{\footnotesize\raggedright Notes: The allocation objective is $Q^{-1}\sum_qA_q/n_q$ at the policy's terminal counts. Redrawn labels keep each replay's terminal counts and draw new pairs with replacement from the same question pools, 20 times per replay; the learned-coefficient estimate is refitted on the new pairs. Monte Carlo standard errors are at most 0.064 in these units.\par}
\end{table}

\subsubsection{Exploration multiplier and variance shrinkage}
\label{app:supp-feedback}

The main configuration uses the exploration multiplier $\gamma=0.08$ (Appendix~\ref{app:ts-selection}). Table~\ref{tab:supp-gamma} reruns the baseline with $\gamma\in\{0,0.02,0.04,0.08,0.16,0.32,0.64\}$ and with the variance estimate in the index shrunk toward its cross-question mean by $n_0\in\{3,10\}$ pseudo-observations. All variants use the same streams as Table~\ref{tab:supp-baseline}.

Three results follow. First, the plug-in greedy rule ($\gamma=0$) is far worse than every other policy at every budget, with realized MSE about 4.7 times Uniform's at $B=2{,}000$; a question whose first labels happen to agree receives almost no further labels. This comparison supports the role of the confidence bonus in avoiding the displayed greedy failure. With $\gamma=0.08$, the remaining excess of realized error over the allocation objective comes mainly from binary questions that receive fewer further labels once their running means move toward 0 or 1 (Appendix~\ref{app:attr-reuse}); a larger bonus or shrinkage reduces this feedback. Second, a larger multiplier or shrinkage trades a small loss on the allocation objective for a larger gain in realized error: at $B=2{,}000$, moving from $\gamma=0.08$ to $0.32$ raises the objective from $3.343$ to $3.390$ but lowers realized MSE from $3.605$ to $3.535$, and with $\gamma=0.32$ or $n_0=10$ TRD-UCB has lower realized MSE than Uniform at every budget, including $B=500$ and $1{,}000$, where the selected setting does not. Shrinkage with $n_0=3$ leaves the allocation objective essentially unchanged ($6.802$ against $6.799$ at $B=1{,}000$) and gives TRD-UCB lower realized MSE than Uniform at $B=1{,}000$, $2{,}000$, and $4{,}000$, by $1.5\%$, $4.0\%$, and $5.3\%$. For $B\ge1{,}000$, the excess of realized error over the objective falls from approximately $5$--$10\%$ at $\gamma=0.08$ to about $4$--$5\%$ at $\gamma=0.32$, against $1.6$--$2.7\%$ for the nonadaptive Oracle and Uniform. Third, the largest relative spread in mean realized MSE across $\gamma\in[0.08,0.64]$ is $3.4\%$ over the four displayed budgets, indicating limited sensitivity over this tested range on this pool.

Table~\ref{tab:syn-gamma} reports 200 separate paired replications under the same Synthetic design ($h=2$, $\rho_q\sim\mathrm{Beta}(2,1)$, $B=2{,}000$). TRD-UCB's realized gap to the Oracle on these replications is $2.8\%$ at $\gamma=0.08$ and $6.1\%$ at $\gamma=0.32$; across $\gamma\in[0.08,0.64]$, Synthetic mean realized MSE spans $30.07$ to $31.91$ ($\times10^{-3}$), a relative spread of $6.1\%$. The main policy keeps $\gamma=0.08$.
\begin{table}[p]
\centering\footnotesize
\setlength{\tabcolsep}{3pt}
\caption{Twin baseline with 500 replays: exploration multiplier and variance shrinkage.}
\label{tab:supp-gamma}
\fitwidth{\begin{tabular}{lrrrrrrrr}
\toprule
Policy & \multicolumn{2}{c}{$B=500$} & \multicolumn{2}{c}{$B=1{,}000$} & \multicolumn{2}{c}{$B=2{,}000$} & \multicolumn{2}{c}{$B=4{,}000$} \\
\cmidrule(lr){2-3} \cmidrule(lr){4-5} \cmidrule(lr){6-7} \cmidrule(lr){8-9}
 & MSE & Obj. & MSE & Obj. & MSE & Obj. & MSE & Obj. \\
\midrule
Oracle & 13.574 & 13.121 & 6.686 & 6.549 & 3.361 & 3.274 & 1.680 & 1.637 \\
Uniform (AI-aug.) & 14.942 & 14.633 & 7.395 & 7.237 & 3.703 & 3.611 & 1.829 & 1.800 \\
Greedy ($\gamma=0$) & 23.951 & 19.322 & 19.369 & 14.828 & 17.228 & 12.789 & 16.239 & 11.881 \\
TRD-UCB, $\gamma=0.02$ & 16.214 & 14.160 & 8.343 & 7.084 & 3.966 & 3.432 & 1.845 & 1.678 \\
TRD-UCB, $\gamma=0.04$ & 16.027 & 14.145 & 7.781 & 6.841 & 3.758 & 3.369 & 1.771 & 1.661 \\
TRD-UCB, $\gamma=0.08$ (selected) & 14.980 & 13.807 & 7.472 & 6.799 & 3.605 & 3.343 & 1.744 & 1.656 \\
TRD-UCB, $\gamma=0.16$ & 15.085 & 13.963 & 7.302 & 6.826 & 3.557 & 3.356 & 1.730 & 1.660 \\
TRD-UCB, $\gamma=0.32$ & 14.612 & 14.039 & 7.249 & 6.898 & 3.535 & 3.390 & 1.738 & 1.673 \\
TRD-UCB, $\gamma=0.64$ & 14.615 & 14.045 & 7.230 & 6.964 & 3.561 & 3.434 & 1.746 & 1.694 \\
TRD-UCB, $\gamma=0.08$, shrink $n_0=3$ & 15.038 & 13.981 & 7.285 & 6.802 & 3.554 & 3.345 & 1.732 & 1.656 \\
TRD-UCB, $\gamma=0.08$, shrink $n_0=10$ & 14.582 & 14.032 & 7.233 & 6.874 & 3.535 & 3.365 & 1.726 & 1.661 \\
\bottomrule\end{tabular}}
\par\vspace{2pt}{\footnotesize\raggedright Notes: 500 paired replays of the 68-question pool with $w_q=c_q=1$ and $K=3$; the same replays as Table~\ref{tab:supp-baseline}. MSE is mean realized squared error of the fitted estimator and Obj.\ is the allocation objective $Q^{-1}\sum_qA_q/n_q$, both in $10^{-3}$ units. Greedy uses the plug-in index without a bonus. Shrinkage replaces $\widehat A_q$ in the index by $\{(n_q-1)\widehat A_q+n_0\bar A\}/(n_q-1+n_0)$, where $\bar A$ is the current cross-question mean. Monte Carlo SEs of the displayed MSE means are at most 0.250 in these units.\par}
\end{table}

\begin{table}[ht]\centering\footnotesize
\caption{Synthetic sensitivity of realized MSE ($\times10^{-3}$) to the exploration multiplier ($B=2{,}000$).}
\label{tab:syn-gamma}
\begin{tabular}{lrrr}\toprule
Setting & Realized MSE & MC SE & Gap to Oracle (\%) \\ \midrule
TRD-UCB $\gamma=0.02$ & 30.29 & 0.45 & 3.6 \\
TRD-UCB $\gamma=0.04$ & 30.14 & 0.44 & 3.1 \\
TRD-UCB $\gamma=0.08$ & 30.07 & 0.44 & 2.8 \\
TRD-UCB $\gamma=0.16$ & 30.39 & 0.45 & 3.9 \\
TRD-UCB $\gamma=0.32$ & 31.02 & 0.47 & 6.1 \\
TRD-UCB $\gamma=0.64$ & 31.91 & 0.49 & 9.1 \\
\bottomrule\end{tabular}
\par\smallskip{\footnotesize\raggedright Notes: The baseline Synthetic design uses $Q=100$, $h=2$, and question-specific $\rho_q\sim\mathrm{Beta}(2,1)$. These 200 paired replications are separate from tuning and from the 500 replications of Section~\ref{sec:synthetic}. Gaps use the mean realized MSE of the feasible Oracle on these same streams.\par}
\end{table}

\endgroup

\begingroup
\subsection{Bias and coverage by final sample size}
\label{app:supp-coverage}

The intervals below are evaluated on the 500 main Twin streams. 
We evaluate two nominal 95\% intervals for each question's final estimate. The plug-in interval is
\begin{equation}
\widehat\mu_q\ \pm\ 1.96\sqrt{\widehat v_q/n_q},\qquad
\widehat v_q=\max\{s_{Y,q}^2-2\widehat\lambda_qs_{YS,q}+\widehat\lambda_q^2s_{S,q}^2,0\}.
\end{equation}
The moderated $t$ interval, in the form of \citet{smyth2004linear}, shrinks $\widehat v_q$ toward the cross-question mean $\bar v$ of the fitted variances at termination by three pseudo-observations and replaces 1.96 by the $t$ quantile with the corresponding degrees of freedom:
\begin{equation}
\widehat\mu_q\ \pm\ t_{0.975,\,n_q+2}\sqrt{\tilde v_q/n_q},\qquad
\tilde v_q=\frac{(n_q-1)\widehat v_q+3\bar v}{n_q+2}.
\end{equation}
The positive floor used in the acquisition index is not applied to either interval. When all sampled answers to a question agree, $\widehat v_q=0$ and the plug-in interval has zero width; these cases remain in the coverage calculation. Coverage is reported at budgets $1{,}000$, $2{,}000$, $4{,}000$, and $8{,}000$, about 15 to 118 labels per question; the $B=8{,}000$ runs use the same streams.

Table~\ref{tab:supp-coverage} reports both intervals for the six policies of this configuration, all of which report the AI-augmented estimate. With the moderated $t$ interval, coverage lies between 92.42\% and 94.76\% for every policy and budget. TRD-UCB's coverage is 93.11\%, 93.82\%, 93.76\%, and 94.34\% at the four budgets, against 94.32\%, 93.77\%, 94.60\%, and 94.70\% for Uniform and 94.74\%, 94.57\%, 94.46\%, and 94.76\% for the Oracle. The plug-in interval covers less at small counts for every policy, including the Oracle, whose counts do not depend on the labels (88.41\% at $B=1{,}000$), and approaches the nominal level as the budget grows (93.76\% for TRD-UCB and 94.35\% for Uniform at $B=8{,}000$). Two features of small samples produce this shortfall. At these counts the multiplier 1.96 is below the $t$ quantile, and zero-width plug-in intervals, almost all from binary questions whose sampled answers agree, make up 4.62\% of TRD-UCB's and 2.65\% of Uniform's intervals at $B=1{,}000$. On the 43 non-binary questions, moderated $t$ coverage lies between 94.49\% and 95.75\% for every policy and budget.

\begin{table}[ht]
\centering\footnotesize
\setlength{\tabcolsep}{4pt}
\caption{Twin baseline with 500 replays: empirical coverage (\%) of nominal 95\% intervals.}
\label{tab:supp-coverage}
\fitwidth{\begin{tabular}{lcccccccc}
\toprule
Policy & \multicolumn{2}{c}{$B=1{,}000$} & \multicolumn{2}{c}{$B=2{,}000$} & \multicolumn{2}{c}{$B=4{,}000$} & \multicolumn{2}{c}{$B=8{,}000$} \\
\cmidrule(lr){2-3} \cmidrule(lr){4-5} \cmidrule(lr){6-7} \cmidrule(lr){8-9}
 & Plug-in & Moderated $t$ & Plug-in & Moderated $t$ & Plug-in & Moderated $t$ & Plug-in & Moderated $t$ \\
\midrule
Oracle & 88.41 & 94.74 & 91.69 & 94.57 & 93.42 & 94.46 & 94.11 & 94.76 \\
TRD-UCB (Ours) & 86.33 & 93.11 & 90.31 & 93.82 & 92.42 & 93.76 & 93.76 & 94.34 \\
TRD-TS & 85.40 & 92.66 & 89.77 & 93.53 & 92.08 & 93.82 & 93.45 & 94.08 \\
StratPPI & 84.26 & 92.42 & 89.38 & 93.49 & 92.20 & 93.57 & 93.70 & 94.24 \\
TRD-$\varepsilon$-Greedy & 87.03 & 93.53 & 89.41 & 93.51 & 91.32 & 93.58 & 92.76 & 93.98 \\
Uniform (AI-aug.) & 89.32 & 94.32 & 92.27 & 93.77 & 93.76 & 94.60 & 94.35 & 94.70 \\
\bottomrule\end{tabular}}
\par\vspace{2pt}{\footnotesize\raggedright Notes: The same 500 replays as Table~\ref{tab:supp-baseline}; the $B=8{,}000$ runs use the same streams. The plug-in interval is $\widehat\mu_q\pm1.96\sqrt{\widehat v_q/n_q}$. The moderated $t$ interval is $\widehat\mu_q\pm t_{0.975,n_q+2}\sqrt{\tilde v_q/n_q}$ with $\tilde v_q=\{(n_q-1)\widehat v_q+3\bar v\}/(n_q+2)$, where $\bar v$ is the cross-question mean of $\widehat v_q$ at termination. Coverage Monte Carlo SEs are at most 0.17 percentage points.\par}
\end{table}

Tables~\ref{tab:supp-count-bins} and~\ref{tab:supp-count-bins-high} report bias, MSE, and the coverage of both intervals in the terminal-count bins $3$, $4$--$5$, $6$--$9$, $10$--$19$, $20$--$39$, $40$--$79$, and $\ge80$ for TRD-UCB, Uniform, and StratPPI; every nonempty budget--bin cell is shown. For a pooled bin statistic $m=\sum_rS_r/\sum_rD_r$, where $D_r$ is the number of questions in that bin and $S_r$ is the sum of their outcomes, its Monte Carlo SE uses the replay-level influence $(S_r-mD_r)/\bar D$, retaining replays with no contributing questions. Signed biases can cancel across questions.

For $B\ge2{,}000$, the shortfall from 95\% concentrates in TRD-UCB's lowest-count bin. At $B=2{,}000$, the bin with $n_q=10$--$19$ holds 4.45\% of TRD-UCB's question--replay observations and has moderated $t$ coverage of 73.21\% (plug-in 36.11\%, with zero estimated SE in 46.23\% of these observations). The lowest bins at $B=4{,}000$ ($n_q=20$--$39$, 6.96\% of observations) and $B=8{,}000$ ($n_q=40$--$79$, 9.01\%) have moderated $t$ coverage of 78.99\% and 88.34\%, and TRD-UCB's other bins at these budgets have 94.38\% to 96.80\%. The lowest bin contains the questions with the smallest difficulty, which the Oracle also samples least, together with replays in which a question's sampled answers varied less than usual. TRD-UCB's index and the interval use the same fitted variance, so a small fitted variance both reduces further acquisitions and narrows the interval.

\begin{table}[p]
\centering\footnotesize
\setlength{\tabcolsep}{4pt}
\renewcommand{\arraystretch}{1.08}
\caption{Twin baseline with 500 replays: diagnostics by terminal sample count, $B=1{,}000$.}
\label{tab:supp-count-bins}
\fitwidth{\begin{tabular}{rlrrrrrrr}
\toprule
& & & & & & \multicolumn{2}{c}{Coverage \% (SE)} & \\
\cmidrule(lr){7-8}
$B$ & Policy & $n_q$ bin & $N$ & Bias (SE) & MSE & Plug-in & Moderated $t$ & Zero SE \% \\
\midrule
$1{,}000$ & TRD-UCB (Ours) & $10\text{--}19$ & $29{,}394$ & $-9.480\;(0.497)$ & $6.947$ & $84.80\;(0.20)$ & $92.58\;(0.15)$ & $5.34$ \\
$1{,}000$ & TRD-UCB (Ours) & $20\text{--}39$ & $4{,}606$ & $40.189\;(1.301)$ & $10.827$ & $96.09\;(0.29)$ & $96.44\;(0.28)$ & $0.00$ \\
$1{,}000$ & Uniform (AI-aug.) & $10\text{--}19$ & $34{,}000$ & $-0.306\;(0.482)$ & $7.395$ & $89.32\;(0.16)$ & $94.32\;(0.12)$ & $2.65$ \\
$1{,}000$ & StratPPI & $6\text{--}9$ & $10{,}138$ & $-3.825\;(1.017)$ & $10.661$ & $61.35\;(0.43)$ & $84.32\;(0.32)$ & $18.21$ \\
$1{,}000$ & StratPPI & $10\text{--}19$ & $15{,}538$ & $-10.392\;(0.614)$ & $6.150$ & $93.01\;(0.21)$ & $95.60\;(0.17)$ & $0.00$ \\
$1{,}000$ & StratPPI & $20\text{--}39$ & $8{,}324$ & $12.856\;(0.939)$ & $7.763$ & $95.84\;(0.22)$ & $96.35\;(0.21)$ & $0.00$ \\
\bottomrule\end{tabular}}
\par\vspace{2pt}{\footnotesize\raggedright Notes: Every nonempty bin for TRD-UCB, Uniform, and StratPPI is shown. $N$ counts question--replay observations, not independent replays. Bias and MSE use $10^{-3}$ units. Bias is the pooled signed error; MSE is the pooled squared error. Parenthesized SEs use replay-cluster ratio influences, including zero-contribution replays. Coverage is for the nominal 95\% plug-in and moderated $t$ intervals of Table~\ref{tab:supp-coverage}; coverage SEs are percentage points. Zero SE \% is the share of plug-in intervals with zero width.\par}
\end{table}

\begin{table}[p]
\centering\footnotesize
\setlength{\tabcolsep}{4pt}
\renewcommand{\arraystretch}{1.08}
\caption{Twin baseline with 500 replays: diagnostics by terminal sample count, $B=2{,}000$, $4{,}000$, and $8{,}000$.}
\label{tab:supp-count-bins-high}
\fitwidth{\begin{tabular}{rlrrrrrrr}
\toprule
& & & & & & \multicolumn{2}{c}{Coverage \% (SE)} & \\
\cmidrule(lr){7-8}
$B$ & Policy & $n_q$ bin & $N$ & Bias (SE) & MSE & Plug-in & Moderated $t$ & Zero SE \% \\
\midrule
$2{,}000$ & TRD-UCB (Ours) & $10\text{--}19$ & $1{,}512$ & $-27.907\;(2.085)$ & $6.309$ & $36.11\;(1.15)$ & $73.21\;(1.09)$ & $46.23$ \\
$2{,}000$ & TRD-UCB (Ours) & $20\text{--}39$ & $27{,}054$ & $-3.764\;(0.328)$ & $3.175$ & $92.10\;(0.17)$ & $94.38\;(0.15)$ & $0.00$ \\
$2{,}000$ & TRD-UCB (Ours) & $40\text{--}79$ & $5{,}434$ & $19.817\;(0.896)$ & $4.991$ & $96.49\;(0.24)$ & $96.80\;(0.24)$ & $0.00$ \\
$2{,}000$ & Uniform (AI-aug.) & $20\text{--}39$ & $34{,}000$ & $-0.094\;(0.330)$ & $3.703$ & $92.27\;(0.14)$ & $93.77\;(0.13)$ & $0.87$ \\
$2{,}000$ & StratPPI & $10\text{--}19$ & $7{,}049$ & $2.192\;(0.787)$ & $4.410$ & $72.10\;(0.49)$ & $86.88\;(0.40)$ & $11.19$ \\
$2{,}000$ & StratPPI & $20\text{--}39$ & $19{,}330$ & $-6.168\;(0.396)$ & $3.263$ & $93.12\;(0.18)$ & $94.91\;(0.16)$ & $0.00$ \\
$2{,}000$ & StratPPI & $40\text{--}79$ & $7{,}621$ & $8.629\;(0.700)$ & $4.218$ & $95.87\;(0.24)$ & $95.98\;(0.23)$ & $0.00$ \\
\midrule
$4{,}000$ & TRD-UCB (Ours) & $20\text{--}39$ & $2{,}366$ & $-8.792\;(0.794)$ & $1.655$ & $71.17\;(0.92)$ & $78.99\;(0.83)$ & $13.57$ \\
$4{,}000$ & TRD-UCB (Ours) & $40\text{--}79$ & $25{,}490$ & $-2.030\;(0.257)$ & $1.565$ & $93.60\;(0.15)$ & $94.61\;(0.14)$ & $0.00$ \\
$4{,}000$ & TRD-UCB (Ours) & $\ge80$ & $6{,}144$ & $8.920\;(0.576)$ & $2.524$ & $95.74\;(0.26)$ & $95.93\;(0.25)$ & $0.00$ \\
$4{,}000$ & Uniform (AI-aug.) & $40\text{--}79$ & $34{,}000$ & $-0.026\;(0.236)$ & $1.829$ & $93.76\;(0.13)$ & $94.60\;(0.13)$ & $0.18$ \\
$4{,}000$ & StratPPI & $20\text{--}39$ & $6{,}025$ & $0.596\;(0.563)$ & $1.732$ & $82.07\;(0.50)$ & $86.46\;(0.44)$ & $4.46$ \\
$4{,}000$ & StratPPI & $40\text{--}79$ & $20{,}330$ & $-2.844\;(0.286)$ & $1.597$ & $93.97\;(0.16)$ & $94.93\;(0.15)$ & $0.00$ \\
$4{,}000$ & StratPPI & $\ge80$ & $7{,}645$ & $5.074\;(0.514)$ & $2.256$ & $95.45\;(0.23)$ & $95.57\;(0.23)$ & $0.00$ \\
\midrule
$8{,}000$ & TRD-UCB (Ours) & $40\text{--}79$ & $3{,}063$ & $-3.641\;(0.425)$ & $0.615$ & $86.19\;(0.64)$ & $88.34\;(0.59)$ & $4.24$ \\
$8{,}000$ & TRD-UCB (Ours) & $\ge80$ & $30{,}937$ & $-0.025\;(0.170)$ & $0.868$ & $94.51\;(0.12)$ & $94.94\;(0.12)$ & $0.00$ \\
$8{,}000$ & Uniform (AI-aug.) & $\ge80$ & $34{,}000$ & $0.033\;(0.165)$ & $0.915$ & $94.35\;(0.12)$ & $94.70\;(0.11)$ & $0.01$ \\
$8{,}000$ & StratPPI & $40\text{--}79$ & $5{,}056$ & $-1.242\;(0.347)$ & $0.665$ & $88.84\;(0.42)$ & $90.59\;(0.39)$ & $1.21$ \\
$8{,}000$ & StratPPI & $\ge80$ & $28{,}944$ & $-0.085\;(0.179)$ & $0.880$ & $94.55\;(0.13)$ & $94.88\;(0.13)$ & $0.00$ \\
\bottomrule\end{tabular}}
\par\vspace{2pt}{\footnotesize\raggedright Notes: Every nonempty bin for TRD-UCB, Uniform, and StratPPI is shown. $N$ counts question--replay observations, not independent replays. Bias and MSE use $10^{-3}$ units. Bias is the pooled signed error; MSE is the pooled squared error. Parenthesized SEs use replay-cluster ratio influences, including zero-contribution replays. Coverage is for the nominal 95\% plug-in and moderated $t$ intervals of Table~\ref{tab:supp-coverage}; coverage SEs are percentage points. Zero SE \% is the share of plug-in intervals with zero width.\par}
\end{table}

\endgroup

\clearpage
\section{Robustness Checks}\label{app:supplemental-validation}

\begingroup

This section collects the robustness checks, grouped by question: weights, costs, and difficulty that move together (Appendix~\ref{app:robust-joint}), changes in the respondent distribution (Appendix~\ref{app:robust-shift}), disjoint task halves (Appendix~\ref{app:supp-task-split}), and the coefficient range (Appendix~\ref{app:twin-lambda-range}). The Twin checks use the 500 paired replay streams of the main comparison under the allocation rules of Section~\ref{subsec:synthetic_setup}, including the TS adaptation in Appendix~\ref{app:ts-unified}; all constants are those of the main comparison, and Uniform denotes Uniform (AI-augmented). These replays include only the AI-augmented policies and the Oracle, not the two human-only methods. Appendices~\ref{app:supp-feedback} and~\ref{app:supp-coverage} use the same streams.

\label{app:supp-implementation}

The replays follow the Twin protocol of Section~\ref{subsec:dt_setup}: each question's available human--AI pairs are resampled with replacement, every policy starts with three initial labels per question, TRD-UCB uses the practical index of Eq.~\eqref{eq:practical_ucb} with $\gamma=0.08$ and $\bar R=2$, and all final estimates refit the projected coefficient with all acquired labels. Each replay and question has its own stream of draws, shared across policies, budgets, and scenarios. We report weighted realized MSE separately from the allocation objective; Monte Carlo SEs use independent replays, and comparisons use paired differences with pointwise 95\% Monte Carlo intervals, without a multiple-comparison adjustment.

At $B=500$ and $1{,}000$, Uniform has slightly lower mean realized MSE than TRD-UCB, and both paired intervals include zero; at $B=2{,}000$ and $4{,}000$, TRD-UCB has lower mean MSE, with paired intervals that exclude zero. Within this configuration, TRD-UCB has lower mean MSE than StratPPI and TRD-TS at every budget, with paired intervals that exclude zero; it also has lower mean MSE than TRD-$\varepsilon$-Greedy at every budget except $B=1{,}000$, where the paired interval includes zero.

Tables~\ref{tab:supp-baseline} and~\ref{tab:supp-baseline-paired} give the baseline outcomes and paired MSE contrasts.
\begin{table}[p]
\centering\footnotesize
\setlength{\tabcolsep}{4pt}
\renewcommand{\arraystretch}{1.08}
\caption{Twin baseline with 500 replays: fitted-estimator error, allocation proxy, and bias.}
\label{tab:supp-baseline}
\fitwidth{\begin{tabular}{rlrrr}
\toprule
$B$ & Policy & MSE (SE) & Proxy (SE) & Bias (SE) \\
\midrule
$500$ & TRD-UCB (Ours) & $14.980\;(0.120)$ & $13.807\;(0.006)$ & $-3.626\;(0.671)$ \\
$500$ & Uniform (AI-aug.) & $\mathbf{14.942}\;(0.128)$ & $14.633\;(0.000)$ & $-0.094\;(0.669)$ \\
$500$ & StratPPI & $16.331\;(0.140)$ & $14.242\;(0.015)$ & $-5.747\;(0.695)$ \\
$500$ & TRD-TS & $15.645\;(0.131)$ & $14.051\;(0.010)$ & $-4.219\;(0.689)$ \\
$500$ & TRD-$\varepsilon$-Greedy & $16.137\;(0.151)$ & $15.676\;(0.020)$ & $0.391\;(0.697)$ \\
$500$ & Oracle & $13.574\;(0.109)$ & $13.121\;(0.000)$ & $0.303\;(0.628)$ \\
\midrule
$1{,}000$ & TRD-UCB (Ours) & $7.472\;(0.061)$ & $6.799\;(0.003)$ & $-2.752\;(0.487)$ \\
$1{,}000$ & Uniform (AI-aug.) & $\mathbf{7.395}\;(0.064)$ & $7.237\;(0.000)$ & $-0.306\;(0.482)$ \\
$1{,}000$ & StratPPI & $7.890\;(0.063)$ & $6.970\;(0.005)$ & $-2.742\;(0.492)$ \\
$1{,}000$ & TRD-TS & $7.714\;(0.064)$ & $6.864\;(0.004)$ & $-3.105\;(0.490)$ \\
$1{,}000$ & TRD-$\varepsilon$-Greedy & $7.465\;(0.060)$ & $6.949\;(0.005)$ & $-2.635\;(0.484)$ \\
$1{,}000$ & Oracle & $6.686\;(0.055)$ & $6.549\;(0.000)$ & $-0.111\;(0.460)$ \\
\midrule
$2{,}000$ & TRD-UCB (Ours) & $\mathbf{3.605}\;(0.031)$ & $3.343\;(0.001)$ & $-1.069\;(0.324)$ \\
$2{,}000$ & Uniform (AI-aug.) & $3.703\;(0.033)$ & $3.611\;(0.000)$ & $-0.094\;(0.330)$ \\
$2{,}000$ & StratPPI & $3.715\;(0.032)$ & $3.399\;(0.002)$ & $-1.118\;(0.332)$ \\
$2{,}000$ & TRD-TS & $3.736\;(0.033)$ & $3.369\;(0.002)$ & $-1.430\;(0.330)$ \\
$2{,}000$ & TRD-$\varepsilon$-Greedy & $3.749\;(0.032)$ & $3.392\;(0.003)$ & $-1.177\;(0.337)$ \\
$2{,}000$ & Oracle & $3.361\;(0.028)$ & $3.274\;(0.000)$ & $0.129\;(0.318)$ \\
\midrule
$4{,}000$ & TRD-UCB (Ours) & $\mathbf{1.744}\;(0.015)$ & $1.656\;(0.000)$ & $-0.522\;(0.235)$ \\
$4{,}000$ & Uniform (AI-aug.) & $1.829\;(0.017)$ & $1.800\;(0.000)$ & $-0.026\;(0.236)$ \\
$4{,}000$ & StratPPI & $1.769\;(0.014)$ & $1.668\;(0.000)$ & $-0.454\;(0.230)$ \\
$4{,}000$ & TRD-TS & $1.766\;(0.015)$ & $1.662\;(0.001)$ & $-0.622\;(0.234)$ \\
$4{,}000$ & TRD-$\varepsilon$-Greedy & $1.886\;(0.019)$ & $1.694\;(0.002)$ & $-0.832\;(0.243)$ \\
$4{,}000$ & Oracle & $1.680\;(0.014)$ & $1.637\;(0.000)$ & $0.107\;(0.227)$ \\
\bottomrule\end{tabular}}
\par\vspace{2pt}{\footnotesize\raggedright Notes: Each row uses 500 paired replays of the 68-question fixed pool with the parameters of Section~\ref{subsec:synthetic_setup} (TRD-UCB $\gamma=0.08$, StratPPI $\alpha=0.5$, TRD-TS $\nu_0=2$, TRD-$\varepsilon$-Greedy $c=5$). MSE, proxy, and signed bias are in $10^{-3}$ units, with Monte Carlo SEs across replays in parentheses. MSE averages squared question errors, the proxy averages $A_q/n_q$, and bias averages signed errors. Table~\ref{tab:supp-coverage} reports interval coverage. Bold is the smallest non-Oracle mean MSE at each budget.\par}
\end{table}

\begin{table}[p]
\centering\footnotesize
\setlength{\tabcolsep}{4pt}
\renewcommand{\arraystretch}{1.08}
\caption{Paired realized-MSE contrasts at every budget of Table~\ref{tab:supp-baseline}.}
\label{tab:supp-baseline-paired}
\fitwidth{\begin{tabular}{rlrrc}
\toprule
$B$ & Contrast & Mean difference & Paired SE & Pointwise 95\% MC interval \\
\midrule
$500$ & Uniform $-$ TRD-UCB & $-0.038$ & $0.079$ & $[-0.194,0.118]$ \\
$500$ & StratPPI $-$ TRD-UCB & $1.351$ & $0.097$ & $[1.161,1.541]$ \\
$500$ & TRD-TS $-$ TRD-UCB & $0.665$ & $0.064$ & $[0.539,0.790]$ \\
$500$ & TRD-$\varepsilon$-Greedy $-$ TRD-UCB & $1.158$ & $0.114$ & $[0.935,1.380]$ \\
\midrule
$1{,}000$ & Uniform $-$ TRD-UCB & $-0.077$ & $0.047$ & $[-0.170,0.015]$ \\
$1{,}000$ & StratPPI $-$ TRD-UCB & $0.418$ & $0.041$ & $[0.338,0.497]$ \\
$1{,}000$ & TRD-TS $-$ TRD-UCB & $0.241$ & $0.028$ & $[0.186,0.297]$ \\
$1{,}000$ & TRD-$\varepsilon$-Greedy $-$ TRD-UCB & $-0.008$ & $0.034$ & $[-0.074,0.059]$ \\
\midrule
$2{,}000$ & Uniform $-$ TRD-UCB & $0.098$ & $0.022$ & $[0.055,0.142]$ \\
$2{,}000$ & StratPPI $-$ TRD-UCB & $0.110$ & $0.016$ & $[0.079,0.141]$ \\
$2{,}000$ & TRD-TS $-$ TRD-UCB & $0.131$ & $0.016$ & $[0.100,0.162]$ \\
$2{,}000$ & TRD-$\varepsilon$-Greedy $-$ TRD-UCB & $0.144$ & $0.017$ & $[0.111,0.178]$ \\
\midrule
$4{,}000$ & Uniform $-$ TRD-UCB & $0.085$ & $0.012$ & $[0.062,0.108]$ \\
$4{,}000$ & StratPPI $-$ TRD-UCB & $0.025$ & $0.006$ & $[0.012,0.038]$ \\
$4{,}000$ & TRD-TS $-$ TRD-UCB & $0.022$ & $0.005$ & $[0.013,0.031]$ \\
$4{,}000$ & TRD-$\varepsilon$-Greedy $-$ TRD-UCB & $0.141$ & $0.013$ & $[0.117,0.166]$ \\
\bottomrule\end{tabular}}
\par\vspace{2pt}{\footnotesize\raggedright Notes: Differences, SEs, and interval endpoints use $10^{-3}$ MSE units. Differences subtract TRD-UCB from the named comparator on each of the same 500 replays; positive values favor TRD-UCB. Intervals are pointwise normal Monte Carlo intervals.\par}
\end{table}

\endgroup

\begingroup
\subsection{Joint dependence of importance, cost, and difficulty}\label{app:robust-joint}

Observed priorities and prices are not available, so these checks construct weights and costs that rise or fall with difficulty, first in the Synthetic design and then on the Twin pool.

\subsubsection{Synthetic design}
\label{app:syn-joint}

The baseline Synthetic design uses equal weights and unit costs. In each replication, let $r_q\in[-1,1]$ be the centered rank of $A_q$; for each sign pair $(s_w,s_c)\in\{-1,+1\}^2$ we set $w_q=(1+0.75s_wr_q)/Q$ and $c_q=1+0.75s_cr_q$, so that weights average $1/Q$ as in the baseline and costs lie between 0.25 and 1.75. A control permutes the same rank grid independently for weights and costs. Policy constants are those of the main comparison, and each budget from 1{,}000 to 4{,}000 is a separate run on the same 500 replications.

Table~\ref{tab:syn-joint} reports the gaps to the Oracle. Outside the boundary scenario $(-,+)$ discussed below, TRD-UCB has the smallest expected-MSE gap among the implementable methods at every budget, and the smallest realized-MSE gap except at $(+,+)$ and $B=1{,}000$, where TRD-TS is lower by an unresolved margin. When importance rises and cost falls with difficulty, the gain over uniform allocation is largest: TRD-UCB's realized MSE is about 60\% below that of Uniform (AI-augmented).

The scenario $(-,+)$, in which harder questions are less important and more costly to label, is a boundary case. The ratio $w_q/c_q$ then offsets the differences in $A_q$, the optimal shares are close to uniform, and Uniform (AI-augmented) is within about 1\% of the Oracle. Learning the allocation has little to recover, while exploration still costs labels: TRD-UCB's realized MSE exceeds that of Uniform (AI-augmented) by 7.2\%, 2.5\%, and 1.6\% at $B=1{,}000$, $2{,}000$, and $3{,}000$, and the difference is not resolved at $B=4{,}000$. In managerial terms, when priorities and prices already steer validation away from the hardest tasks, uniform allocation is nearly optimal, and the cost of learning the allocation shrinks as the budget grows. In every scenario, TRD-UCB's realized MSE is 35\% to 74\% below that of Traditional (human-only).
\begin{table}[p]
\centering\footnotesize
\setlength{\tabcolsep}{4pt}
\renewcommand{\arraystretch}{1.05}
\caption{Synthetic weight--cost--difficulty dependence: gap (\%) to the same-path Oracle, 500 paired replications per condition.}
\label{tab:syn-joint}
\fitwidth{\begin{tabular}{llrrrrrrr}
\toprule
Scenario & $B$ & \makecell{Traditional\\(human-only)} & \makecell{TS-Neyman\\(human-only)} & \makecell{Uniform\\(AI-aug.)} & StratPPI & TRD-TS & \makecell{TRD-$\varepsilon$-\\Greedy} & \makecell{TRD-UCB\\(Ours)} \\
\midrule
\multicolumn{9}{l}{\emph{Panel A. Expected MSE}}\\
\midrule
$(+,+)$ & $1{,}000$ & 100.1 & 79.5 & 24.5 & 14.9 & 9.3 & 24.8 & \textbf{8.3} \\
 & $2{,}000$ & 100.3 & 73.1 & 24.7 & 7.7 & 4.9 & 6.7 & \textbf{3.3} \\
 & $3{,}000$ & 100.4 & 70.8 & 24.7 & 5.3 & 3.3 & 3.2 & \textbf{2.1} \\
 & $4{,}000$ & 100.4 & 69.5 & 24.7 & 4.1 & 2.5 & 1.9 & \textbf{1.5} \\
\midrule
$(+,-)$ & $1{,}000$ & 300.4 & 110.2 & 149.2 & 31.2 & 20.8 & 131.0 & \textbf{6.4} \\
 & $2{,}000$ & 306.6 & 96.3 & 153.1 & 18.6 & 12.3 & 36.6 & \textbf{3.6} \\
 & $3{,}000$ & 306.9 & 90.1 & 153.3 & 14.9 & 8.5 & 16.9 & \textbf{2.4} \\
 & $4{,}000$ & 306.9 & 86.8 & 153.3 & 13.0 & 6.5 & 9.7 & \textbf{1.8} \\
\midrule
$(-,+)$ & $1{,}000$ & 84.4 & 84.6 & \textbf{1.1} & 8.9 & 5.1 & 6.5 & 8.9 \\
 & $2{,}000$ & 84.6 & 81.0 & \textbf{1.1} & 5.4 & 2.7 & 3.1 & 3.7 \\
 & $3{,}000$ & 84.6 & 79.5 & \textbf{1.2} & 3.6 & 1.8 & 2.3 & 2.2 \\
 & $4{,}000$ & 84.6 & 78.8 & \textbf{1.2} & 2.7 & 1.4 & 1.6 & 1.5 \\
\midrule
$(-,-)$ & $1{,}000$ & 152.2 & 111.8 & 38.2 & 12.6 & 14.2 & 39.3 & \textbf{7.5} \\
 & $2{,}000$ & 152.4 & 100.9 & 38.3 & 6.1 & 7.7 & 9.2 & \textbf{3.4} \\
 & $3{,}000$ & 152.4 & 96.7 & 38.3 & 4.2 & 5.2 & 3.5 & \textbf{2.1} \\
 & $4{,}000$ & 152.4 & 94.6 & 38.3 & 3.2 & 4.0 & 1.9 & \textbf{1.6} \\
\midrule
Permuted & $1{,}000$ & 153.4 & 95.0 & 51.8 & 17.2 & 13.5 & 47.9 & \textbf{8.0} \\
 & $2{,}000$ & 154.1 & 85.5 & 52.3 & 9.0 & 7.4 & 12.3 & \textbf{3.5} \\
 & $3{,}000$ & 154.1 & 81.8 & 52.3 & 6.5 & 5.0 & 5.3 & \textbf{2.2} \\
 & $4{,}000$ & 154.1 & 79.9 & 52.3 & 5.3 & 3.8 & 3.0 & \textbf{1.6} \\
\midrule
\multicolumn{9}{l}{\emph{Panel B. Realized MSE}}\\
\midrule
$(+,+)$ & $1{,}000$ & 81.9 & 62.9 & 24.9 & 15.3 & \textbf{8.0} & 26.0 & 8.7 \\
 & $2{,}000$ & 88.8 & 63.4 & 24.6 & 8.2 & 4.9 & 6.9 & \textbf{3.0} \\
 & $3{,}000$ & 91.4 & 65.6 & 24.7 & 6.0 & 3.2 & 2.8 & \textbf{2.0} \\
 & $4{,}000$ & 93.2 & 63.8 & 24.3 & 3.5 & 2.3 & 2.1 & \textbf{1.6} \\
\midrule
$(+,-)$ & $1{,}000$ & 262.6 & 92.7 & 149.0 & 30.5 & 18.1 & 131.1 & \textbf{4.9} \\
 & $2{,}000$ & 285.0 & 88.3 & 154.0 & 18.6 & 11.7 & 37.2 & \textbf{3.2} \\
 & $3{,}000$ & 291.8 & 84.4 & 155.2 & 15.5 & 8.0 & 17.0 & \textbf{2.3} \\
 & $4{,}000$ & 296.2 & 83.1 & 154.9 & 13.0 & 6.1 & 8.6 & \textbf{1.5} \\
\midrule
$(-,+)$ & $1{,}000$ & 67.8 & 68.2 & \textbf{1.0} & 8.9 & 5.4 & 7.4 & 8.3 \\
 & $2{,}000$ & 75.1 & 70.3 & \textbf{1.3} & 6.3 & 2.3 & 3.6 & 3.9 \\
 & $3{,}000$ & 76.5 & 72.8 & \textbf{1.2} & 4.6 & 2.4 & 2.9 & 2.9 \\
 & $4{,}000$ & 79.6 & 73.9 & 1.4 & 2.9 & \textbf{1.4} & 1.8 & 1.6 \\
\midrule
$(-,-)$ & $1{,}000$ & 128.0 & 92.1 & 37.3 & 11.9 & 12.5 & 39.7 & \textbf{6.4} \\
 & $2{,}000$ & 139.7 & 91.6 & 38.7 & 6.4 & 7.8 & 9.1 & \textbf{4.0} \\
 & $3{,}000$ & 142.5 & 90.1 & 39.1 & 4.1 & 5.1 & 3.6 & \textbf{2.4} \\
 & $4{,}000$ & 143.4 & 89.3 & 37.4 & 2.6 & 3.3 & 2.0 & \textbf{1.3} \\
\midrule
Permuted & $1{,}000$ & 129.0 & 78.0 & 50.6 & 17.4 & 11.9 & 47.2 & \textbf{8.1} \\
 & $2{,}000$ & 138.2 & 76.1 & 50.6 & 8.3 & 6.1 & 11.4 & \textbf{2.9} \\
 & $3{,}000$ & 142.1 & 75.8 & 51.2 & 6.7 & 5.0 & 5.9 & \textbf{2.6} \\
 & $4{,}000$ & 144.6 & 74.5 & 50.6 & 5.2 & 3.5 & 3.0 & \textbf{1.7} \\
\bottomrule\end{tabular}}
\par\vspace{2pt}{\footnotesize\raggedright Notes: Scenarios are defined in Appendix~\ref{app:syn-joint}. Expected MSE uses $A_q$ for AI-augmented outputs and $\operatorname{Var}(Y_q)$ for human-only outputs. Bold marks the smallest gap among the implementable methods.\par}
\end{table}

\subsubsection{Twin data}
\label{app:supp-dependence}

Let $r_q=2(\operatorname{midrank}(A_q)-1)/(Q-1)-1$, using the baseline population difficulties. For each of the four sign pairs $(s_w,s_c)\in\{-1,+1\}^2$, set
\begin{equation}
w_q=1+0.75s_wr_q,\qquad c_q=1+0.75s_cr_q.
\end{equation}
Weights and costs have the same marginal spreads in every case, range from $0.25$ to $1.75$, and average one. Their rank correlations with $A$ equal the selected signs; their Pearson correlations with $A$ are those signs times $0.940$, and their mutual correlation is $s_ws_c$. A control randomly permutes the same rank grid, independently for weights and costs; its realized Pearson correlations $(w,A)$, $(c,A)$, and $(w,c)$ are $0.154$, $-0.047$, and $0.105$ (rank correlations $0.178$, $-0.008$, and $0.105$). These are constructed decision scenarios, not observed importance measures or prices.

At budgets 1{,}000 and 2{,}000, the comparison with Uniform depends on the joint structure. TRD-UCB has lower MSE when high difficulty accompanies higher importance and lower cost, when both importance and cost decrease with difficulty, and in the independently permuted control. When importance decreases but cost increases with difficulty, Uniform has lower MSE at both budgets. When both increase with difficulty, Uniform has lower MSE at $B=1{,}000$, and the paired difference at $B=2{,}000$ is unresolved. On the allocation objective, TRD-UCB is closest to the Oracle among the non-Oracle policies in this six-method configuration in every scenario and budget. In the two scenarios where Uniform has lower realized MSE, its realized error is within 8\% of the Oracle's. When importance falls and cost rises with difficulty, $w_q/c_q$ falls as $A_q$ rises, and Uniform's realized error is 3.6--4.1\% above the Oracle's. When both rise, it is 7.3--7.9\% above. When importance rises and cost falls with difficulty, TRD-UCB's realized error is 36.5\% and 39.8\% below Uniform's at $B=1{,}000$ and $2{,}000$. With each policy's final counts kept and the labels redrawn, TRD-UCB has lower realized MSE than Uniform in every joint scenario at both budgets, by 1.2\% to 41.8\%, including 1.2--1.6\% when importance falls and cost rises with difficulty and 4.0--5.6\% when both rise.

Table~\ref{tab:supp-joint} reports all five joint scenarios, with paired Uniform-minus-TRD-UCB differences and their Monte Carlo uncertainty.
\begin{table}[p]
\centering\footnotesize
\setlength{\tabcolsep}{4pt}
\renewcommand{\arraystretch}{1.08}
\caption{Joint dependence of importance weights, label costs, and rectification difficulty (500 Twin replays).}
\label{tab:supp-joint}
\fitwidth{\begin{tabular}{llccccc}
\toprule
Scenario & $B$ & Oracle & TRD-UCB & StratPPI & Uniform (AI-aug.) & \shortstack{Uniform $-$ TRD-UCB (SE)\\95\% MC interval} \\
\midrule
$(+,+)$ & $1{,}000$ & $8.651\;(0.078)$ & $9.721\;(0.090)$ & $10.125\;(0.098)$ & $\mathbf{9.283}\;(0.093)$ & \shortstack{$-0.438\;(0.065)$\\$[-0.566,-0.309]$} \\
$(+,+)$ & $2{,}000$ & $4.317\;(0.042)$ & $4.707\;(0.048)$ & $4.822\;(0.048)$ & $\mathbf{4.656}\;(0.048)$ & \shortstack{$-0.051\;(0.032)$\\$[-0.114,0.013]$} \\
\midrule
$(+,-)$ & $1{,}000$ & $5.338\;(0.043)$ & $\mathbf{5.937}\;(0.050)$ & $7.307\;(0.078)$ & $9.351\;(0.094)$ & \shortstack{$3.414\;(0.083)$\\$[3.251,3.576]$} \\
$(+,-)$ & $2{,}000$ & $2.654\;(0.022)$ & $\mathbf{2.813}\;(0.025)$ & $3.129\;(0.033)$ & $4.670\;(0.048)$ & \shortstack{$1.857\;(0.043)$\\$[1.772,1.941]$} \\
\midrule
$(-,+)$ & $1{,}000$ & $5.295\;(0.043)$ & $5.760\;(0.042)$ & $6.063\;(0.044)$ & $\mathbf{5.484}\;(0.042)$ & \shortstack{$-0.276\;(0.031)$\\$[-0.337,-0.215]$} \\
$(-,+)$ & $2{,}000$ & $2.634\;(0.022)$ & $2.863\;(0.023)$ & $2.970\;(0.024)$ & $\mathbf{2.741}\;(0.021)$ & \shortstack{$-0.122\;(0.016)$\\$[-0.153,-0.091]$} \\
\midrule
$(-,-)$ & $1{,}000$ & $5.035\;(0.040)$ & $\mathbf{5.396}\;(0.039)$ & $5.863\;(0.044)$ & $5.514\;(0.042)$ & \shortstack{$0.118\;(0.026)$\\$[0.066,0.169]$} \\
$(-,-)$ & $2{,}000$ & $2.515\;(0.021)$ & $\mathbf{2.656}\;(0.021)$ & $2.761\;(0.022)$ & $2.751\;(0.021)$ & \shortstack{$0.095\;(0.014)$\\$[0.068,0.121]$} \\
\midrule
Permuted control & $1{,}000$ & $6.195\;(0.055)$ & $\mathbf{6.888}\;(0.061)$ & $7.570\;(0.069)$ & $7.755\;(0.077)$ & \shortstack{$0.867\;(0.059)$\\$[0.750,0.983]$} \\
Permuted control & $2{,}000$ & $3.086\;(0.028)$ & $\mathbf{3.298}\;(0.032)$ & $3.470\;(0.034)$ & $3.869\;(0.039)$ & \shortstack{$0.571\;(0.030)$\\$[0.511,0.630]$} \\
\bottomrule\end{tabular}}
\par\vspace{2pt}{\footnotesize\raggedright Notes: Let $r_q$ be the centered difficulty rank in $[-1,1]$. The four sign cases use $w_q=1+0.75s_wr_q$ and $c_q=1+0.75s_cr_q$. The control independently permutes the same rank grid for weights and costs, preserving their marginal spreads. Each condition uses 500 paired replays with the parameters of Section~\ref{subsec:synthetic_setup}. MSE is $\sum_qw_q(\widehat\mu_q-\mu_q)^2/\sum_qw_q$, averaged over replays, in $10^{-3}$ units; parentheses give Monte Carlo SEs. The final column gives Uniform minus TRD-UCB, its paired SE, and its pointwise 95\% normal Monte Carlo interval; positive differences favor TRD-UCB. Bold marks the smallest non-Oracle mean MSE.\par}
\end{table}

\endgroup

\begingroup
\subsection{Changes in the respondent distribution}\label{app:robust-shift}

Two checks change which respondents the policies sample while keeping the 68 questions: a reweighting of the respondent sampling distribution, and a restriction to half of the respondents.

\subsubsection{Controlled change in the respondent sampling distribution}
\label{app:supp-shift}

We construct two shifts without using human labels to choose respondent weights. Let $g_i$ be respondent $i$'s mean available AI prediction and $u_i=(\operatorname{midrank}(g_i)-0.5)/2{,}058$. Set respondent weight $\omega_i=\exp(\eta u_i)$, for $\eta=-\log4$ and $+\log4$, then normalize these weights among the available paired respondents within each question. Both shifts retain the original support, with respondent density ratios at most four to one. Sampling probabilities, target human means, known AI means, and population difficulties all use the shifted law. Parameters, initialization, and budgets are unchanged, and no policy is retuned for either shift.

The downward and upward tilts change question means by an average absolute $0.015$ and $0.013$, respectively, with maximum absolute changes $0.075$ and $0.070$. The minimum effective-pool-size fractions, $(\sum_ip_{iq}^2)^{-1}/N_q$, are $0.842$ and $0.860$; difficulty-rank correlations with baseline are $0.995$ and $0.991$. These diagnostics show a change in respondent mix while also making clear that the difficulty ordering changes little.

On the allocation objective, TRD-UCB is closest to the Oracle among the non-Oracle policies in this six-method configuration under both tilts at both budgets, with gaps of 2.0\% to 4.0\%, against 2.8\% to 4.8\% for TRD-TS and 9.5\% to 11.4\% for Uniform. At $B=1{,}000$, Uniform has lower MSE under the downward tilt, while the upward-tilt paired difference is unresolved. At $B=2{,}000$, TRD-UCB has lower MSE under both tilts. Both tilts leave the Oracle's advantage over Uniform near its baseline size, so the results repeat the unshifted pattern; with each policy's final counts kept and the labels redrawn, TRD-UCB has lower realized MSE than Uniform under both tilts at both budgets, by 5.5\% to 7.9\%. Moderated $t$ coverage under both tilts is within 0.83 percentage points of the unshifted values for TRD-UCB and for Uniform. This is a controlled distribution-shift stress test within one observed dataset.

Table~\ref{tab:supp-shift} reports both shifted distributions at both budgets.
\begin{table}[p]
\centering\footnotesize
\setlength{\tabcolsep}{4pt}
\renewcommand{\arraystretch}{1.08}
\caption{Controlled changes in the respondent sampling distribution (500 Twin replays).}
\label{tab:supp-shift}
\fitwidth{\begin{tabular}{llccccc}
\toprule
Scenario & $B$ & Oracle & TRD-UCB & StratPPI & Uniform (AI-aug.) & \shortstack{Uniform $-$ TRD-UCB (SE)\\95\% MC interval} \\
\midrule
$\eta=-\log 4$ & $1{,}000$ & $6.975\;(0.056)$ & $7.722\;(0.065)$ & $8.151\;(0.067)$ & $\mathbf{7.579}\;(0.067)$ & \shortstack{$-0.143\;(0.047)$\\$[-0.234,-0.051]$} \\
$\eta=-\log 4$ & $2{,}000$ & $3.418\;(0.028)$ & $\mathbf{3.679}\;(0.033)$ & $3.795\;(0.034)$ & $3.742\;(0.034)$ & \shortstack{$0.063\;(0.023)$\\$[0.018,0.109]$} \\
\midrule
$\eta=+\log 4$ & $1{,}000$ & $6.555\;(0.055)$ & $\mathbf{7.222}\;(0.064)$ & $7.554\;(0.065)$ & $7.265\;(0.069)$ & \shortstack{$0.043\;(0.047)$\\$[-0.048,0.134]$} \\
$\eta=+\log 4$ & $2{,}000$ & $3.187\;(0.025)$ & $\mathbf{3.419}\;(0.029)$ & $3.509\;(0.030)$ & $3.545\;(0.032)$ & \shortstack{$0.126\;(0.024)$\\$[0.079,0.173]$} \\
\bottomrule\end{tabular}}
\par\vspace{2pt}{\footnotesize\raggedright Notes: Respondent probabilities are proportional to $\exp(\eta u_i)$, where $u_i$ ranks the respondent's mean available AI prediction; human outcomes are not used in the tilt. Sampling, human targets, known AI means, and Oracle difficulties use the same shifted law. All empirical support is retained. Each condition uses 500 paired replays with the parameters of Section~\ref{subsec:synthetic_setup}; values are $10^{-3}$ MSE units with Monte Carlo SEs in parentheses. The final column gives Uniform minus TRD-UCB with its paired SE and pointwise 95\% interval; positive differences favor TRD-UCB. Bold marks the smallest non-Oracle mean MSE.\par}
\end{table}

\endgroup

\begingroup
\subsubsection{Labels from half of the respondents}\label{app:pilot-partial-frame}

We examine the effect of limiting which Twin respondents can be acquired. The analysis uses the 500 paired streams of Section~\ref{subsec:dt_setup} and evaluates the same 68 question means defined from the full pool.

Table~\ref{tab:ec5-frame-contrasts} compares two acquisition frames for TRD-UCB while keeping the target fixed at the full-pool means: the full frame samples from all 2{,}058 respondents, whereas the partial frame samples only from a fixed random half of 1{,}029 respondents, both with replacement within question. The partial-frame-minus-full-frame mean MSE differences are approximately $0.00$, $0.04$, $0.05$, and $0.07$ ($\times10^{-3}$) at budgets 1{,}000 to 4{,}000; only the last difference, about $4\%$ of the error, has a paired interval that excludes zero. The difference is resolved only at the largest budget because the half frame's question means differ from the full-pool targets: with many labels from the half frame, the estimate converges to a value whose squared difference from the full-pool mean averages $0.081$ ($\times10^{-3}$) per question under any allocation rule, about 1\% of TRD-UCB's realized MSE at $B=1{,}000$ ($7.47$) and about 5\% at $B=4{,}000$ ($1.74$). Restricting acquisitions to half of the respondents therefore has at most a small effect on accuracy in this fixed population.
\begin{table}[ht]
\centering\footnotesize
\caption{TRD-UCB realized MSE under the two acquisition frames ($\times10^{-3}$ per question), evaluated against the same full-pool target.}
\label{tab:ec5-frame-contrasts}
\begin{tabular}{rrrrc}\toprule
$B$ & Half frame (1,029) & Full frame (2,058) & Paired difference & Paired 95\% interval \\
\midrule
1,000 & 7.472 & 7.472 & 0.000 & $[-0.173,0.173]$ \\
2,000 & 3.643 & 3.605 & 0.038 & $[-0.054,0.129]$ \\
3,000 & 2.406 & 2.353 & 0.053 & $[-0.006,0.112]$ \\
4,000 & 1.814 & 1.744 & 0.070 & $[0.027,0.113]$ \\
\bottomrule\end{tabular}
\par\vspace{2pt}{\footnotesize\raggedright Notes: Differences are partial minus full frame; intervals are pointwise normal Monte Carlo intervals from 500 paired streams. Both frames sample with replacement within question; the eligible rows of the partial frame belong to a fixed half of the respondents, and the target and the pool AI mean are those of the full pool.\par}
\end{table}

\endgroup

\begingroup
\subsection{Disjoint task halves}
\label{app:supp-task-split}

Each deployment learns its question difficulties from the labels it buys. Table~\ref{tab:supp-task-split} splits the 68 questions by task type into two disjoint halves of 34 questions, runs each half as its own deployment with $\gamma=0.08$ unchanged, and halves the budgets so that labels per question match the baseline. TRD-UCB's allocation-objective gaps range from $0.7\%$ to $5.2\%$, compared with $7.1\%$ to $10.5\%$ for Uniform. At a half-pool budget of 500, Uniform has lower realized error on half A and the comparison on half B is unresolved. At 1{,}000, TRD-UCB has lower realized error on half B and the comparison on half A is unresolved; at 2{,}000, TRD-UCB has lower realized error on both halves. These statements use the displayed pointwise paired intervals and illustrate the budget dependence of realized gains. A half-pool budget gives the same labels per question as a full-pool budget twice as large, so the half-pool budget 500, about 15 labels per question, matches the full-pool $B=1{,}000$, where Uniform also has the lower mean realized MSE. Half A contains 17 of the 20 binary questions with human-response variance above $0.15$, which carry most of TRD-UCB's excess. With each policy's final counts kept and the labels redrawn, TRD-UCB has lower realized MSE than Uniform on both halves at every displayed budget, by 5.0\% to 7.9\%. The parameters are not retuned by half.
\begin{table}[ht]
\centering\footnotesize
\setlength{\tabcolsep}{3pt}
\caption{Task-split check: disjoint task-type halves of the Twin pool.}
\label{tab:supp-task-split}
\fitwidth{\begin{tabular}{llrrrrlrr}
\toprule
Half & $B$ & Oracle & TRD-UCB & StratPPI & Uniform (AI-aug.) & Uniform $-$ TRD-UCB (SE) [95\% MC] & \multicolumn{2}{c}{Objective gap (\%)} \\
\cmidrule(lr){8-9} & & \multicolumn{4}{c}{Realized MSE} & & TRD-UCB & Uniform (AI-aug.) \\
\midrule
A & 500 & 7.863 & 9.094 & 9.247 & 8.666 & $-0.429\;(0.081)$ $[-0.587,-0.270]$ & 5.2 & 9.8 \\
A & 1,000 & 3.954 & 4.419 & 4.450 & 4.384 & $-0.036\;(0.040)$ $[-0.114,0.042]$ & 2.9 & 10.5 \\
A & 2,000 & 2.010 & 2.111 & 2.120 & 2.180 & $0.070\;(0.020)$ $[0.031,0.108]$ & 1.5 & 10.2 \\
\midrule
B & 500 & 5.617 & 6.046 & 6.533 & 6.036 & $-0.010\;(0.045)$ $[-0.099,0.078]$ & 2.2 & 7.5 \\
B & 1,000 & 2.816 & 2.925 & 3.023 & 3.003 & $0.077\;(0.025)$ $[0.028,0.127]$ & 1.3 & 7.1 \\
B & 2,000 & 1.382 & 1.417 & 1.447 & 1.476 & $0.059\;(0.012)$ $[0.036,0.082]$ & 0.7 & 7.2 \\
\bottomrule\end{tabular}}
\par\vspace{2pt}{\footnotesize\raggedright Notes: Half A contains 34 questions from 6 task types (Probability Matching vs Maximizing, Nonseparability of Risks and Benefits, Anchoring, Abs vs Rel Savings, Framing Problem, Outcome Bias); half B contains 34 questions from the other 8 task types (False Consensus, Less is More, Conjunction (Linda), WTA/WTP -- Thaler Problem, Allais Problem, Myside Bias, Denominator Neglect, Omission Bias). Each half is run as its own deployment with the parameters of Section~\ref{subsec:synthetic_setup} unchanged, $K=3$, $w_q=c_q=1$, and 500 paired replays in which each question draws the same observations as in the full-pool replays; budgets are halved relative to the corresponding full-pool comparison. Both halves come from the pool that supplied the Twin tuning replays. The TRD-UCB multiplier $\gamma=0.08$ is also the choice on the Synthetic tuning design alone (Table~\ref{tab:tuning-grid}); StratPPI ($\alpha=0.5$), TRD-TS, and TRD-$\varepsilon$-Greedy use the constants of the main comparison, which were chosen with both environments. Realized MSE and paired differences are in $10^{-3}$ units; positive differences favor TRD-UCB. Objective gap is the relative gap of the allocation objective to the same-stream Oracle.\par}
\end{table}

\endgroup

\begingroup

\subsection{Twin coefficient range and finite-sample error}\label{app:twin-lambda-range}

The allocation benchmark uses the minimum residual variance within the chosen coefficient interval. This appendix examines the Twin coefficient distribution and the sensitivity of the results to a wider or unrestricted range. By Appendix~\ref{subsec:finite_upper_bounds}, any fixed interval $[a,b]$ keeps the $O(\ln B/B^2)$ guarantee relative to its own oracle difficulty $A_q^{[a,b]}$, so the $[0,2]$ rule below is covered by the theory. A wider interval can only lower the oracle objective; the price is a larger range bound in the confidence radius and noisier fitted coefficients at small counts. The unrestricted rule has no fixed bounded interval and lies outside this result; we include it as an empirical reference. In the fixed-pool diagnostics, $\lambda_q^*$ is the coefficient projected onto $[0,1]$, and unrestricted slopes are reported separately.

Tables~\ref{tab:twin-question-diagnostics-1}--\ref{tab:twin-question-diagnostics-2} summarize the full available pool for each of the 68 questions, using finite-pool population variances and covariances with denominator $N_q$, matching the with-replacement sampling law. Among the 47 questions with nonconstant AI predictions, the unrestricted coefficients range from $-0.30$ to $0.72$: five are negative and none exceeds one. Correlation and the unrestricted coefficient are undefined for the remaining 21 questions and are reported as NA. Table~\ref{tab:twin-clipping-population} compares the population rectification difficulties and continuous Neyman shares with and without the restriction to $[0,1]$. For the five negative slopes, the $[0,1]$ optimum is zero; for the 21 constant predictors, every coefficient gives the same residual variance.

\begin{table}[p]\centering\footnotesize
\setlength{\tabcolsep}{3pt}
\caption{Full-pool Twin question diagnostics (part 1 of 2).}
\label{tab:twin-question-diagnostics-1}
\fitwidth{\begin{tabular}{lrrrrrrr}\toprule
Question & $N_q$ & $\Var(Y)$ & $\Var(S)$ & $\operatorname{Corr}(Y,S)$ & $\lambda^u$ & $\lambda^*$ & $A_q$ \\ \midrule
Denominator neglect & 2058 & 0.2352 & 0.1343 & -0.0425 & -0.0562 & 0.0000 & 0.2352 \\
False Cons. self \_1 & 2058 & 0.1056 & 0.1277 & 0.5353 & 0.4868 & 0.4868 & 0.0753 \\
False Cons. self \_10 & 2058 & 0.0920 & 0.0344 & 0.0570 & 0.0932 & 0.0932 & 0.0917 \\
False Cons. self \_2 & 2058 & 0.0829 & 0.0657 & 0.5241 & 0.5885 & 0.5885 & 0.0601 \\
False Cons. self \_3 & 2058 & 0.1030 & 0.0793 & 0.5636 & 0.6425 & 0.6425 & 0.0703 \\
False Cons. self \_4 & 2058 & 0.1061 & 0.1790 & 0.5327 & 0.4101 & 0.4101 & 0.0760 \\
False Cons. self \_5 & 2058 & 0.0627 & 0.0974 & 0.2927 & 0.2349 & 0.2349 & 0.0574 \\
False Cons. self \_6 & 2058 & 0.1168 & 0.1071 & 0.5809 & 0.6067 & 0.6067 & 0.0774 \\
False Cons. self \_7 & 2058 & 0.0732 & 0.0533 & 0.4149 & 0.4863 & 0.4863 & 0.0606 \\
False Cons. self \_8 & 2058 & 0.1009 & 0.1210 & 0.4622 & 0.4219 & 0.4219 & 0.0793 \\
False Cons. self \_9 & 2058 & 0.1270 & 0.1175 & 0.5285 & 0.5495 & 0.5495 & 0.0915 \\
Omission bias & 2058 & 0.1065 & 0.0112 & 0.1283 & 0.3950 & 0.3950 & 0.1047 \\
Q157 & 1003 & 0.0594 & 0.0034 & -0.0726 & -0.3041 & 0.0000 & 0.0594 \\
Q158 & 1055 & 0.0580 & 0.0861 & 0.0186 & 0.0153 & 0.0153 & 0.0579 \\
Q159\_1 & 1029 & 0.0450 & 0.0100 & 0.0433 & 0.0920 & 0.0920 & 0.0450 \\
Q159\_2 & 1029 & 0.0537 & 0.0099 & 0.0876 & 0.2036 & 0.2036 & 0.0532 \\
Q159\_3 & 1029 & 0.0349 & 0.0091 & 0.0367 & 0.0720 & 0.0720 & 0.0348 \\
Q160\_1 & 1029 & 0.0444 & 0.0067 & -0.0444 & -0.1148 & 0.0000 & 0.0444 \\
Q160\_2 & 1029 & 0.0532 & 0.0092 & -0.0125 & -0.0299 & 0.0000 & 0.0532 \\
Q160\_3 & 1029 & 0.0543 & 0.0397 & 0.0303 & 0.0354 & 0.0354 & 0.0542 \\
Q161 & 1052 & 0.0294 & 0.0003 & 0.0041 & 0.0440 & 0.0440 & 0.0294 \\
Q162 & 1006 & 0.0524 & 0.0002 & 0.0436 & 0.6751 & 0.6751 & 0.0523 \\
Q163 & 1002 & 0.1610 & 0.0000 & NA & NA & 0.0000 & 0.1610 \\
Q165 & 1056 & 0.1506 & 0.1383 & 0.0242 & 0.0252 & 0.0252 & 0.1505 \\
Q167 & 1049 & 0.0531 & 0.0000 & NA & NA & 0.0000 & 0.0531 \\
Q169 & 1009 & 0.2151 & 0.0000 & NA & NA & 0.0000 & 0.2151 \\
Q171 & 735 & 0.0964 & 0.0326 & 0.0874 & 0.1502 & 0.1502 & 0.0957 \\
Q172 & 672 & 0.1318 & 0.0389 & 0.0698 & 0.1285 & 0.1285 & 0.1312 \\
Q173 & 651 & 0.1205 & 0.0417 & 0.0282 & 0.0479 & 0.0479 & 0.1204 \\
Q174 & 735 & 0.0860 & 0.0275 & 0.2994 & 0.5291 & 0.5291 & 0.0783 \\
Q175 & 672 & 0.0734 & 0.0198 & 0.1395 & 0.2688 & 0.2688 & 0.0720 \\
Q176 & 651 & 0.0691 & 0.0176 & 0.1974 & 0.3905 & 0.3905 & 0.0664 \\
Q177 & 735 & 0.0833 & 0.0118 & 0.2700 & 0.7160 & 0.7160 & 0.0772 \\
Q178 & 672 & 0.0762 & 0.0123 & 0.2484 & 0.6193 & 0.6193 & 0.0715 \\
\bottomrule\end{tabular}}
\par\smallskip{\footnotesize\raggedright Notes: Each row uses all available scaled human--AI pairs for the question. Population variances and covariances use denominator $N_q$, matching with-replacement sampling from the empirical pool. $\lambda^u=\Cov(Y,S)/\Var(S)$, $\lambda^*=\Pi_{[0,1]}(\lambda^u)$, and $A_q=\Var(Y-\lambda^*S)$. For constant $S$, the correlation and unrestricted slope are undefined (NA); the coefficient is set to zero. These are fixed-pool quantities, not adaptive sample estimates. {Here $\lambda^*$ minimizes residual variance within $[0,1]$, including boundary optima and constant predictors.}\par}
\end{table}
\begin{table}[p]\centering\footnotesize
\setlength{\tabcolsep}{3pt}
\caption{Full-pool Twin question diagnostics (part 2 of 2).}
\label{tab:twin-question-diagnostics-2}
\fitwidth{\begin{tabular}{lrrrrrrr}\toprule
Question & $N_q$ & $\Var(Y)$ & $\Var(S)$ & $\operatorname{Corr}(Y,S)$ & $\lambda^u$ & $\lambda^*$ & $A_q$ \\ \midrule
Q179 & 651 & 0.0701 & 0.0100 & 0.2204 & 0.5836 & 0.5836 & 0.0667 \\
Q183 & 1031 & 0.1987 & 0.0453 & 0.0266 & 0.0557 & 0.0557 & 0.1986 \\
Q184 & 1027 & 0.2064 & 0.1716 & 0.0010 & 0.0011 & 0.0011 & 0.2064 \\
Q189 & 712 & 0.0487 & 0.0053 & 0.1792 & 0.5442 & 0.5442 & 0.0471 \\
Q190 & 673 & 0.0913 & 0.0179 & 0.0281 & 0.0633 & 0.0633 & 0.0912 \\
Q191 & 673 & 0.0243 & 0.0001 & 0.0225 & 0.3333 & 0.3333 & 0.0242 \\
Q192 & 1051 & 0.2316 & 0.0000 & NA & NA & 0.0000 & 0.2316 \\
Q193 & 1007 & 0.2350 & 0.0000 & NA & NA & 0.0000 & 0.2350 \\
Q194 & 1015 & 0.0668 & 0.0233 & 0.2083 & 0.3526 & 0.3526 & 0.0639 \\
Q195 & 1043 & 0.0605 & 0.0242 & 0.1324 & 0.2093 & 0.2093 & 0.0594 \\
Q198\_1 & 1032 & 0.0723 & 0.0000 & NA & NA & 0.0000 & 0.0723 \\
Q198\_10 & 1032 & 0.2062 & 0.0000 & NA & NA & 0.0000 & 0.2062 \\
Q198\_2 & 1032 & 0.1130 & 0.0000 & NA & NA & 0.0000 & 0.1130 \\
Q198\_3 & 1032 & 0.1759 & 0.0000 & NA & NA & 0.0000 & 0.1759 \\
Q198\_4 & 1032 & 0.1644 & 0.0000 & NA & NA & 0.0000 & 0.1644 \\
Q198\_5 & 1032 & 0.1632 & 0.0000 & NA & NA & 0.0000 & 0.1632 \\
Q198\_6 & 1032 & 0.1875 & 0.0000 & NA & NA & 0.0000 & 0.1875 \\
Q198\_7 & 1032 & 0.1932 & 0.0000 & NA & NA & 0.0000 & 0.1932 \\
Q198\_8 & 1032 & 0.2105 & 0.0000 & NA & NA & 0.0000 & 0.2105 \\
Q198\_9 & 1032 & 0.2196 & 0.0000 & NA & NA & 0.0000 & 0.2196 \\
Q203\_1 & 1026 & 0.0384 & 0.0000 & NA & NA & 0.0000 & 0.0384 \\
Q203\_2 & 1026 & 0.1157 & 0.0000 & NA & NA & 0.0000 & 0.1157 \\
Q203\_3 & 1026 & 0.2248 & 0.0000 & NA & NA & 0.0000 & 0.2248 \\
Q203\_4 & 1026 & 0.1539 & 0.0000 & NA & NA & 0.0000 & 0.1539 \\
Q203\_5 & 1026 & 0.2339 & 0.0000 & NA & NA & 0.0000 & 0.2339 \\
Q203\_6 & 1026 & 0.2385 & 0.0000 & NA & NA & 0.0000 & 0.2385 \\
nonseparability ris \_1 & 2058 & 0.0611 & 0.0024 & -0.0125 & -0.0629 & 0.0000 & 0.0611 \\
nonseparability ris \_2 & 2058 & 0.0406 & 0.0052 & 0.0746 & 0.2082 & 0.2082 & 0.0404 \\
nonseparability ris \_3 & 2058 & 0.0383 & 0.0043 & 0.0833 & 0.2472 & 0.2472 & 0.0380 \\
nonseparability ris \_4 & 2058 & 0.0365 & 0.0052 & 0.0816 & 0.2167 & 0.2167 & 0.0362 \\
nonseparabilty bene \_1 & 2058 & 0.0252 & 0.0105 & 0.1668 & 0.2583 & 0.2583 & 0.0245 \\
nonseparabilty bene \_2 & 2058 & 0.0604 & 0.0062 & 0.0961 & 0.2992 & 0.2992 & 0.0599 \\
nonseparabilty bene \_3 & 2058 & 0.0767 & 0.0107 & 0.0662 & 0.1770 & 0.1770 & 0.0763 \\
nonseparabilty bene \_4 & 2058 & 0.0730 & 0.0094 & 0.0926 & 0.2574 & 0.2574 & 0.0723 \\
\bottomrule\end{tabular}}
\par\smallskip{\footnotesize\raggedright Notes: Each row uses all available scaled human--AI pairs for the question. Population variances and covariances use denominator $N_q$, matching with-replacement sampling from the empirical pool. $\lambda^u=\Cov(Y,S)/\Var(S)$, $\lambda^*=\Pi_{[0,1]}(\lambda^u)$, and $A_q=\Var(Y-\lambda^*S)$. For constant $S$, the correlation and unrestricted slope are undefined (NA); the coefficient is set to zero. These are fixed-pool quantities, not adaptive sample estimates. {Here $\lambda^*$ minimizes residual variance within $[0,1]$, including boundary optima and constant predictors.}\par}
\end{table}
\begin{table}[ht]\centering\footnotesize
\setlength{\tabcolsep}{4pt}
\caption{Effect of restricting the population coefficient on the five affected Twin questions.}
\label{tab:twin-clipping-population}
\begin{tabular}{lrrrrr}\toprule
Question & $\lambda^u$ & $A_q^u$ & $A_q$ & Increase (\%) & $100(p_q-p_q^u)$ \\ \midrule
Denominator neglect & -0.0562 & 0.2348 & 0.2352 & 0.1807 & 0.0019 \\
Q157 & -0.3041 & 0.0591 & 0.0594 & 0.5304 & 0.0030 \\
Q160\_1 & -0.1148 & 0.0443 & 0.0444 & 0.1979 & 0.0009 \\
Q160\_2 & -0.0299 & 0.0532 & 0.0532 & 0.0155 & 0.0000 \\
nonseparability ris \_1 & -0.0629 & 0.0611 & 0.0611 & 0.0155 & 0.0000 \\
\bottomrule\end{tabular}
\par\smallskip{\footnotesize\raggedright Notes: $A_q^u$ minimizes the residual variance over unrestricted coefficients; $A_q$ restricts coefficients to $[0,1]$. Increase is $100(A_q/A_q^u-1)$. For $w_q=c_q=1$, $p_q\propto\sqrt{A_q}$ and $p_q^u\propto\sqrt{A_q^u}$ are normalized continuous Neyman shares. No identifiable slope exceeds one. Constant-prediction questions have $A_q=A_q^u=\Var(Y)$ and are omitted; share normalization also affects other questions.\par}
\end{table}
Across all 68 questions, the restriction increases the continuous Oracle objective by 0.013\% at any fixed budget. The largest change in a normalized Oracle share is 0.003 percentage points; the total variation between the share vectors is 0.006\%. These full-pool calculations isolate the correction-class restriction; they do not measure online estimation effects in Table~\ref{tab:twin-lambda-range}.

For question $q$, the corrected mean is $\widehat\mu_q(\lambda)=\bar Y_q+\lambda(\mu_{S,q}-\bar S_q)$, where $\mu_{S,q}$ is the mean AI prediction over the target pool. At $\lambda=0$ the output is the human-label sample mean, with no AI correction; at $\lambda=1$ it applies the full unit-weight prediction-powered correction.  A negative unrestricted coefficient describes an inverse signal that can be used by an interval containing negative values. Likewise, an unrestricted coefficient above one motivates a larger upper bound (Appendix~\ref{subsec:finite_upper_bounds}); either choice changes the coefficient class and its oracle difficulty.

We compare three TRD-UCB policies that use the same coefficient rule for both residual-based acquisition and final estimation: projection of the sample covariance-to-predictor-variance ratio onto $[0,1]$, projection onto $[0,2]$, and no projection. A predictor variance at or below $10^{-12}$ sets the coefficient to zero under all three rules. The Oracle retains the fitted $[0,1]$ output rule as a fixed reference. Table~\ref{tab:twin-lambda-range} reports the mean, over 500 paired Twin replays, of the sum of question-level squared errors at four checkpoints of the same paths, $B=1{,}000$ to $4{,}000$. Parentheses give each policy's percentage gap relative to the Oracle mean at that budget.

\begin{table}[ht]\centering\footnotesize\setlength{\tabcolsep}{4pt}
\caption{Twin sum of question-level realized MSE ($\times10^{-2}$) across budgets for three coefficient ranges, 500 paired replays. Gap to Oracle (\%) in parentheses.}
\label{tab:twin-lambda-range}
\begin{tabular}{r c|ccc}\toprule
$B$ & Oracle & TRD-UCB $[0,1]$ & TRD-UCB $[0,2]$ & TRD-UCB unrestricted \\\midrule
1{,}000 & 45.46 & \textbf{50.81 (11.8\%)} & 51.34 (12.9\%) & 52.19 (14.8\%) \\
2{,}000 & 22.86 & \textbf{24.51 (7.3\%)} & 24.61 (7.7\%) & 24.80 (8.5\%) \\
3{,}000 & 15.22 & \textbf{16.00 (5.2\%)} & 16.03 (5.4\%) & 16.12 (5.9\%) \\
4{,}000 & 11.43 & \textbf{11.86 (3.8\%)} & 11.88 (3.9\%) & 11.92 (4.3\%) \\
 
\bottomrule\end{tabular}
\par\vspace{2pt}{\footnotesize\raggedright Notes: All three rules use the same coefficient for acquisition and final estimation; the Oracle uses the fitted $[0,1]$ output. Bold marks the lowest mean in each row.\par}
\end{table}

The $[0,1]$ TRD-UCB policy has the lowest realized MSE among the three TRD-UCB ranges at every reported budget, and the paired 95\% intervals for its differences from the other two ranges exclude zero at every reported budget. At $B=1{,}000$ the Oracle gaps are 11.8\%, 12.9\%, and 14.8\% for the $[0,1]$, $[0,2]$, and unrestricted rules; at $B=4{,}000$ they are 3.8\%, 3.9\%, and 4.3\%. More labels narrow the separation among the TRD-UCB rules, while the $[0,1]$ range retains the lowest realized error in this replay.

 Sample ratios can fall outside $[0,1]$ even when the unrestricted population slope lies inside it; projection cannot move them farther from $\lambda_q^*$, and the observations behind them remain in the estimates. At budget $B=500$ on the $[0,1]$ TRD-UCB acquisition paths, the $500\times68=34{,}000$ question-by-replay coefficient fits have 18.3\% negative ratios, 8.8\% ratios above one, and 44.5\% zero sample predictor variances. These are checkpoint proportions, not proportions at the three-label initialization.  These fits reflect sampling fluctuation as well as the five negative unrestricted population slopes. The same fitted coefficient enters both the final mean correction and the rectification difficulty that guides later acquisitions, so the range comparison evaluates estimation and allocation together. It supports the $[0,1]$ range in this pool.
\endgroup

\end{document}